\RequirePackage{fix-cm}

\documentclass[smallextended]{svjour3}
\smartqed  
\usepackage{graphicx}
\usepackage{fullpage, xcolor}
\usepackage[english]{babel}
\usepackage{stmaryrd,enumerate}
\usepackage[export]{adjustbox}

\usepackage{wrapfig}
\usepackage{placeins}

\usepackage{latexsym}

\usepackage{amsthm}
\usepackage{amsmath}
\usepackage{amscd}
\usepackage[cmtip,arrow]{xy}
\usepackage{pb-diagram, pb-xy}
\usepackage{color}
\usepackage{amssymb,epsfig}
\usepackage{amsfonts}
\usepackage{scrextend}

\usepackage{mathrsfs}
\usepackage{fourier}
\usepackage{verbatim}

\RequirePackage[numbers]{natbib}
\usepackage{natbib}
\usepackage{notes2bib}

\usepackage{hyperref}

\usepackage{setspace}
\usepackage{extarrows}

\def\argmin{\mathop{\rm argmin}}

\def \S {\mathbb{S}}
\def \C { {\mathbb{B}}}
\def \d {\mathrm{d}}
\def \d {\mathrm{d}}

\usepackage[ruled, vlined]{algorithm2e}

\makeatletter
\RestyleAlgo{boxruled}
\renewcommand{\@algocf@capt@plain}{above}
\makeatother

\SetKwInput{KwInput}{Input}                
\SetKwInput{KwOutput}{Output}

\theoremstyle{definition}
\newtheorem{prop}{Proposition}

\begin{document}

\title{Amplitude Mean of Functional Data on $\mathbb{S}^2$ {and its Accurate Computation}
}


\author{Zhengwu Zhang
\and Bayan Saparbayeva}

\institute{Z. Zhang \at
              356 Hanes Hall, Department of Statistics and Operations Research, UNC Chapel Hill \\
              Tel.: +191-99-627998\\
              \email{zhengwu\_zhang@unc.edu}           
           \and
           B. Saparbayeva \at
              Department of Neurology, University of Rochester Medical Center\\
              \email{bayan\_saparbayeva@urmc.rochester.edu}
}


\maketitle

\begin{abstract}
Manifold-valued functional data analysis (FDA) has become an active area of research motivated by the rising availability of trajectories or longitudinal data observed on non-linear manifolds. The challenges of analyzing such data come from many aspects, including infinite dimensionality and nonlinearity, as well as time-domain or phase variability. In this paper, we study the amplitude part of manifold-valued functions on $\S^2$, which is invariant to random time warping or re-parameterization. {We represent a smooth function on $\S^2$ using a pair of  components: a starting point and a transported square-root velocity curve (TSRVC).Under this representation, the space of all smooth functions on $\S^2$ forms a vector bundle, and the simple $L_2$ norm becomes a time-warping invariant metric on this vector bundle.} Utilizing the nice geometry of $\S^2$, we develop a set of efficient and accurate tools for temporal alignment of functions, geodesic computing, and sample mean and {covariance} calculation. At the heart of these tools, they rely on gradient descent algorithms with carefully derived gradients. We show the advantages of these newly developed tools over its competitors with extensive simulations and real data and demonstrate the importance of considering the amplitude part of functions instead of mixing it with phase variability in manifold-valued FDA. 
\end{abstract}

\section{Introduction}
\label{intro}

Functional data analysis (FDA) has been an active area of research for decades, motivated by the rich availability of trajectories or longitudinal data observed over time. Well-known monographs include \cite{ramsay2005functional, ramsay2007applied, bowman2010functional} and \cite{ferraty2006nonparametric} among many others. Many advanced statistical techniques have been developed for classical functional data that take values in a vector space $p:[0, 1]\rightarrow \mathbb{R}^d,$ for instance, principal component analysis, functional linear regression and functional quantile regression (see \cite{RamsayDalzell1991, Hsing2015TheoreticalFO, LinYao2019}). 
However, many modern applications consider functional data taking values on 
a non-linear Riemannian manifold raised in the form $p:[0, 1]\rightarrow \mathcal{M}$, where $\mathcal{M}$ is a Riemannian manifold. 
Compared with methods for vector valued functional data, analytical methods for functional data on non-linear Riemannian manifold are very limited. In this paper, we consider a special manifold $\mathcal{M} = \S^2$, and develop fundamental statistical tools, such as geodesic calculation, temporal alignment and sample mean and covariance calculation, for analyzing functional data on $\S^2$. Two examples of such data are shown in Figure \ref{fig:intro_hurrican_bird}, where the left panel shows migration paths of a type of bird called {\it Swainson hawk} and the right shows some tracks for hurricanes originated from the Atlantic ocean. 

\begin{figure}[h]
\centering
\includegraphics[width=0.8\textwidth]{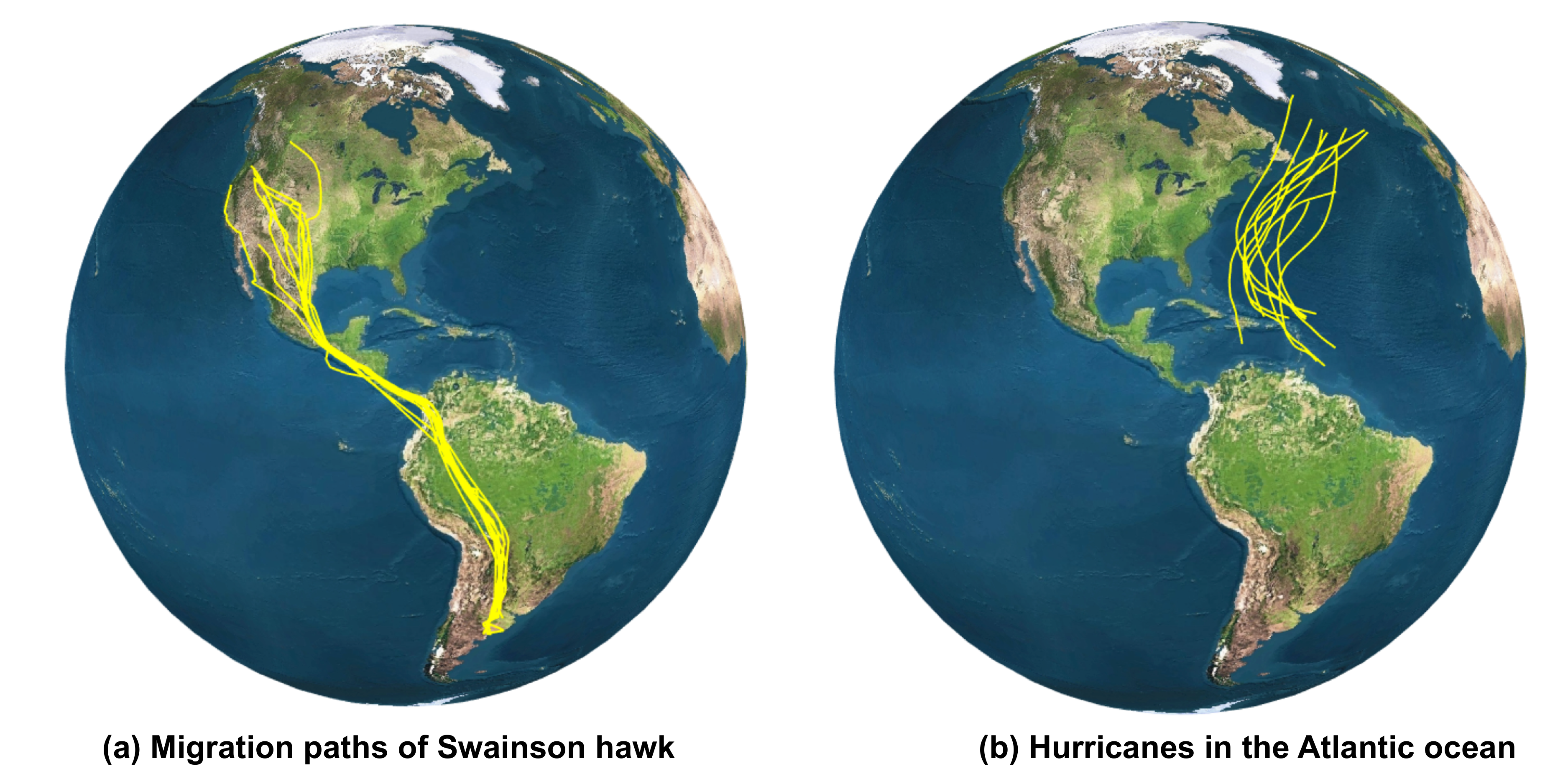}
\caption{Examples of functional data on manifold. The left panel shows migration paths of a type of bird called Swainson hawk and the right shows some tracks for hurricanes originating in 
the Atlantic ocean.}
\label{fig:intro_hurrican_bird}
\end{figure}

The challenges of analyzing functional data on a Riemannian manifold come from many aspects, such  as infinite dimensionality and nonlinearity. While the infinite dimensionality property is easy to understand given the consideration is functional data, the nonlinearity challenge comes from the fact that the functions take values on a non-linear manifold space, making many advanced techniques relying on linear structure ineffective. Another important challenge is the time domain variability in the functional data, which is less often considered in the literature. Taking the bird migration tracks as an example, two birds can fly similar paths but with different speeds. This phenomenon is also presented in traditional functional data, e.g, the misaligned peaks in the biomechanical data presented in Figure 1.11 in \cite{ramsay2005functional}, where alignment of the the functional data is necessary.

To more precisely describe the variability in the time domain, we start with introducing some notations. Let $p:[0, 1]\mapsto \S^2$ denote a smooth function on the 2-sphere, and $\gamma:[0, 1]\mapsto [0, 1]$ be a positive diffeomorphism with $\gamma(0) = 0$ and $\gamma(1) = 1$. Here the $\gamma$ plays the role of re-parameterizing the functional data, that is $\tilde{p} = p(\gamma)$ is a re-parameterized version of $p$.  $p$ and $\tilde{p}$ have the same path on $\S^2$, but $p(t)$ can be significantly different from $\tilde{p}(t)$ for the same $t$. For any two smooth functions or trajectories $p_1, p_2:[0, 1] \mapsto \S^2$, to remove the re-parameterization variability, we often align them first, which is done through finding a $\gamma$ such that $p_1(t)$ is optimally registered to $p_2(\gamma(t))$ for all $t \in [0, 1]$. This problem can be extended to multiple functions $\{p_1, \dots, p_n\}$ and we want to find a template and set of functions $\{\gamma_1, \dots, \gamma_n\}$ so that the template and $\{p_1(\gamma_1), \dots, p_1(\gamma_n)\}$ are optimally aligned. The template is regarded as amplitude mean and is our main consideration in this paper.  

Classical statistical analysis of functional data on  manifolds usually considers only raw amplitude variation (see \cite{RamsayDalzell1991, ramsay2005functional, MullerStadtmuller2005, Hsing2015TheoreticalFO}, and \cite{LinYao2019}) under some distance function, for example
\begin{align}\label{eqn:metric1}
 d(p_1, p_2) = \int_0^1 d_{g} \big(p_1(t), p_2(t)\big)\d t, 
\end{align}
where $d_{g}$ is the geodesic distance function on a manifold $\mathcal{M}.$ Utilizing this distance function, we can calculate summary statistics such as mean and covariance on the tangent space of the mean (see \cite{LinYao2019, MullerLiu2004}). However, due to the presence of the phase variability, the mean based on \eqref{eqn:metric1} does not always reflect the common pattern of the data and the covariance is inflated, as illustrated later in the real data analysis.  Alignment or phase-amplitude separation is an important step in manifold-valued FDA.

Structural analysis was proposed in \cite{KneipGasser1992}  to align curves by identifying the timing of salient features before applying further statistics. Another alignment method is the Procrustes method (see \cite{RamsayLi1998}) that iteratively warps each curve to the sample Fr\'{e}chet mean. The alignment is performed with the help of dynamic time warping algorithms (see \cite{Bertsekas1995DynamicPA, GasserGervini2004, james2007, Sakoe1978DynamicPA, WangGasser1997}), e.g., synchronizing two functions $p_1, p_2$ with respect to some distance function
\begin{align}\label{eqn:metric2}
  d(p_1, p_2) = \inf_{\gamma\in\Gamma} \int^1_0 d_{g} \Big(p_1(t), p_2\big(\gamma(t)\big)\Big)\d t,  
\end{align}
where $\Gamma$ is the set of positive diffeomorphisms $\Gamma = \big\{\gamma\in {\rm Diff}([0, 1]): \gamma(0) = 0, \ \gamma(1) = 1\big\}.$ However, due to the fact that the distance function \eqref{eqn:metric1} is not invariant under the group of re-parameterization $\Gamma,$ the objective function \eqref{eqn:metric2} has the vanishing effect (\cite{marron2014, Michor2005, Michor2003RiemannianGO}). Sobolev metrics (see \cite{Michor2007AnOO, Sundaramoorthi05sobolevactive, Younes1998ComputableED}) overcome these issues but are not always easy to compute. In \cite{Younes1998ComputableED},  the importance of the fact that the distance function should be invariant under the Lie group of actions was elaborated.

A few recent studies start to develop more appropriate metrics and distance functions to perform FDA on manifolds. Leveraging advancements in shape analysis, 
\cite{su2014} extended the square-root velocity curve in the Euclidean space \cite{Srivastava_2011} to a manifold by parallel transporting all square-root velocity vectors to some reference point $c\in\mathcal{M}$ along geodesic paths
\begin{align*}
 q(t) = \left( \frac{\dot{p}(t)}{\sqrt{|\dot{p}(t)|}} \right)^{||}_{p(t) \rightarrow c}.
\end{align*} 
Although this generalization makes the distance function invariant under $\Gamma,$ the random choice of the reference point $c\in\mathcal{M}$ can introduce distortions and uncertainty into the analysis.  For example, in Figure \ref{fig:comp_su}, we show how the reference point can affect the computed distance between given functional data or trajectories on $\S^2$ using the method in \cite{su2014}. To overcome {these limitations}, more intrinsic methods were proposed in \cite{ZhengwuZhang2018, zhang2018rate, le2019discrete, brigant2016computing}. {Compared with \cite{su2014}, these new methods represent each functional data as a pair, its starting point and  its speed vector field renormalized by the square root of its norm, to 1) reduce the distortion brought by parallel translating to a far away tangent space (of the reference point), and 2) to avoid the need of choosing an arbitrary reference point. With such representation, the manifold-valued functions are considered as elements of an (infinite-dimensional) manifold, and then they equip it with a Riemannian metric that can be invariant with respect to re-parameterization of functions.}  Panel b in Figure \ref{fig:comp_su} demonstrates the advantage of  \cite{ZhengwuZhang2018} over \cite{su2014} using eight simulated trajectories. However, due to the complexity of the proposed metrics and the manifold itself, these methods all have significant computational challenges to calculate geodesics and amplitude mean after alignment.

\begin{figure}[h]
\centering
\includegraphics[width=0.9\textwidth]{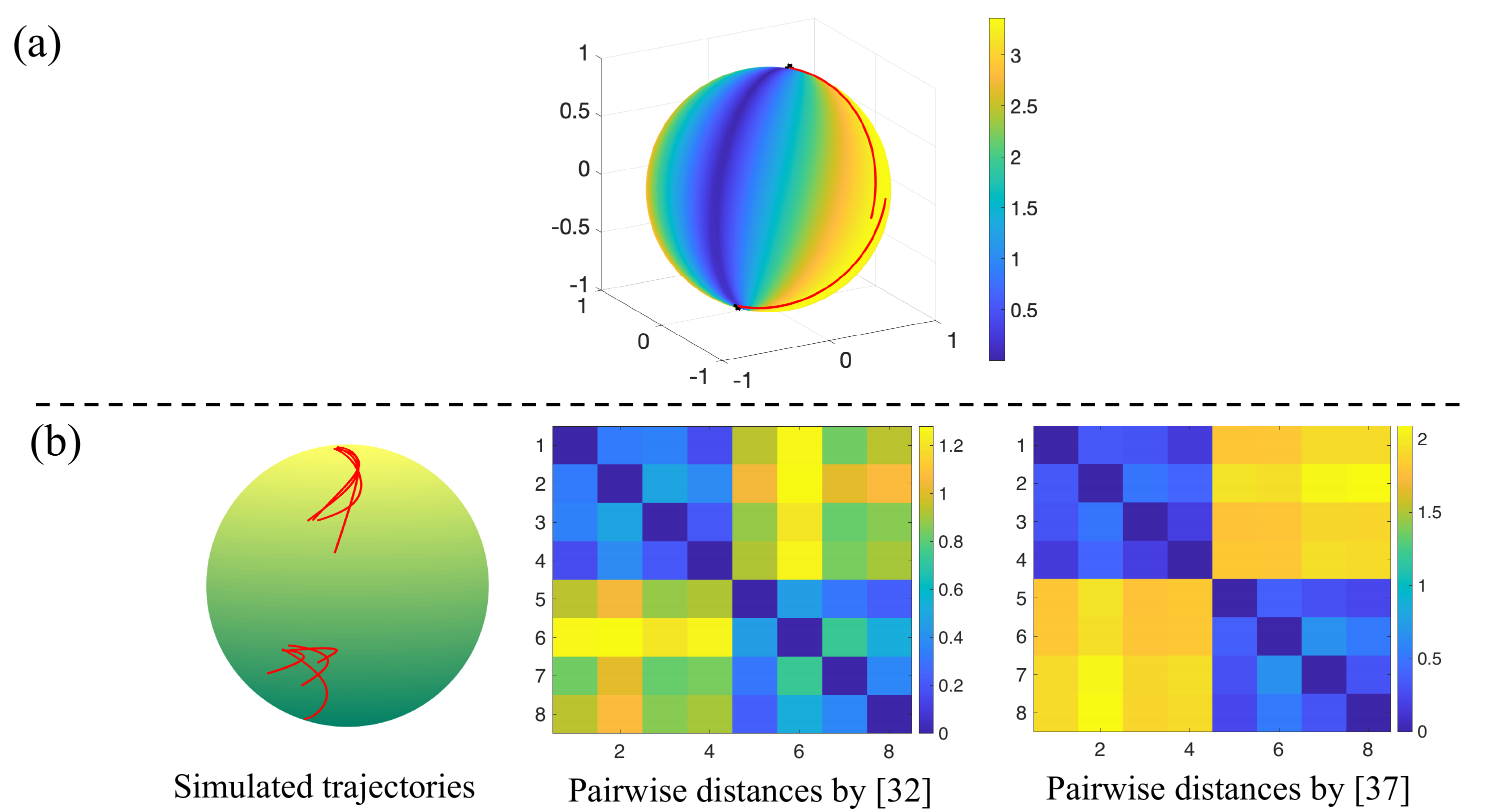}
\caption{Panel (a) shows the distance between two red trajectories at different reference point computed by the method in \cite{su2014}. We sampled 100*100 points on the sphere as reference points and mark the distance at each reference point using the color. Panel (b) compares \cite{su2014} with \cite{ZhengwuZhang2018} using pairwise geodesic distances between eight simulated trajectories.}
\label{fig:comp_su}
\end{figure}

In this work, we adopt the representation framework in \cite{ZhengwuZhang2018} to inherit its advantages in analyzing FDA on manifolds, and develop a set of efficient and accurate tools for analyzing the amplitude mean of functional data on $\S^2$. The most important contribution of this paper is a gradient descent algorithm for more effectively computing the geodesic between the amplitudes of two functions on  $\S^2$. To achieve this, we have to invent novel computational tools and algorithms.  More specifically, compared with \cite{ZhengwuZhang2018}, this paper's novel contributions include:  1) we explicitly derive the analytical formulas for parallel transport along any circular arc $\beta$, where in \cite{ZhengwuZhang2018} the parallel transports along $\beta$ are computed by numerical approximation; 2) we develop an efficient  and accurate gradient descent method to obtain the geodesic between two elements on ${\mathbb{B}}$ (here ${\mathbb{B}}$ denotes the set of all functional data in our framework, see its definition in section 2);  and  3) we utilize the semi-linearity of ${\mathbb{B}}$ to simplify the geodesic calculation. 
The proposed new tools have been implemented in Matlab (source code can be found on GitHub \url{https://github.com/Bayan2019/2DSphericalTrajectories}), and compared with the tools in \cite{ZhengwuZhang2018} in both simulated and real-world data to demonstrate their advantages. 

The rest of the paper is organized as follows. In sections 2 and 3 the Riemannian geometry of smooth functions on $\mathbb{S}^2$ is introduced. It is worth to notice that Theorem 1 is an extended version of Theorem 1 in \cite{ZhengwuZhang2018}.   In section 4, we show how to calculate the sample Fr\'echet mean for a set of functional data and their amplitudes.  {Section 5 discusses the covariance function computation on tangent space of the Fr\'echet mean. In sections 6 and 7} we present simulation studies and real data analyses, respectively. Section 8 concludes the paper. 

\section{Riemannian Geometry of Smooth Functions on $\S^2$}

Let $p:[0,1] \mapsto \S^2$ denote a smooth function on the 2-sphere, and let the set of all such functions be denoted as $\mathcal{F} = \{p: [0, 1]\mapsto \S^2| p \text{ is smooth}\}$. Also, $\Gamma$ is the set of positive diffeomorphisms $\Gamma = \big\{\gamma \in {\rm Diff}([0, 1]): \gamma(0) = 0, \ \gamma(1) = 1\big\}.$ $\Gamma$ forms a group action under the composition: $\mathcal{F} \times \Gamma \mapsto \mathcal{F}$ according to $p \circ \gamma \mapsto p(\gamma)$, where $p$ and $p(\gamma)$ follow the same trajectory on $\S^2$ but with different phase (temporal evaluation speed). We represent the function $p$ using its starting point and a transported square-root vector curve (TSRVC):
\begin{equation*}
 x = p(0)\in \S^2, \qquad \quad q(t) = \left(\frac{\dot{p}(t)}{\sqrt{|\dot{p}(t)|}}\right)_{p(t) \rightarrow x = p(0)}^{||} \in T_x \S^2,
\end{equation*} 
where $(v)^{||}_{p(t) \rightarrow x}$  represents parallel transport of vector $v \in T_{p(t)} \S^2$ to $T_{x}\S^2$, and the parallel transport is done along the curve itself $p(t).$ Note that $q(t)$ is a function in $T_{x}\S^2$.  We represent $p$ using a pair $(x, q)$, illustrated in Figure \ref{fig:rep_manifold} (a). The TSRVC representation is bijective: any $p \in \mathcal{F}$ can be uniquely represented by a pair $(x, q)$ and we can reconstruct $p$ from $(x, q)$ using covariant integral \cite{zhang2018rate}. When it is convenient, we use $p = (x, q)$ for notation simplicity and we will explicitly point out the notation change. 

\begin{figure}
\begin{tabular}{c|ccc}
 \includegraphics[width=0.28\linewidth]{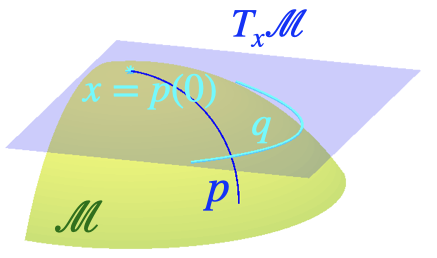} & & & \includegraphics[width=0.5\textwidth]{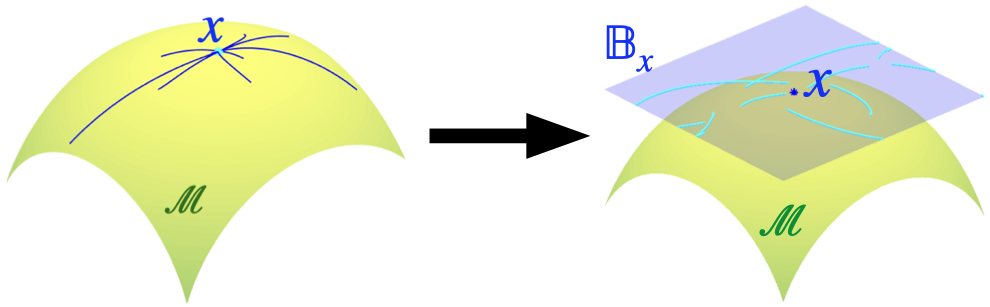} \\
 (a) & & & (b) 
\end{tabular}
\caption{Panel (a) illustrates using $(x, q)$ to represent a function $p$ on $\S^2$, and panel (b) shows the space ${\mathbb{B}}_x$ containing all absolutely continuous functions that start at $x$. Here we have $\mathcal{M} = \S^2$.}
\label{fig:rep_manifold}
\end{figure}

First of all, for any $x \in \S^2$ we have ${\mathbb{B}}_x = \mathbb{L}_2([0, 1], T_{x}\S^2)$ to represent all absolutely continuous functions or trajectories on $\S^2$ that start at $x.$ Namely, we consider a smooth trajectory $p:[0, 1] \rightarrow \S^2$ that starts at $p(0) = x$ as an element in ${\mathbb{B}}_x$
\begin{equation*}
 q(t) = \left(\frac{\dot{p}(t)}{\sqrt{|\dot{p}(t)|}}\right)_{p(t) \rightarrow x}^{||} \in {\mathbb{B}}_x.
\end{equation*}
So the full space of interest becomes a bundle
\begin{equation*}
 {\mathbb{B}} = \bigsqcup_{x\in \S^2}{\mathbb{B}}_x,
\end{equation*}
which is similar to the tangent bundle $T\S^2$ but instead of tangent vectors we consider square-integrable functions on tangent spaces. For the manifold ${\mathbb{B}}$ we have the following tangent space
\begin{equation*}
 T_{(x, q)}{\mathbb{B}} = T_x\S^2 \bigoplus {\mathbb{B}}_x,
\end{equation*}
and define the inner product 
\begin{equation}\label{eqn:t-metric}
 \big\langle(u_1, w_1), (u_2, w_2)\big\rangle_{(x, q)} = \langle u_1, u_2\rangle_x + \int_0^1\langle w_1(t), w_2(t)\rangle_x\d t,
\end{equation}
where $(u_i, w_i) \in T_{(x, q)}{\mathbb{B}}$, and $\langle\cdot, \cdot\rangle_x$ is a scalar product in $T_x\S^2$. {The inner produce defines a simple $L_2$ metric on ${\mathbb{B}}$.}

Our main focus is the amplitude of functions, i.e., the quotient space ${\mathbb{B}}/\Gamma.$ However, since the distance function on $ {\mathbb{B}}/\Gamma$ is determined through the distance function on $ {\mathbb{B}},$ we have to review the geometry of $ {\mathbb{B}}$ first. As for all Riemannian manifolds, the distance between two elements $(x_0, q_0), (x, q) \in {\mathbb{B}}$ is determined by the length of geodesic curve, i.e., the shortest path 
\begin{align*}
 \big(\beta(s), q(s, \cdot)\big)\in  {\mathbb{B}}
\end{align*}
connecting $(x_0, q_0)$ and $(x, q)$ with
\begin{align*}
 \big(\beta(0) = x_0, \quad q(0, \cdot) = q_0\big) \quad {\rm and} \quad \big(\beta(1) = x, \quad q(1, \cdot) = q\big).
\end{align*}

\begin{prop}\label{prop1}
If a path $\big(\beta(s), q(s, \cdot)\big)$ is geodesic on $ {\mathbb{B}}$ under the metric \eqref{eqn:t-metric}, then it satisfies the following properties:
\begin{enumerate}
 \item the base-curve $\beta(s)$ is constant-speed parameterized;
 \item TSRVC $q(s, \cdot)$ is covariantly linear along $\beta,$ and that is, $\nabla_{\dot{\beta}(s)} \big(\nabla_{\dot{\beta}(s)} q(s, t)\big) = 0.$
\end{enumerate}
\end{prop}
\begin{proof}
The proof can be found in \cite{GasserGervini2004}.
\end{proof}

If a path on $ {\mathbb{B}}$ satisfies the two properties listed above, then it is completely determined by the curve $\beta \in \S^2$, which is called the {\it base-curve} in the following context. That is, once we fix $\beta(s)$, we have $q(s, t) = q_{0, \beta(s)}^{||}(t) + s\big(q_{\beta(s)}^{||}(t) - q_{0, \beta(s)}^{||}(t)\big),$ where $q_{0, \beta(s)}^{||}(t)$ is the parallel transport of $q_0(t)$ from $\beta(0)$  to $\beta(s)$ along $\beta$ for $s\in[0, 1]$, and $q_{\beta(s)}^{||}(t)$ is the parallel transport of $q(t)$ from $\beta(1)$ to $\beta(s)$ along $\beta$. Moreover, the length of this path is defined as
\begin{equation}\label{eqn:dbeta}
 d_\beta(p_0, p) = d_{\beta}((x_0, q_0), (x, q)) = \sqrt{\ell_{\beta}^2 + \int_0^1\big|q_{0, \beta(1)}^{||}(t) - q(t)\big|^2\d t},
\end{equation}
where $\ell_{\beta}^2 = \Big(\int_0^1 \sqrt{\big\langle\dot{\beta}(t), \dot{\beta}(t) \big\rangle_{\beta(t)}}\d t\Big)^2 = |\dot{\beta}(0)|^2.$ If a path $(\beta(s), q(s, \cdot))$ is not a geodesic on $ {\mathbb{B}}$, we can still use equation (\ref{eqn:dbeta}) to quantify its length. If $\beta$ is unknown, we find the geodesic distance according to
\begin{equation}\label{eqn:metric3}
 d_{\mathbb{C}}(p_0, p) = \min_{\beta} d_{\beta}\big(p_0, p\big).
\end{equation}

The following theorem presents the main advantage of using TSRVC.

\begin{theorem}\label{thm2}
 For any two trajectories $p_0, p  \in \mathcal{F}$, and their corresponding TSRVC representations $(x_0, q_0), (x, q) \in {\mathbb{B}},$ any smooth base-curve $\beta,$ and any re-parameterization $\gamma\in\Gamma$, there is equality
 \begin{align*}
 d_{\beta}\big(p_0, p\big) = d_{\beta}\big(p_0(\gamma(t)), p(\gamma(t))\big) = d_{\beta}\Big(\big(x_0, q_0(\gamma(t)) \sqrt{\dot{\gamma}(t)}\big), \big(x, q(\gamma(t)) \sqrt{\dot{\gamma}(t)}\big)\Big).
 \end{align*}
\end{theorem}

The proof of this theorem is presented in the supplement. This theorem highlights the advantage of using TSRVC representation to study functional data or trajectories on $\S^2$: the action of $\Gamma$ on $ {\mathbb{B}}$ under the metric $d_{\beta}$ is by isometries. The isometry property of re-parameterization allows us to focus on the amplitude of a function and analyze it in a manner that is invariant to random time warping. 

To represent the amplitude of a function, we first introduce a closed set $\tilde{\Gamma}$, containing all non-decreasing, absolutely continuous functions $\gamma$ on $[0, 1]$ such that $\gamma(0) = 0$ and $\gamma(1) = 1$. It can be shown that $\Gamma$ is a dense subset of $\tilde{\Gamma}$ and the orbit of a TSRVC under $\tilde{\Gamma}$, defined as $[p] := (x, [q]) = \big\{\big(x, q(\gamma) \sqrt{\dot{\gamma}}\big) \big| \gamma \in \tilde{\Gamma}\big\}$, forms a closed set (which is not the case under $\Gamma$). For theory development, we utilize $\tilde{\Gamma}$, and define the function amplitudes as the set of orbits under the group action $\tilde{\Gamma}$: $ {\mathbb{B}}/\tilde{\Gamma} = \{(x, [q]) | (x, q) \in  {\mathbb{B}}\}$. In practice, we will only use elements in $\Gamma$ for simplicity. According to Theorem \ref{thm2}, we define the distance between amplitudes on $ {\mathbb{B}}/\tilde{\Gamma}$ as
\begin{equation}\label{eqn:unparametic}
 d_{ {\mathbb{B}}/\tilde{\Gamma}}([p_0], [p]) = \min_{\gamma_1, \gamma_2 \in \tilde{\Gamma}} d_{ {\mathbb{B}}}(p_0\circ\gamma_1, p\circ\gamma_2) \approx \min_{\gamma \in {\Gamma}} d_{ {\mathbb{B}}}(p_0, p\circ\gamma) \quad .
\end{equation}

\section{Geodesic Computation between Smooth Functions on $\S^{2}$}

We now study how to obtain the geodesic between two elements on $ {\mathbb{B}}$ and $ {\mathbb{B}}/\tilde{\Gamma}$. Proposition \ref{prop1} describes two important properties about the geodesic path, which will help us to find the geodesic.

\subsection{Geometry of base-curve $\beta$}
\label{sec:basecurve}

To understand the geometry of base-curve $\beta$, we start with introducing the concept named $p$-{\it optimal curves}. A curve $\beta$ connecting two points $x_0$ and $x$ on $\S^2$ is called $p$-{optimal} if $\ell_{\beta}\leq\ell_{\eta}$, where $\ell_{\eta}$ is the length of $\eta$, and $\eta$ is any curve that connects $x_0$ and $x$ with the parallel transport map from $T_{x_0}\S^2$ to $T_{x}\S^2$ equal to the parallel transport map along $\beta$ from $T_{x_0}\S^2$ to $T_{x}\S^2$ (see the detailed definition in \citep{ZhengwuZhang2018}). With this $p$-optimal concept in mind, we have the following lemmas. 
\begin{lemma}\label{lemma1}
If a path $\big(\beta(s), q(s, \cdot)\big)$ on $ {\mathbb{B}}$ is a geodesic, then the base-curve $\beta$ is a $p$-optimal curve.
\end{lemma}
\begin{proof}
Let us assume that we can find a shorter curve than $\beta$ from $x_0$ to $x$ that induces the same parallel transport from $T_{x_0}\S^2$ to $T_{x}\S^2,$ then we can reduce $\ell_{\beta}$ without affecting $\int_0^1 \big|q_{0, \beta(1)}^{||}(t) - q(t)\big|^2\d t.$ Hence $\big(\beta(\cdot), q(\cdot, t)\big)$ does not have the shortest distance, and therefore it would not be a geodesic on $ {\mathbb{B}}.$ So there is a contradiction.
\end{proof}

\begin{lemma}\label{lemma2}
For any two points $x_0, x\in\mathbb{S}^2,$ the only $p$-optimal curves connecting them on $\mathbb{S}^2$ are the circular arcs.
\end{lemma}
\begin{proof}
Let $\beta$ be the curve joining $x_0$ with $x$ on $\S^2, $ and $\mathcal{P}\big(\beta\big)_0^1 : T_{x_0} \mathbb{S}^2\rightarrow T_{x}\mathbb{S}^2$ be the parallel translation map induced by $\beta$. Let a circle passing $x_0$ and $x$ on $\S^2$ be $\zeta = \zeta_1\cup\zeta_2$, where $\zeta_1$ is the circular arc from $x_0$ to $x$ that induces the same parallel transport $\mathcal{P}\big(\beta\big)_0^1$ as $\beta.$ The Gauss Bonnet theorem states that the angle of rotation of the parallel transport map $T_{x_0}\mathbb{S}^2 \rightarrow T_{x_0}\mathbb{S}^2$ induced by the closed curve $\zeta$ is equal of the integral of the Gaussian curvature over the region enclosed by the loop $\zeta.$ Since the Gaussian curvature of $\mathbb{S}^2$ equals $+1$ at every point, the angle of rotation is equal to the area enclosed by $\zeta.$ So the circle $\zeta$ has the shortest distance among all loops with the same parallel transport map. Therefore $\ell_{\beta} + \ell_{\zeta_2} \geq \ell_{\zeta_1} + \ell_{\zeta_2}.$ Hence $\beta = \zeta_1.$
\end{proof}

\begin{figure}
\begin{center}
\begin{tabular}{ccccc}
 \includegraphics[width=0.25\textwidth]{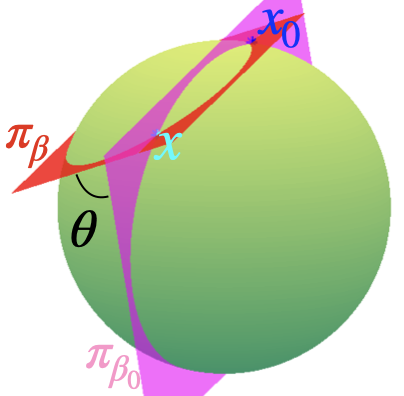} & \qquad & \raisebox{5.5\height}{\larger[7]$\Rightarrow$} & \qquad &
 \includegraphics[width=0.25\textwidth]{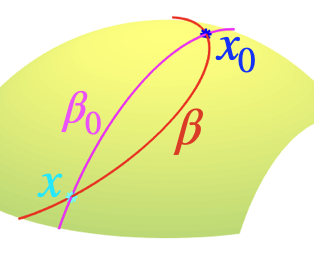} \\
 (a) & & & & (b)
\end{tabular}
\end{center}
\caption{Illustration of $p$-optimal curve $\beta$ connecting $x$ and $x_0$. Panel (a) shows two planes $\pi_{\beta}$ and $\pi_{\beta_0}$ containing $x_0$ and $x$ on $\S^2$. The interaction between $\S^2$ and $\pi_{\beta}$ contains a circular arc $\beta$ and the interaction between $\S^2$ and $\pi_{\beta_0}$ gives $\beta_0$, which is the geodesic on $\S^2$ between $x_0$ and $x$. Panel (b) shows $\beta$ and $\beta_0$.}
\label{fig:circarc}
\end{figure}

The two lemmas significantly reduce the search space for the base-curve $\beta$ in the geodesic calculation between two elements $(x_0, q_0)$ and $(x, q)$ on $ {\mathbb{B}}$, i.e., the base-curve $\beta$ must be among the circular arcs connecting $x_0$ and $x$ on $\S^2$. 

Now let us consider how to construct circular arcs connecting $x_0$ and $x$ on $\S^2$. We start with constructing a plane $\pi_\beta$ containing $x_0$ and $x$ (refer to the red plane in Figure \ref{fig:circarc} panel (a)). The plane $\pi_{\beta}$ is determined by its normal vector $n$ (a vector orthogonal to $\pi_{\beta}$). We construct vector $n$ based on an angle $\theta\in\left[-\frac{\pi}{2}, \frac{\pi}{2}\right]$ between $n$ and the cross product of $x_0$ and $x$, defined as $x_0 \times x$.  When we vary $\theta$, we get different $n$, and thus $\pi_\beta$, and the intersection between $\pi_\beta$ and $\S^2$. $n$ has the following explicit expression as a function of $\theta$:
\begin{equation*}
 n(\theta; x_0, x) = \frac{x_0 + x}{|x_0 + x|}\sin\theta + \frac{x_0\times x}{|x_0\times x|}\cos\theta.
\end{equation*}
The intersections  between $\pi_\beta$ and $\S^2$ give $p$-optimal curves with the following explicit expression:
\begin{equation} \label{eqn:para:beta}
\beta(s; \theta) = x_0 \cos(s\phi) + (n\times x_0) \sin(s\phi) + n \langle n, x_0\rangle \big(1 - \cos(s\phi)\big),
\end{equation}
where $\phi = 2 \arcsin\left(\frac{1}{2} \sqrt{\frac{2 - 2\langle x_0, x\rangle}{1 - \langle n, x_0\rangle^2}}\right)$, and $n$ is determined by $\theta$. For notation simplicity, $\beta(s; \theta)$ is also denoted as $\beta(\theta)$ or $\beta$ when there is no confusion.  Figure \ref{fig:circarc} panel (b) illustrates a circular arc $\beta$ and the geodesic $\beta_0$ between $x_0$ and $x$ on $\S^2$. Note that according to equation \eqref{eqn:para:beta}, the circular arcs between $x_0$ and $x$ are parameterized by a single parameter $\theta$ and have a closed form solution, which allows us to derive an efficient optimization algorithm to find the geodesic base-curve $\beta$. 

\subsection{Geodesic calculation on $ {\mathbb{B}}$}
\label{subsec:distance}

For two smooth functions $p_0$ and $p$ on $\S^2$, their geodesic distance can be calculated using formula \eqref{eqn:metric3}, and the geodesic is completely determined by the base-curve $\beta$. With the help of the two lemmas in section \ref{sec:basecurve}, we have significantly reduced the search domain of $\beta$ and can represent it using only one parameter $\theta$. Now we rewrite the geodesic distance calculation as a function of $\theta$: 
\begin{equation*}
 d^2_{ {\mathbb{B}}}(p_0, p) = \min_{\theta \in\left[-\frac{\pi}{2}, \frac{\pi}{2}\right]} d^2(p_0, p, \theta) = \min_{\theta \in\left[-\frac{\pi}{2}, \frac{\pi}{2}\right]} \ell_{\beta(\theta)}^2 + \int_0^1 \Big|q_{0,\beta(\theta)}^{||}(t) - q(t)\Big|^2 \d t.
\end{equation*}
One very attractive property of the $p$-optimal base-curve constructed according to equation \eqref{eqn:para:beta} is that we have explicit expressions for $\ell_{\beta(\theta)}$ and the parallel transport $q_{0, \beta(\theta)}^{||}(t)$ as functions of $\theta$. We refer the readers to the Propositions 1 and 2 in the supplement for detailed derivations. 

To summarize, the distance function $d(p_0, p, \theta)$ depends on $\theta$ that determines the base-curve $\beta(\theta).$ Therefore, to get $d^2_{ {\mathbb{B}}}(p_0, p)$, we can obtain $\frac{\partial}{\partial\theta} d^2(p_0, p, \theta),$ and rely on the gradient descent method to find the optimal $\theta$. Fortunately, on $\S^2$, we also can obtain an explicit solution for the gradient, which is presented in section 4 in the supplement. We derive  Algorithm \ref{alg:parageo} to find  $\theta(p_0, p) = \argmin_{\theta \in \left[-\frac{\pi}{2}, \frac{\pi}{2}\right]} d^2(p_0, p, \theta).$

\begin{algorithm}[ht]\label{alg:parageo}
\caption{The method for finding geodesical $\theta(p_0, p)$} 
\KwInput{two smooth functions $p_0$ and $p$, $\theta_0 \in \left[-\frac{\pi}{2}, \frac{\pi}{2}\right],$ the gradient step $\lambda \in \mathbb{R},$ and the degree of accuracy $\varepsilon \in \mathbb{R}_+$}
\KwOutput{a point $\theta \in \left[-\frac{\pi}{2}, \frac{\pi}{2}\right]$}
\Begin{
Set $\theta = \theta_0$. \\
 \Repeat { \rm $\left|\frac{\partial}{\partial\theta}d^2(p_0, p, \theta)\right| < \varepsilon$ }
 {
  Compute $\bar{\theta} = \theta - \lambda \frac{\partial}{\partial\theta} d^2(p_0, p, \theta).$ \\
  \eIf{$\bar{\theta} \in \left[-\frac{\pi}{2}, \frac{\pi}{2}\right]$ \ {\rm and} \ $d^2(p_0, p, \bar{\theta}) < d^2(p_0, p, \theta)$}{ $\theta = \theta - \lambda \frac{\partial}{\partial\theta} d^2(p_0, p, \theta)$\;}{ $\lambda = \frac{\lambda}{2}$\;}
 }
}
\end{algorithm}

Finally, after finding the optimal $\theta(p_0, p)$, we build the geodesic path as
\begin{align*}
 \beta(s) =  \beta(s; \theta(p_0, p)) \qquad {\rm and} \qquad q(s, t) = & q_{0,\beta(s)}^{||}(t) + s \big(q_{\beta(1 - s)}^{||}(t) - q_{0, \beta(s)}^{||}(t)\big).
\end{align*} 
Figure \ref{fig:geodesics1} shows one example of $p_0$ and $p$ on $\S^2$, $d^2(p_0, p, \theta)$ and its derivative $\frac{\partial}{\partial\theta} d^2(p_0, p, \theta)$, the base-curve $\beta$ and the final geodesic path between $p_0$ and $p$. 

\begin{figure}
\begin{center}
\begin{tabular}{c|ccc}
 \includegraphics[width=0.35\textwidth]{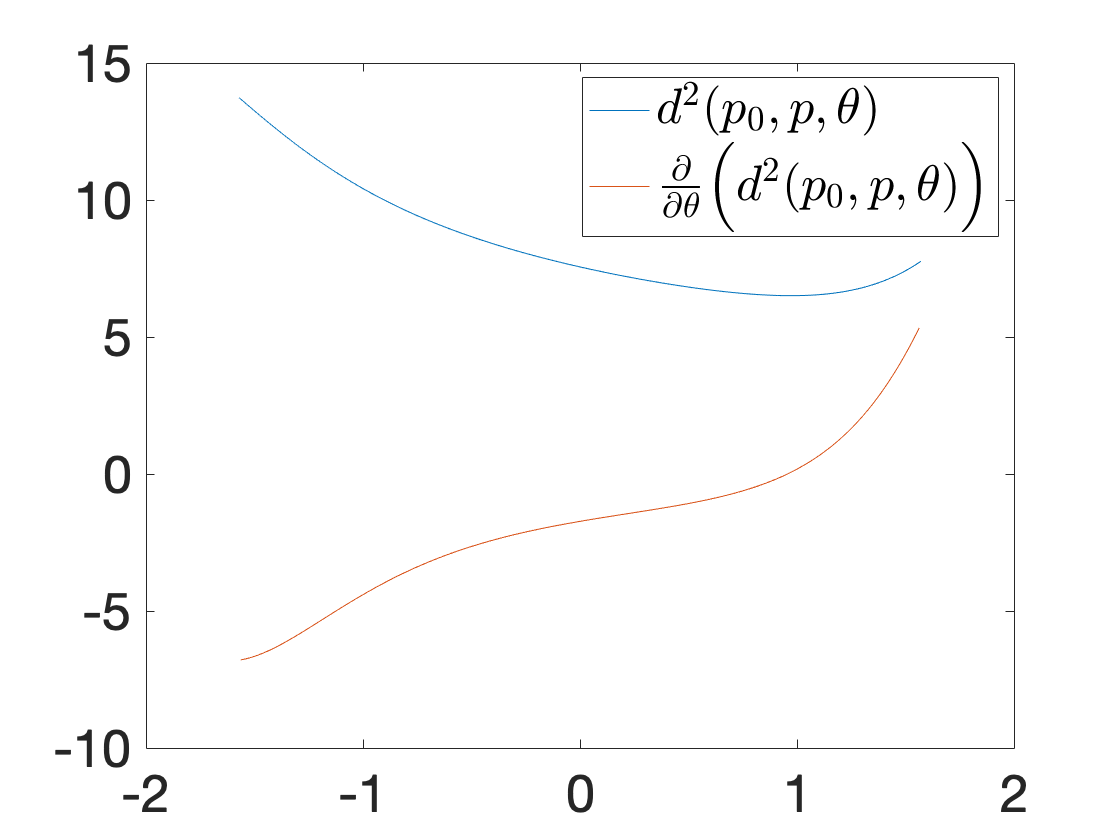} &
 \includegraphics[width=0.25\textwidth]{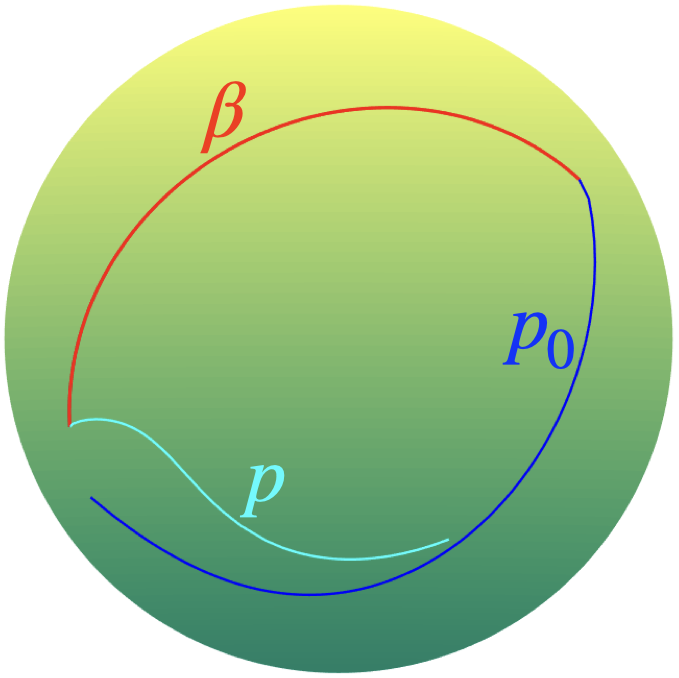} & &
 \includegraphics[width=0.25\textwidth]{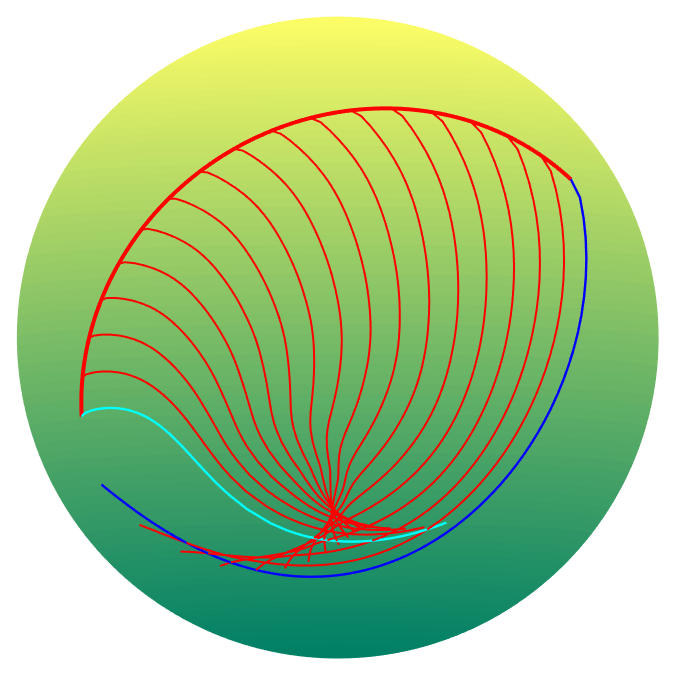} \\
 (a) &  & (b) &  
\end{tabular}
\end{center}
\caption{Panel (a) shows $d^2(p_0, p, \theta)$ with its derivative $\frac{\partial}{\partial \theta} d^2(p_0, p, \theta)$. Panel (b) left side shows the base-curve $\beta$ which determines the geodesic on $ {\mathbb{B}}$ between $p_0$ and $p$, and the right side shows the geodesic between $p_0$ and $p$.}
\label{fig:geodesics1}
\end{figure}

\subsection{The exponential and inverse exponential maps on $ {\mathbb{B}}$}
\label{subsec:InverseExponential}

For statistical analysis of a set of observations on $ {\mathbb{B}}$, two other important tools are the exponential and inverse exponential maps. Let $(u, w) \in T_{(x, q)} {\mathbb{B}}$, the exponential map $\exp_{(x, q)}(s(u, w))$ is a mapping from $T_{(x, q)} {\mathbb{B}}$ to $\C$, and gives a geodesic  $(\beta(s), q(s, t))$ on $\C$, where $s \in [0, 1]$.   Given $p_0$ and $p$ on $\S^2$, and their TSRVC representations $(x_0, q_0)$ and $(x, q)$, the inverse exponential map $ \exp^{-1}_{(x_0, q_0)}(x, q)$ gives a tangent vector $(u, w)$ on $T_{(x_0, q_0)} {\mathbb{B}}$, such that $\exp_{(x_0, q_0)}((u, w)) = (x, q)$. Since we consider $p$ and $(x, q)$ as the same, for notation simplicity $\exp_p(s(u, w))$ is also used to denote the exponential map from $T_p {\mathbb{B}}$ to $ {\mathbb{B}}$ and $\exp_{p_0}^{-1}p$ to denote for the inverse exponential map from $ {\mathbb{B}}$ to $T_{p_0} {\mathbb{B}}$.

On $\C$, once we find the optimal geodesic parameterized by $\theta \in \left[-\frac{\pi}{2}, \frac{\pi}{2}\right]$ for $p_0$ and $p$ by Algorithm \ref{alg:parageo}, it is straightforward to get the inverse exponential map
\begin{align} \label{eqn:InverseExp}
 \exp_{p_0}^{-1}p =  & \big(\dot{\beta}(0; \theta), q^{||}_{\beta} - q_0\big),
\end{align}
where $\beta$ is determined by $\theta$. Figure \ref{fig:InverseExponentialMap} shows one example of $p_0$ and $p$ on $\S^2$ and the inverse exponential map $\exp_{p_0}^{-1}p$. 
\begin{figure}
\begin{center}
\begin{tabular}{ccccc}
 \includegraphics[width=0.2\textwidth]{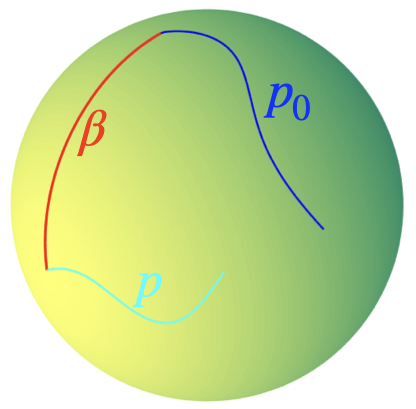} & 
 \raisebox{6.5\height}{\scalebox{2}{$\Rightarrow$}} &
 \includegraphics[width=0.23\textwidth]{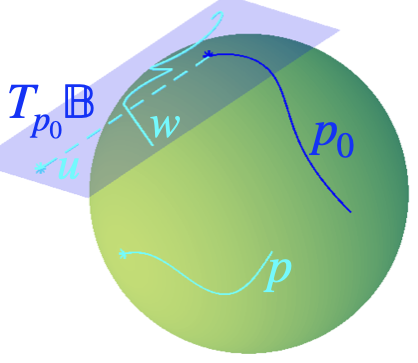} & 
 \raisebox{6.5\height}{\scalebox{2}{$\simeq$}} &
 \includegraphics[width=0.35\textwidth]{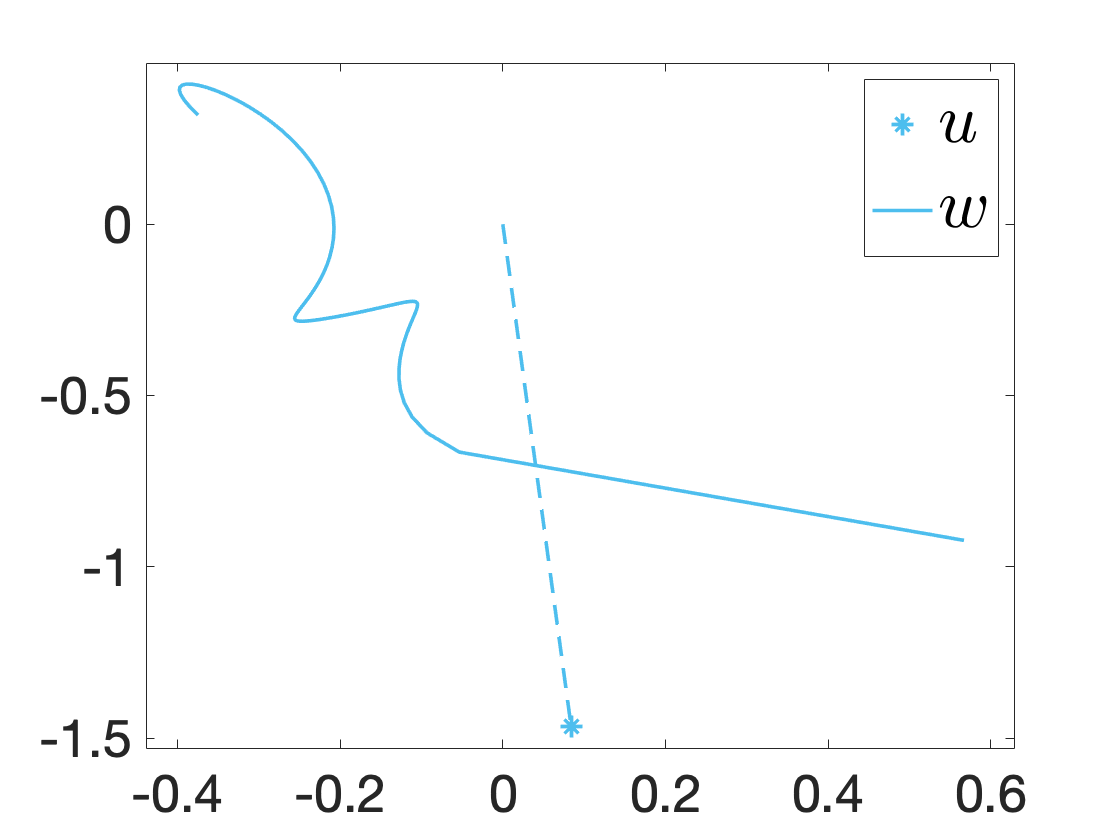}\\
 (a) &  &  & (b) &  
\end{tabular}
\end{center}
\caption{Panel (a) shows the base-curve $\beta$ which determines the geodesics on $ {\mathbb{B}}$ between two functions $p_0$ and $p,$.  Panel (b) shows $(u, w) = \exp_{p_0}^{-1}p$ on the tangent space $T_{x_0}\mathbb{S}^2.$}
\label{fig:InverseExponentialMap}
\end{figure}

The exponential mapping $\exp_{p_0}((u, w))$ is harder to compute. Let $(\beta(s), q(s, \cdot)) = \exp_{p_0}(s(u, w))$ for $s \in [0,1]$. We need to construct a $(\beta(s), q(s, \cdot))$ that satisfies the two properties in Proposition \ref{prop1} and $\dot{\beta}(0) = u$. We describe our procedure to construct such $(\beta(s), q(s, \cdot))$ as follows.

For a given $u\in T_{x_0}\mathbb{S}^{2}$, we can construct a set of $p$-optimal curves  determined by a set of planes $\pi_{\beta}$ that intersect with $T_{x_0}\mathbb{S}^2$ by the line $x_0 + t u$ for $t\in\mathbb{R}$, where $\pi_{\beta}$ is determined by a normal vector
\begin{align}\label{eqn:normal1}
 n(\vartheta; x_0, u) = x_0\cos\vartheta + \frac{x_0 \times u}{|x_0 \times u|} \sin\vartheta \qquad {\rm for} \quad \vartheta \in \left(\arcsin\left(\frac{|u|}{\pi}\right), \pi - \arcsin\left(\frac{|u|}{\pi}\right)\right).
\end{align}
One can find out that
\begin{align}\label{eqn:exp1}
\begin{split}
 \beta(s; \vartheta) & = x_0 \cos(s\phi) + (n\times x_0) \sin(s\phi) + n\langle n, x_0\rangle \big(1 - \cos(s\phi)\big) \qquad \text{with} \quad \phi(\vartheta) = \frac{|u|}{\sin\vartheta}, \\
  q(s, t) &= \big(s w(t) + q_0(t)\big)_{\beta(s)}^{||}.
 \end{split}
\end{align}
It is easy to see that the above path \eqref{eqn:exp1} satisfies the two properties in Proposition \ref{prop1}. By setting $s = 1$, we get the end point $p = \exp_{p_0}\big((u, w)\big)$ with
\begin{align*} 
p = (x, q) \qquad \text{where} \quad
\begin{split}
 x & = x_0 \cos\phi + (n\times x_0) \sin\phi + n\langle n, x_0\rangle (1 - \cos\phi),\\
 q &= (w + q_0)_{\beta}^{||}.
\end{split}
\end{align*}

Figure \ref{fig:ExponentialMap1} panel (a) illustrates a set of planes $\pi_\beta$ that contain the vector $u\in T_x\mathbb{S}^2$ (parameterized by $\vartheta$ with normal vectors in \eqref{eqn:normal1}), and panel (b) shows $\beta(s; \vartheta)$ and  $p$ which are also parameterized by $\vartheta$.

\begin{figure}
\begin{center}
\begin{tabular}{ccc}
 \includegraphics[width=0.265\textwidth]{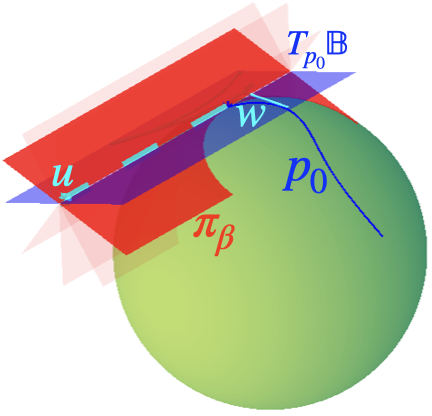} & \raisebox{4.5\height}{\larger[7]$\Rightarrow$} &
 \includegraphics[width=0.29\textwidth]{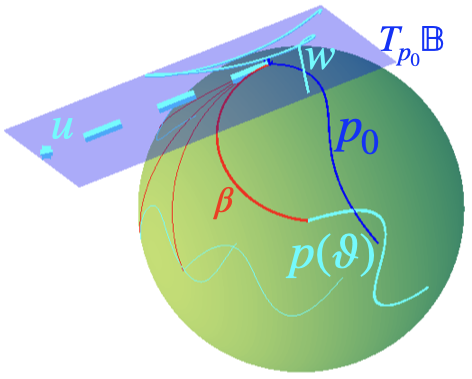}\\
 (a) & & (b) \\
\end{tabular}
\end{center}
\caption{Illustration on how to compute $\exp_{p_0}((u, w))$. Panel (a) shows a set of planes $\pi_\beta$ containing the vector $u\in T_x\mathbb{S}^2$, and panel (b) shows $\beta(s; \vartheta)$ (the interactions between $\pi_\beta$ and $\S^2$) and  $p$ which are also parameterized by $\vartheta$.}
\label{fig:ExponentialMap1}
\end{figure}

\begin{figure}
\begin{center}
\begin{tabular}{ccccc}
 \includegraphics[width=0.35\textwidth]{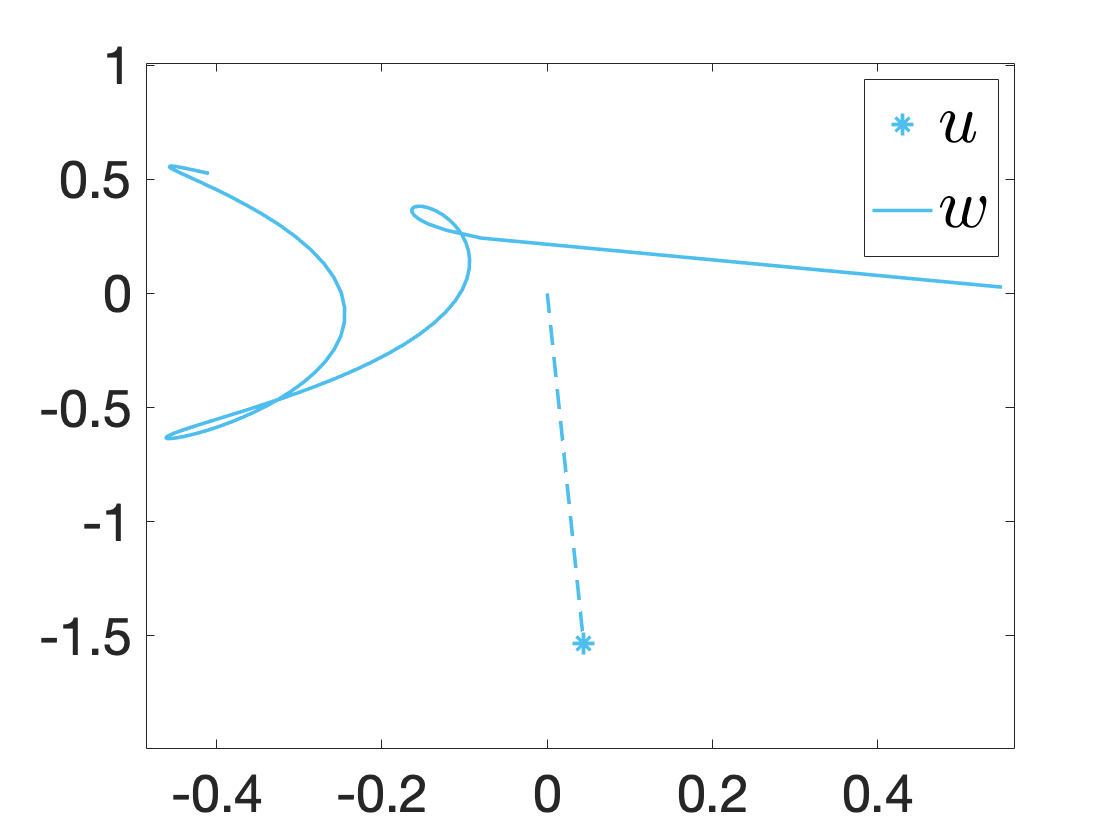}& 
 \raisebox{6.5\height}{\larger[4]$\simeq$}&
 \includegraphics[width=0.27\textwidth]{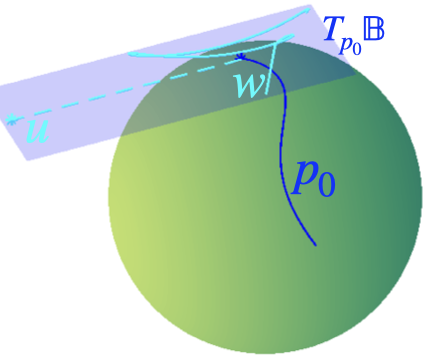}& 
 \raisebox{6.5\height}{\larger[4]$\Rightarrow$}&
 \includegraphics[width=0.2\textwidth]{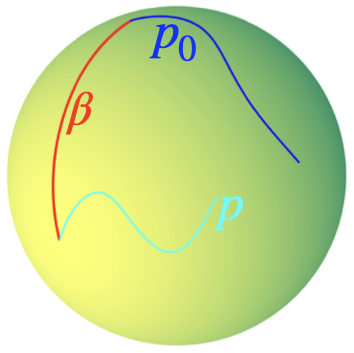}\\
  & (a) &  & & (b)
\end{tabular}
\end{center}
\caption{Panel (a) illustrates an element $(u, w)\in T_{p_0} {\mathbb{B}},$ where $u$ is a vector and  $w$ is a smooth function.  Panel (b) shows $p = \exp_{p_0}((u, w))$ and  the base-curve $\beta$.} 
\label{fig:ExponentialMap2}
\end{figure}

To this end the challenge is to find an appropriate $\vartheta,$ such that the angle $\theta(\vartheta) = \arcsin\left(\sqrt{\frac{2}{1 + \langle x_0, x\rangle}}\cos\vartheta\right)$ is optimal for $p_0$ and $p$, i.e., the base-curve $\beta$ gives the smallest distance between $p_0$ and $p$. We then derive a gradient descent method to find the optimal $\vartheta \in \left(\arcsin\left(\frac{|u|}{\pi}\right), \pi - \arcsin\left(\frac{|u|}{\pi}\right)\right)$ { in Algorithm 1 in the supplement}.

\subsection{Geodesic calculation between the amplitudes of functions on $\C/\Gamma$}
\label{subsec:ampldist}

Now we consider the amplitude of smooth functions on $\S^2$, which is defined as elements in the quotient space $ {\mathbb{B}}/\tilde{\Gamma}$, e.g., $[p_0], [p] \in  {\mathbb{B}}/\tilde{\Gamma}$ . Formula \eqref{eqn:unparametic} is used to calculate the geodesic distance between $[p_0]$ and $[p]$. In this case, we have a bivariate optimization problem: in addition to $\theta$, we need to find a time warping function $\gamma$ to minimize the distance between two orbits $[p_0]$ and $[p]$. We adopt a coordinate descent approach to iteratively solve the single variable optimization problems: (i) given $\theta$, we solve $\gamma$ using the Dynamic Programming algorithm in \cite{Bertsekas1995DynamicPA}: $ \gamma = \argmin_{\gamma\in\Gamma} d^2(p_0\circ\gamma, p, \theta)$; and (ii) given $\gamma$, we solve $\theta$ based on the Algorithm \ref{alg:parageo}.

We use following Algorithm \ref{alg:ampldist} to find the base-curve $\beta^*$ on $ {\mathbb{B}}/\Gamma$ and the optimal time warping $\gamma^*$ to align $p_0$ to $p$. 

\begin{algorithm}[ht]\label{alg:ampldist}
 \caption{The method for finding geodesical $\theta^{*}([p_0], [p])$ and $\gamma^*$ in $ {\mathbb{B}}/\Gamma$}
\KwInput{ 
 two smooth functions $p_0, p$, 
 $\theta_0\in\left[-\frac{\pi}{2}, \frac{\pi}{2}\right],$
 the gradient step $\lambda\in\mathbb{R},$ 
 and the degree of accuracy $\varepsilon\in\mathbb{R}$}
\KwOutput{a point $\theta^{*}\in\left[-\frac{\pi}{2}, \frac{\pi}{2}\right]$ and the re-parameterization $\gamma^{*}\in\Gamma$}
\Begin{
 Set $\theta^{*} = \theta_0$. Find $\gamma^{*} = \arg\min_{\gamma\in\Gamma} d^2(p_0, p, \theta^{*}).$ Then get $\tilde{p}_0 = p_0\circ\gamma^{*}.$\\

 \Repeat { \rm $\left|\frac{\partial}{\partial\theta} d^2(\tilde{p}_0, p, \theta^{*})\right| < \varepsilon$ }
 {
  Compute $\theta^{*}$ based on Algorithm \ref{alg:parageo} with input $(\tilde{p}_0, p, \theta^{*}, \lambda, \varepsilon)$.
	
  Find $\gamma^{*} = \argmin_{\gamma\in\Gamma} d^2(p_0, p, \theta^{*}).$ Update $\tilde{p}_0 = p_0\circ\gamma^{*}.$
 }
}
\end{algorithm}

\begin{figure}
\centering
\begin{tabular}{c|c c c}
 \includegraphics[width=0.25\textwidth]{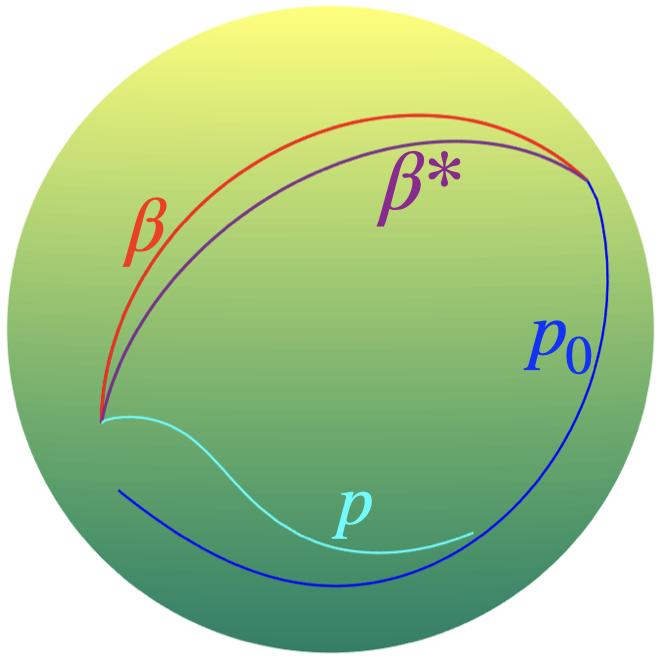} & 
 \includegraphics[width=0.25\textwidth]{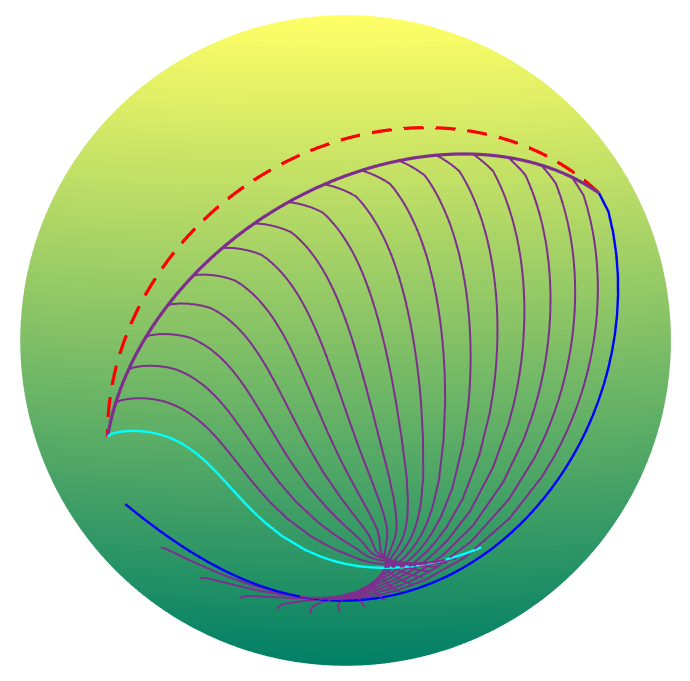} & &
 \includegraphics[width=0.35\textwidth]{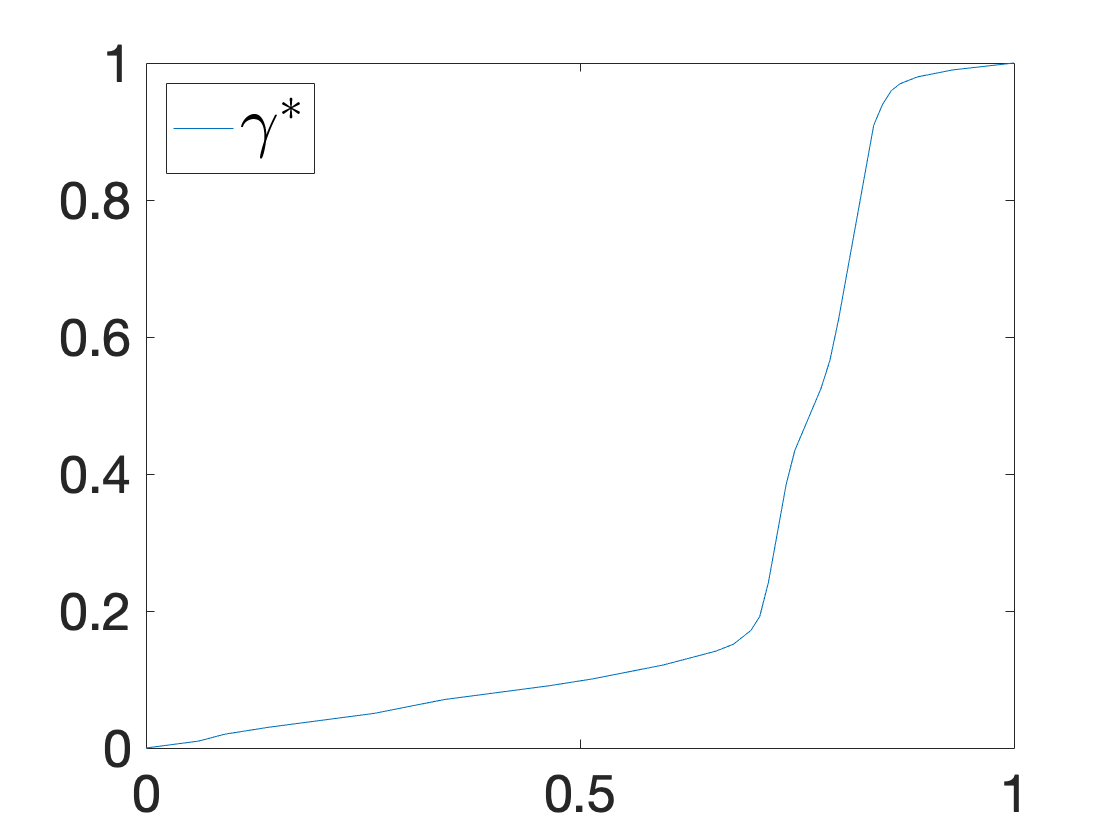}\\
 \includegraphics[width=0.25\textwidth]{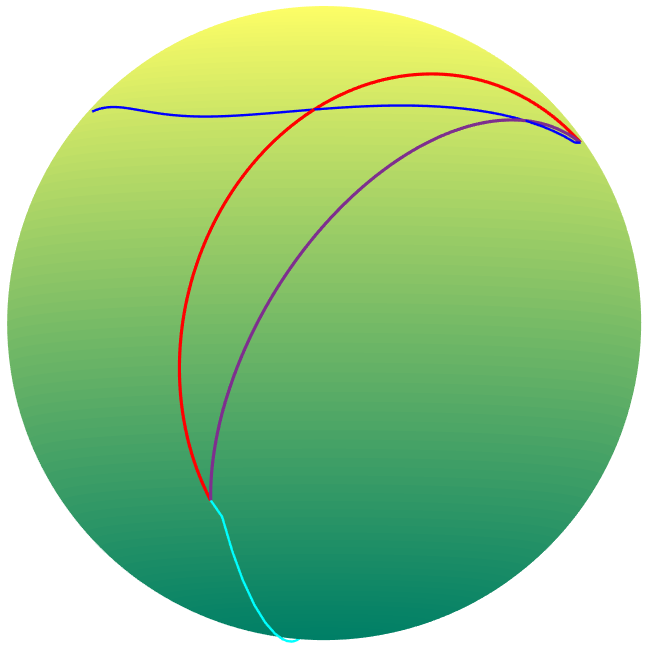} & 
 \includegraphics[width=0.25\textwidth]{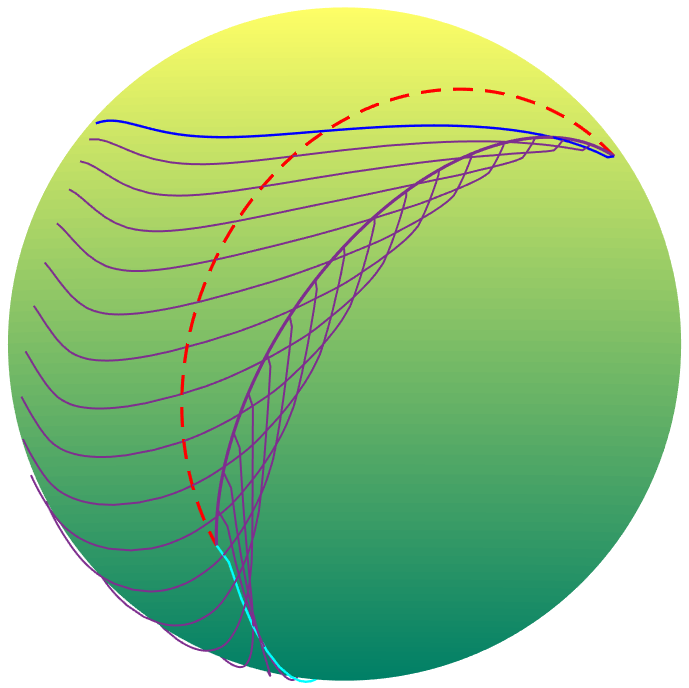} & &
 \includegraphics[width=0.35\textwidth]{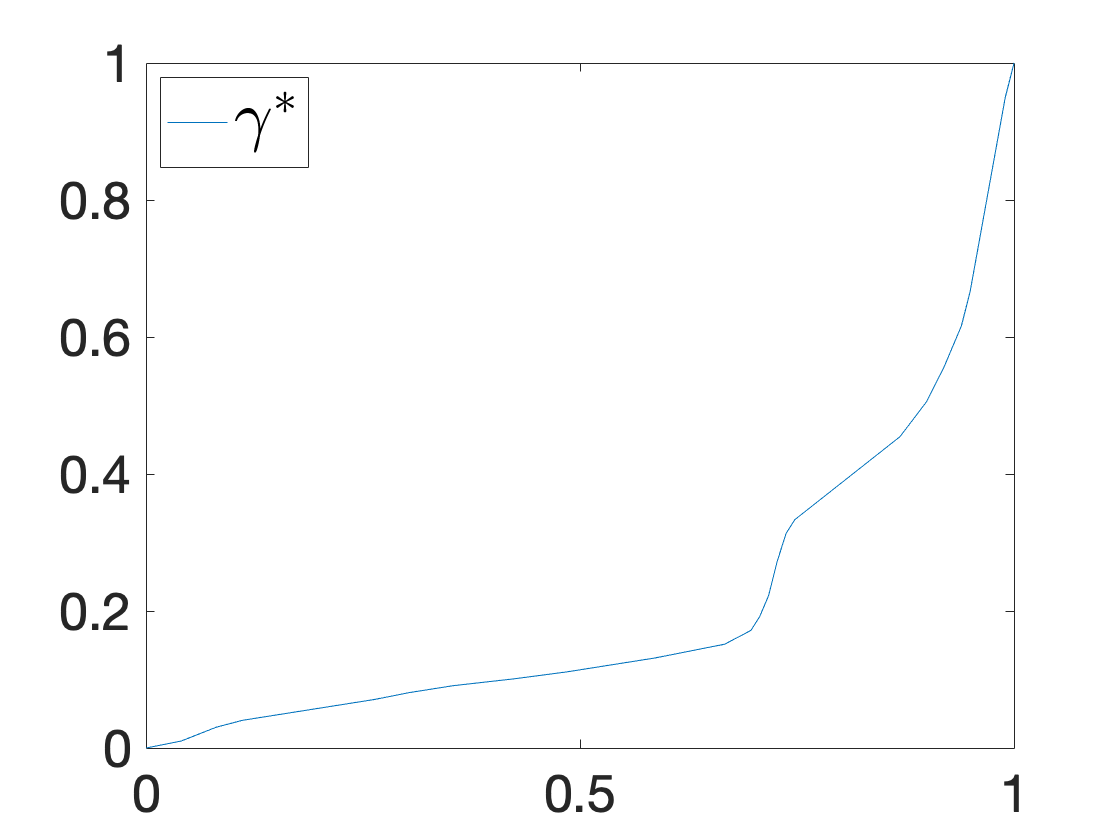}\\
 (a) &  & (b) &  
\end{tabular}
\caption{Comparison of base-curves $\beta$ (for geodesics on $\C$) and $\beta^*$ (for geodesics on $\C/\Gamma$).  Panel (a) shows two examples of $p$, $p_0$ and the geodesic base-curves $\beta$ and $\beta^{*}$. Panel (b) shows the geodesics between amplitudes of functions on  $ {\mathbb{B}}/\Gamma$ with the corresponding time warping functions $\gamma^{*}\in\Gamma$ on the right. }
\label{fig:geodesics2}
\end{figure}

It is important to notice that the geodesic base-curve $\beta^{*}$ in the space of function amplitudes $ {\mathbb{B}}/\Gamma$ can be different from the geodesic base-curve $\beta$ on $ {\mathbb{B}}$. Figure \ref{fig:geodesics2} shows two examples of $p_0$ and $p$ on $\S^2$, the base-curve $\beta$ in $ {\mathbb{B}}$, the base-curve $\beta^*$ in $ {\mathbb{B}}/\Gamma$, the geodesic path between $p_0$ and $p$ in $ {\mathbb{B}}/\Gamma$, and the optimal $\gamma^*$ for aligning $p_0$ to $p$. 

\section{The Sample Fr\'{e}chet Mean on $ {\mathbb{B}}$ and $ {\mathbb{B}}/\tilde{\Gamma}$}\label{sec:Frechet}

With a clear understanding of the geometry of $ {\mathbb{B}}$ and $ {\mathbb{B}}/\tilde{\Gamma}$, we now are ready to study the Fr\'{e}chet means of smooth functions on $\S^2$ (elements on $ {\mathbb{B}}$) and their amplitudes (elements on $ {\mathbb{B}}/\tilde{\Gamma}$),  {and explore their use in statistical modeling of random functions on $\S^2$.} Let us denote a stochastic Riemannian process (e.g., Gaussian process) on $\S^2$ as $\{p(t)\}_{t\in[0, 1]}$, and with associated TSRVC representation $(x, q)$ given by 
\begin{equation*}
x = p(0), \qquad \quad q(t) = \left(\frac{\dot{{p}}(t)}{\sqrt{|\dot{{p}}(t)|}}\right)_{{p}(t) \rightarrow {p}(0)}^{||}.
\end{equation*}

\subsection{The sample Fr\'{e}chet mean on $ {\mathbb{B}}$}
\label{subsec:frechetC}

To estimate the mean of the stochastic Riemannian process, we consider the Fr\'{e}chet function $F(p_{\mu}) = \mathbb{E} \left( d^2_{ {\mathbb{B}}}(p, p_{\mu}) \right)$, and its finite-version
\begin{equation*}
F_n(p_{\mu}; {p}_1, \dots, {p}_n) = \frac{1}{n} \sum_{i=1}^n d_{ {\mathbb{B}}}^2({p}_i, p_{\mu}).
\end{equation*}
Then the estimator of the  mean is given by the sample Fr\'{e}chet mean
\begin{equation}\label{eqn:samplefm}
p_{\mu} = \argmin_{p\in\mathcal{F}} F_n(p; {p}_1, \dots, {p}_n) .
\end{equation}
 {The theoretical properties of the  Fr\'echet mean in a general manifold have been extensively studied in \cite{bhattacharya2005large,afsari2011riemannian}. Here we focus on computational tools and derive computational algorithms to solve (\ref{eqn:samplefm}).}  {
Namely we are using a gradient descent algorithm to find the sample Fr\'{e}chet mean on $\mathbb{B}$ and $\mathbb{B}/\Gamma$ (for more general manifolds, we refer the readers to \cite{pennec1998computing,le_2001,Groisser2004NewtonsMZ}).}  We rewrite the Fr\'{e}chet function on $ {\mathbb{B}}$ as a function of $x \in \S^2$ and base-curves $\beta_i$:
\begin{align*}
 F_n(x, \beta_1, \dots, \beta_n) =& \frac{1}{n}\sum_{i=1}^n d_{\beta_i}^2\Big(\big(x_i, q_i(t)\big), \big(x, q(t, x, \beta_1, \dots, \beta_n)\big)\Big) \\
 =& \frac{1}{n} \sum_{i=1}^n \Big(\ell_{\beta_i}^2 + \int_0^1 \big|q - q_{i, \beta_i}^{||}\big|^2\d t\Big),
\end{align*}
where we have    
\begin{equation}\label{eqn:TSRVC}
 q = q(t, x, \beta_1, \dots, \beta_n) = \frac{1}{n}\sum_{i=1}^n q_{i, \beta_i}^{||}(t),
\end{equation}
with ${q}_{i, \beta_i}^{||}$ being the  parallel transportations of  ${q}_i(t)$ along $\beta_i,$ ${q}_{i, \beta_i}^{||}(t) = \big(q_i(t)\big)_{\beta_i(0) \rightarrow \beta_i(1)}^{||}$. We then optimize $F_n$ iteratively with respect to $x$ and $\beta_i$'s to obtain the sample Fr\'{e}chet mean on $ {\mathbb{B}}$:
\begin{align*}
 \big(\beta_1^x, \dots, \beta_n^x\big) &= \argmin_{\beta_1, \dots, \beta_n} F_n(x, \beta_1, \dots, \beta_n),\\
 x &= \argmin_{x\in\S^2} F_n\Big(x, q\big(t, x, \beta_1^x, \dots, \beta_n^x\big)\Big).
\end{align*}

Therefore the problem of finding the sample Fr\'{e}chet mean reduces to the optimization problem on the product space $\mathbb{S}^2\times\left[-\frac{\pi}{2}, \frac{\pi}{2}\right]^n$. Since  $q$ is determined by the starting point $x\in\mathbb{S}^2$ and the base-curves $\beta_1, \dots, \beta_n,$ which are defined by the angles $\theta_1, \dots, \theta_n\in\left[-\frac{\pi}{2}, \frac{\pi}{2}\right]$, for a given $x\in\mathbb{S}^2,$ we use the gradient descent method to find the optimal 
\begin{equation}\label{eqn:beta}
 \theta = (\theta_1, \dots, \theta_n) = \argmin_{(\theta_1, \dots, \theta_n)\in \left[-\frac{\pi}{2}, \frac{\pi}{2}\right]^n} F_n\big(x, \beta_1(\theta_1), \dots, \beta_n(\theta_n)\big) \quad,
\end{equation}
where we have to compute
\begin{equation*}
 \nabla_{\theta}F_{n}\big(x, \beta_1(\theta_1), \dots, \beta_n(\theta_n)\big) = \big(\hat{F}_{1n\theta}({p}_1, p, \theta_1), \dots, \hat{F}_{nn\theta}({p}_n, p, \theta_n)\big), \end{equation*}
where $\hat{F}_{in\theta}({p}_i, p, \theta_i) = \frac{\partial}{\partial\theta_i} d^2(p_i, p, \theta)$. We propose the following Algorithm \ref{alg:opttheta} to find the optimal  
$\theta = \big(\theta_1, \dots, \theta_n\big).$

\begin{algorithm}[ht]\label{alg:opttheta}
\caption{The method for finding optimal $\theta=(\theta_1, \dots, \theta_n)$ for \eqref{eqn:beta}}
\KwInput{observed functions $p_1,\cdots,p_n$, the point $x\in\mathbb{S}^2,$ the initial value $\theta_0 \in \left[-\frac{\pi}{2}, \frac{\pi}{2}\right]^n,$ the gradient step $\lambda_1\in\mathbb{R},$ and the degree of accuracy $\varepsilon_1\in\mathbb{R}$}
\KwOutput{a vector $\theta = (\theta_1, \dots,  \theta_n) \in \left[-\frac{\pi}{2}, \frac{\pi}{2}\right]^n$}
\Begin{
 Set $\theta = \theta_0.$ 

 \Repeat { \rm $\big\|\nabla_{\theta}F_{n}(x, \theta_1, \dots, \theta_n)\big\| < \varepsilon_1$ }
 {
  Compute $\bar{\theta} = \theta - \lambda_1 \nabla_{\theta} F_n\big(x, \beta_1(\theta_1), \dots, \beta_n(\theta_n)\big).$
	
  \eIf{$\bar{\theta}\in\left[-\frac{\pi}{2}, \frac{\pi}{2}\right]^n$ \quad {\rm and} \quad $F_n\big(x, \beta_1(\bar{\theta}_1), \dots, \beta_n(\bar{\theta}_n)\big) < F_n\big(x, \beta_1(\theta_1), \dots, \beta_n(\theta_n)\big)$}{ $\theta = \bar{\theta}$\;}{ $\lambda_1 = \frac{\lambda_1}{2}$\;}
 }
}
\end{algorithm}

Similarly, for given $\theta = (\theta_1, \cdots, \theta_n)$, we perform the gradient descent algorithm to find the optimal $x$ according to
$\argmin_{x\in\mathbb{S}^2} F_{n}\big(x, \beta_1(\theta_1), \dots, \beta_n(\theta_n)\big),$ where we have to compute
\begin{equation*}
 \nabla_{x} F_{n}\big(x, \beta_1(\theta_1), \dots, \beta_n(\theta_n)\big) = \frac{1}{n} \sum_{i=1}^n \nabla_x d^2(p_i, p, \theta_i), 
\end{equation*}
The explicit expression for $\nabla_x d^2(p_i, p, \theta_i)$ can be found in the supplement. 
Procedure \ref{alg:optx} describes the full algorithm to find the sample Fr\'{e}chet mean on $ {\mathbb{B}}$.

\begin{algorithm}[ht]\label{alg:optx}
\caption{Sample Fr\'{e}chet mean calculation on $ {\mathbb{B}}$}
\KwInput{ the observed functions $p_1, \dots, p_n$, the gradient steps $\lambda_1, \lambda_2 \in\mathbb{R},$ and the degrees of accuracy $\varepsilon_1, \varepsilon_2 \in\mathbb{R}$}
\KwOutput{ $x_\mu \in \mathbb{S}^2$, $q_\mu \in  {\mathbb{B}}_{x_\mu}$ and the vector $\theta \in \left[-\frac{\pi}{2}, \frac{\pi}{2}\right]^n.$}
\Begin{
 Set $x = \frac{x_1 + \dots + x_n}{|x_1 + \dots + x_n|},$ where $x_i = p_i(0)$.
	
 \Repeat { \rm $\Big|\nabla_{x} F_{n}\big(x, \beta_1(\theta_1), \dots, \beta_n(\theta_n)\big)\Big| < \varepsilon_2$ }{
 Compute $\theta$ using Algorithm \ref{alg:opttheta} with input $({p}_1, \dots, {p}_n, x, \theta, \lambda_1, \varepsilon_1).$ \\

 Set $x = \exp_x \Big(-\lambda_2 \nabla_x F_{n}\big(x, \beta_1(\theta_1), \dots, \beta_n(\theta_n)\big)\Big),$ where $\exp_x(\cdot)$ represents the exponential map on $\S^2.$}
}
\end{algorithm}

Remark: since each squared distance function $d_{\C}^2({p}_i, \cdot)$ is convex in some neighborhood $U_i \in \C$ of ${p}_i$  {(see \cite{zhang_sra_2016}, \cite{NIPS2017_Liu_Shang_Cheng_Cheng_Jiao}, \cite{bacak2014convex})},  the Fr\'{e}chet function should be convex at the intersection of the convex neighborhoods $U_i$, $U = \bigcap_{i=1}^n U_i$. So if we choose an appropriate initial point $p_\mu^0\in U,$ for any gradient descent algorithm, we should converge to the local optimal point $p_\mu$ by linear convergence rate. 

\subsection{The sample Fr\'echet mean on $ {\mathbb{B}}/\tilde{\Gamma}$}
\label{subsec:FrechetCGamma}

To calculate the mean of functions in $ {\mathbb{B}}/\tilde{\Gamma}$, we have to incorporate the alignment process. Given a set of smooth functions $p_1, \dots, p_n$, we analyze their amplitudes by choosing an element $\tilde{p}_i$ from each $[p_i]$ and working with $\tilde{p}_1, \dots, \tilde{p}_n$ using the geometry structure defined for $ {\mathbb{B}}$. Here  $\tilde{p}_1, \dots, \tilde{p}_n$ are well aligned among each other so that  $I = \argmin_{\gamma \in \tilde{\Gamma}}d_ {\mathbb{B}}(\tilde{p}_i(\gamma),\tilde{p}_j)$ for $i\neq j$ with $I$ being the identity function. Therefore, with an appropriate initialization, we take the following iterative procedure to calculate the Fr\'echet mean $\tilde{p}_{\mu}$ and its TSRVC $(\tilde{x}_\mu, \tilde{q}_\mu)$ on $ {\mathbb{B}}/\tilde{\Gamma}$:
\begin{enumerate}
    \item Update alignment between  $\tilde{p}_1, \dots, \tilde{p}_n$ and $\tilde{p}_\mu$, and update $\tilde{p}_\mu$;
    \item Update the base-curves $\beta_i^*$ between $x_1,...,x_n$ and $\tilde{x}_\mu$
    \item Update $\tilde{x}_\mu$
\end{enumerate}

For step 1, we adopt the Procrustes process from \cite{RamsayLi1998}. For given $\tilde{x}_\mu \in \mathbb{S}^2$, $\beta^*_i$, the $\tilde{q}_\mu$ is first computed as the average of the transported TSRVCs $\tilde{q}_\mu(t) = \frac{1}{n} \sum_{i=1}^n q_{i, \beta^*_i}^{||}(t).$ We then pair-wisely align every transported TSRVC $q_{i, \beta_i^*}^{||}(t)$ to $\tilde{q}_\mu$ with $\gamma_i = \argmin_{\gamma\in\Gamma} d^2_{\beta_i^*}(p_i\circ\gamma, \tilde{p}_\mu),$ and set $\tilde{p}_i = p_i(\gamma_i)$. Next, set $\tilde{q}_{\mu}(t) = \frac{1}{n} \sum_{i=1}^n \tilde{q}_{i, \beta^*_i}^{||}(t),$ where $\tilde{q}_{i, \beta^*_i}^{||}(t)$ is the transported TSRVC of $\tilde{p}_i$.  For steps 2 and 3, we use similar algorithms as the ones presented in section \ref{subsec:frechetC}. Therefore, we develop the following algorithms for computing $\tilde{p}_\mu$. 

\begin{algorithm}[ht]\label{alg:basecurvecg}
 \caption{The method for finding optimal base-curves $\beta_1^*, \dots, \beta_n^*$ on $ {\mathbb{B}}/\tilde{\Gamma}$}
\KwInput{ 
${p}_1, \dots, {p}_n$,
 the point $\tilde{x}_\mu \in\mathbb{S}^2,$ 
 the initial vector $\theta_0\in\left[-\frac{\pi}{2}, \frac{\pi}{2}\right]^n,$
 the gradient step $\lambda_1\in\mathbb{R},$ 
 and the degree of accuracy $\varepsilon_1\in\mathbb{R}$}
\KwOutput{a vector $\theta = (\theta_1, \dots, \theta_n)$ for $\beta_1^*, \dots, \beta^*_n$ and aligned paths $\tilde{p}_1, \dots, \tilde{p}_n$
}
\Begin{
 Set $\tilde{p}_i = p_i = (x_i, \title{q}_i)$  for $i=1, \dots, n$ and $\theta = \theta_0.$ 

 \Repeat { $\big|\nabla_{\theta} F_{n}(x, \beta_1(\theta_1), \dots, \beta_n(\theta_n))\big| < \varepsilon_1$ }
 {
  Align $q_1, \dots, q_n$ to $\tilde{q}_\mu = \frac{1}{n} \sum_{i=1}^n q_{i,\beta_i(\theta_i)}^{||}$ at $ {\mathbb{B}}_{x_\mu}$,
	
  Compute $\bar{\theta} = \theta - \lambda_1 \nabla_{\theta} F_n\big(x, \beta_1(\theta_1), \dots, \beta_n(\theta_n)\big).$
	
  \eIf{$\bar{\theta} \in \left[-\frac{\pi}{2}, \frac{\pi}{2}\right]^n$ 
   \quad {\rm and} $F_n\big(x, \beta_1(\bar{\theta}_1), \dots, \beta_n(\bar{\theta}_n)\big) < F_n\big(x, \beta_1(\theta_1), \dots, \beta_n(\theta_n)\big)$} {
   $\theta = \theta - \lambda_1 \nabla_{\theta} F_{n}\big(x, \beta_1(\theta_1), \dots, \beta_n(\theta_n)\big)$\;
   }{
   $\lambda_1 = \frac{\lambda_1}{2}$\;
   }
  }
 }
\end{algorithm}

Algorithm \ref{alg:basecurvecg} allows us to finish steps 1 and 2. We now integrate results from Algorithm \ref{alg:basecurvecg} with an algorithm to update $x_\mu$ to finalize our computational procedure (presented in  Algorithm \ref{alg:aloptx}) for calculating the sample mean amplitude.  

\begin{algorithm}[ht]\label{alg:aloptx}
\caption{Sample Fr\'echet mean calculation on $ {\mathbb{B}}/\tilde{\Gamma}$}
 \KwInput{ 
 ${p}_1, \dots, {p}_n\in {\mathbb{B}},$ two gradient steps $\lambda_1, \lambda_2 \in \mathbb{R},$ and degrees of accuracy $\varepsilon_1, \varepsilon_2 \in \mathbb{R}_{+}$}
 \KwOutput{Fr\'echet mean $\tilde{p}_\mu$ represented as $(\tilde{x}_{\mu}, \tilde{q}_\mu)$, aligned functions $\tilde{p}_1, \dots, \tilde{p}_n$ to $\title{p}_\mu$, and base-curves $\beta_i$ connecting $x_1, \dots, x_n$ with $\tilde{x}_\mu$ in the form of $\theta$.}
 \Begin{
  Set $x = \frac{{x}_1 + \dots + {x}_n}{|{x}_1 + \dots + {x}_n|},$ $\theta = 0,$ and $\tilde{p}_i = p_i$ for $i=1, \dots, n.$
	
  \Repeat { \rm $\Big|\nabla_x F_n\big(x, \beta_1(\theta_1), \dots, \beta_n(\theta_n)\big)\Big| < \varepsilon_2$ }
  {
   Compute $\theta(x), \tilde{p}_1, \dots, \tilde{p}_n $ from Algorithm \ref{alg:basecurvecg} with inputs $(p_1, \dots, p_n, x, \theta, \lambda_1, \varepsilon_1).$ \\ 
   	
   Set $x = \exp_x \Big(-\lambda_2 \nabla_x \hat{F}_{n}\big(x, \beta_1(\theta_1), \dots, \beta_n(\theta_n)\big)\Big).$
  }
}
\end{algorithm}

\section{ {Covariance on the Tangent Space of the Fr\'echet mean}}
\label{sec:StatModel}
 {For a set of observed functions, in addition to the mean, another interesting statistical quantity is the sample covariance. For example, a Gaussian process (GP) can be uniquely decided by its mean and covariance functions.} 
For the Riemannian stochastic process $p = (x, q) \in {\mathbb{B}}$ we have the Fr\'echet mean
\begin{align*}
 p^{*} = (x^{*}, q^{*}) = \argmin_{\hat{p}\in {\mathbb{B}}} \mathbb{E} d_{ {\mathbb{B}}}^2(\hat{p}, p).
\end{align*} 

 {To compute the covariance function, we consider the tangent space at $p^{*} = (x^*, q^*)$ (denoted as $T_{p^*}\C$ for notation simplicity), which is a functional space $\mathbb{L}_2([0, 1], T_{x^{*}}\mathbb{S}^2 \times T_{x^{*}}\mathbb{S}^2)$ with imposed constraints that first functions are constants.} The main issue here is that there exists statistical framework for vectors and vector-functions but not for the direct product of a vector space and vector-function space. The difficulty is to determine the  {covariance} function $\Sigma.$ We consider our space of interest as the 
principal in the vector-function space and therefore define the  {covariance} function at the vector-function space. We have the following  {covariance} function on $\mathbb{L}_2([0, 1], T_{x^{*}}\mathbb{S}^2\times T_{x^{*}}\mathbb{S}^2)$
\begin{align*}
 \Sigma = \frac{1}{4} \mathbb{E}\nabla d_{ {\mathbb{B}}}^2(p, p^{*}) \otimes \nabla d_{ {\mathbb{B}}}^2(p, p^{*}) = \mathbb{E} \exp_{p^{*}}^{-1}p \otimes \exp_{p^{*}}^{-1}p,
\end{align*}
that maps $(\bar{u}, \bar{w})\in \mathbb{L}_2([0, 1], T_{x^{*}} \mathbb{S}^2\times T_{x^{*}}\mathbb{S}^2)$ with respect to an orthonormal basis for $T_{x^{*}}\mathbb{S}^2$ as follow
\begin{align*}
 \big(\Sigma (\bar{u}, \bar{w})\big)(s) = \big(\hat{u}, \hat{w}(s)\big), \qquad \text{with} \quad
 \begin{split}
  \begin{pmatrix}
   \hat{u}^1 \\
   \hat{u}^2
  \end{pmatrix} =&  
  \begin{pmatrix}
   \kappa^1_1 & \kappa^1_2 \\ 
   \kappa^2_1 & \kappa^2_2
  \end{pmatrix} 
  \begin{pmatrix}
   \bar{u}^1 \\
   \bar{u}^2
  \end{pmatrix} + \int_0^1 \begin{pmatrix}
   k^1_1(t) & k^1_2(t) \\ 
   k^2_1(t) & k^2_2(t)
  \end{pmatrix} \begin{pmatrix}
   \bar{w}^1(t) \\
   \bar{w}^2(t)
  \end{pmatrix}\d t, \\
  \begin{pmatrix}
   \hat{w}^1(s) \\
   \hat{w}^2(s)
  \end{pmatrix} =& \begin{pmatrix}
   k^1_1(s) & k^1_2(s) \\ 
   k^2_1(s) & k^2_2(s)
  \end{pmatrix}\begin{pmatrix}
   \bar{u}^1 \\
   \bar{u}^2
  \end{pmatrix} + \int_0^1\begin{pmatrix}
   K^1_1(s, t) & K^1_2(s, t) \\ 
   K^2_1(s, t) & K^2_2(s, t)
  \end{pmatrix} \begin{pmatrix}
   \bar{w}^1(t) \\
   \bar{w}^2(t)
  \end{pmatrix}\d t,
 \end{split}
\end{align*}
for
\begin{align} \label{equ:cov}
  \kappa^{i}_j = \mathbb{E} u^iu^j, \quad k^{i}_j(t)=\mathbb{E} u^i w^j(t), \quad K^{i}_j(s, t)=\mathbb{E} w^i(s)w^j(t).
\end{align}

 { In finite sample case, let $p_\mu$ be the sample Fr\'echet mean, the sample covariance function is defined as 
\begin{align*}
\hat{\Sigma} = \frac{1}{n-1} \sum_{i=1}^n \exp_{p_\mu}^{-1}p_i \otimes \exp_{p_\mu}^{-1}p_i. 
\end{align*}
With a way to estimate the mean and covariance functions, we now can use  statistical models to quantify the uncertainty of observed data, e.g., Gaussian process (GP) or mixture of GPs. Note that although we have only discussed the covariance function computation on $\mathbb{B}$, extension to $\mathbb{B}/\Gamma$ is straightforward since in the process of computing the sample Fr\'echet mean on $\mathbb{B}/{\Gamma}$ (i.e., Algorithm \ref{alg:aloptx}), we align all functions.}

\section{Simulation Studies}

In this section, we demonstrate the computational tools we developed using simulated data and compare them with existing ones, e.g. those in \cite{ZhengwuZhang2018}. 

We first simulate curves on $\S^2$ according 
the following probabilistic model. A mean function $p^*$ is first simulated and then we generate a random covariate function $\Sigma$ on the tangent space $T_{p^{*}} {\mathbb{B}}$. Using a Gaussian process (GP) model 
\begin{align}\label{eqn:sim1}
 (u_i, w_i)\simeq GP\big(0, \Sigma\big), \quad {\rm with} \quad
 \Sigma = \begin{pmatrix}
  \kappa_{1}^1 & \kappa_{2}^1 & k_{1}^1(t) & k_{2}^1(t) \\
  \kappa_{1}^2 & \kappa_{2}^2 & k_{1}^2(t) & k_{2}^2(t) \\
  k_{1}^1(s) & k_{2}^1(s) & K_{1}^1(s, t) & K_{2}^1(s, t) \\
  k_{1}^2(s) & k_{2}^2(s) & K_{1}^2(t, s) & K_{2}^2(s, t) 
 \end{pmatrix}.
\end{align}
We simulate random tangent vectors $(u, w)$ and map them back to $\S^2$ using the exponential mapping $\exp_{p^{*}}(\cdot)$. In total, we simulated 8 mean functions $p^*$ with different shapes, and for each $p^*$, we generated a random covariance function $\Sigma$, and sampled 10 functions. 
Next, for a total of 80 functions, we inserted phase variability to these functions (Figure \ref{fig:SimCurves2} describes the process of inserting phase variability). For any two functions $p_0$, $p$, Algorithm \ref{alg:ampldist} was used to align them and compare their amplitude difference. Figure \ref{fig:PairsOfSimCurves} shows a few examples of geodesics between the amplitudes of a pair of functions. We compared our Algorithm \ref{alg:ampldist} with the one described in \cite{ZhengwuZhang2018} (equation (7) in \cite{ZhengwuZhang2018}). It is important to note that the algorithm from \cite{ZhengwuZhang2018} only computes approximated geodesic distance. More specifically, their parallel transports along the base-curve $\beta$ are compromised by a numerical approximation procedure. The parallel transport of TSRVC is approximated by many small parallel transports along a piece-wise geodesic approximation of $\beta$. Moreover, to find the best $\beta$, they use an exhaustive search algorithm. As a consequence, the algorithm from \cite{ZhengwuZhang2018} crucially depends on $N$, the number of discrete points on $\beta,$ and $M$, the number of discrete searching points on the range of $[0, 2\pi)$ for finding the optimal $\beta.$  

From Figure \ref{fig:PairsOfSimCurves} we see that our new algorithm can give better geodesic distances compared with the one in \cite{ZhengwuZhang2018}. The improvement is more significant when the shapes of functions compared are more complex (e.g., the second to the fourth rows). For the algorithm in \cite{ZhengwuZhang2018},  one can observe that when $M$ increases the squared distance decreases, since the chance of finding the best $\beta$ increases. When $N$ increases the geodesic increases. This is because when $N$ increases, the quality of approximation of $\beta$ improves and the squared length of the baseline $\ell_{\beta}^2$ increases (in \cite{ZhengwuZhang2018} this values is approximated from below).

We then compared pair-wise differences among the 80 simulated functions, and Figure \ref{fig:SquareDistanceMatrices2} panel (a) shows pairwise squared distance matrix computed with our Algorithm \ref{alg:ampldist}. The same matrix was also computed using the algorithm in \cite{ZhengwuZhang2018} with $N = 120$ and $M = 240$. Panel (b) shows the difference matrix between the two algorithms (squared distance of \cite{ZhengwuZhang2018} -  squared distance of Algorithm \ref{alg:ampldist}). Let $d^2_a$ represent the squared distance from Algorithm \ref{alg:ampldist} and $d^2_b$ from \cite{ZhengwuZhang2018}. We calculate the percentage of improvement of Algorithm \ref{alg:ampldist} based on $100*(d^2_b -d^2_a)/d^2_b$, and show a histogram of these percentages in panel (c).  We see that the proposed algorithm improves the geodesic calculation in most cases. Actually, the average percentage of improvement is around 9\%, indicating that we have significantly improved the geodesic computation. 

\begin{figure}
 \begin{center}
  \begin{tabular}{ccc}
   \includegraphics[width=0.365\textwidth]{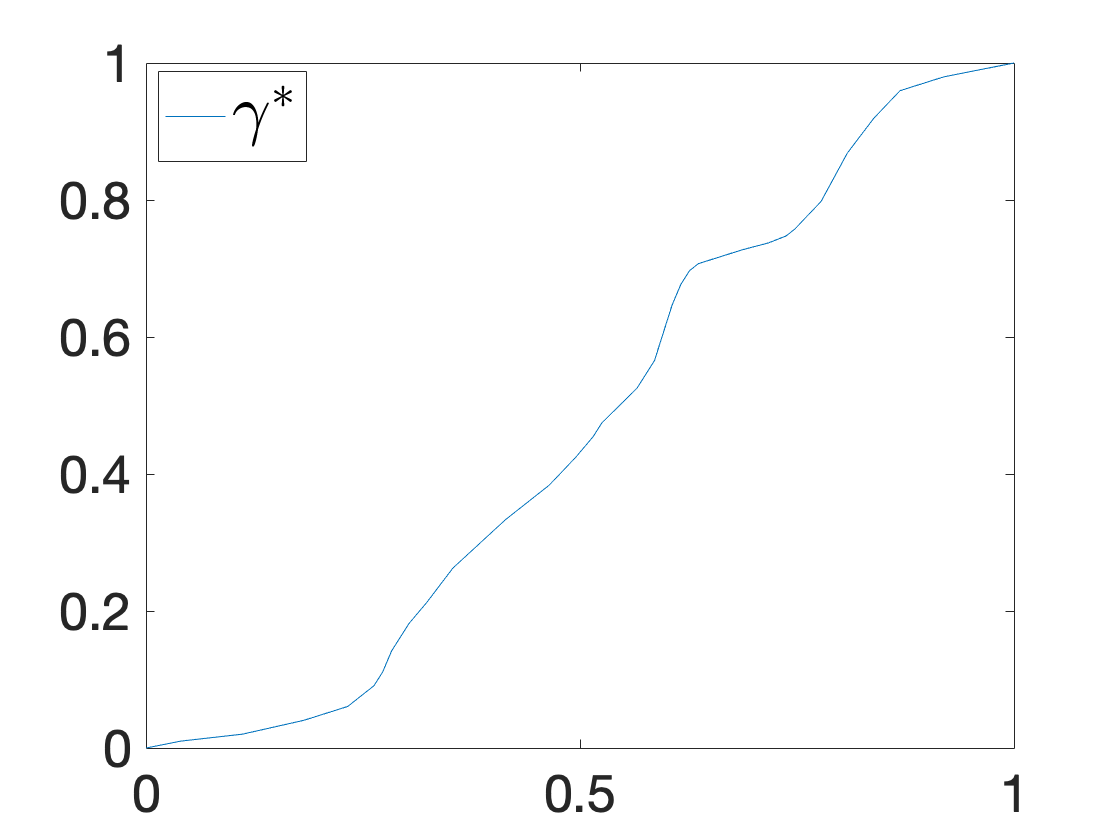} & 
   \includegraphics[width=0.265\textwidth]{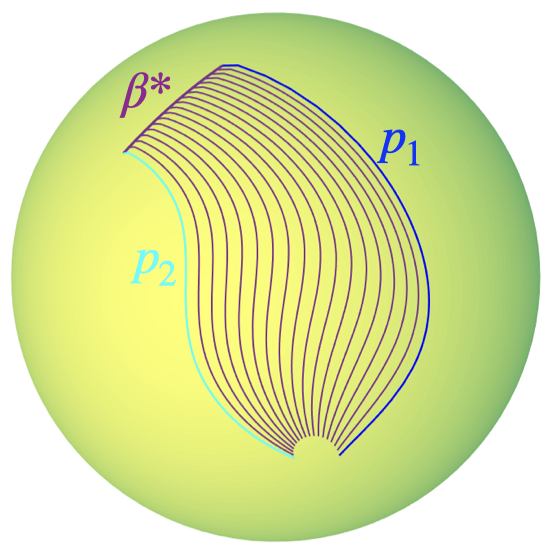} &
   \raisebox{1.1\height}{\scalebox{0.7}{\begin{tabular}{c|c}
       Methods & $d^2_{ {\mathbb{B}}/\Gamma}(p_1, p_2)$ \\
       \hline\hline 
       Algorithm \ref{alg:ampldist} & 0.778398 \\
       \hline
       \cite{ZhengwuZhang2018}
        with $N = 30, M = 60$ & 0.810474 \\
       \hline
       \cite{ZhengwuZhang2018}
        with $N = 30, M = 120$ & 0.810315 \\ 
       \hline
       \cite{ZhengwuZhang2018}
        with $N = 30, M = 240$ & 0.810237 \\  
       \hline
       \cite{ZhengwuZhang2018}
        with $N = 60, M = 60$ & 0.810529 \\
       \hline
       \cite{ZhengwuZhang2018}
        with $N = 60, M = 120$ & 0.810369 \\
          \hline
       \cite{ZhengwuZhang2018}
        with $N = 60, M = 240$ & 0.81029 \\ 
       \hline
       \cite{ZhengwuZhang2018}
        with $N = 120, M = 60$ & 0.810542 \\
       \hline
       \cite{ZhengwuZhang2018}
        with $N = 120, M = 120$ & 0.810382 \\
       \hline
       \cite{ZhengwuZhang2018}
        with $N = 120, M = 240$ & 0.810304 \\
   \end{tabular}}} \\
   \includegraphics[width=0.365\textwidth]{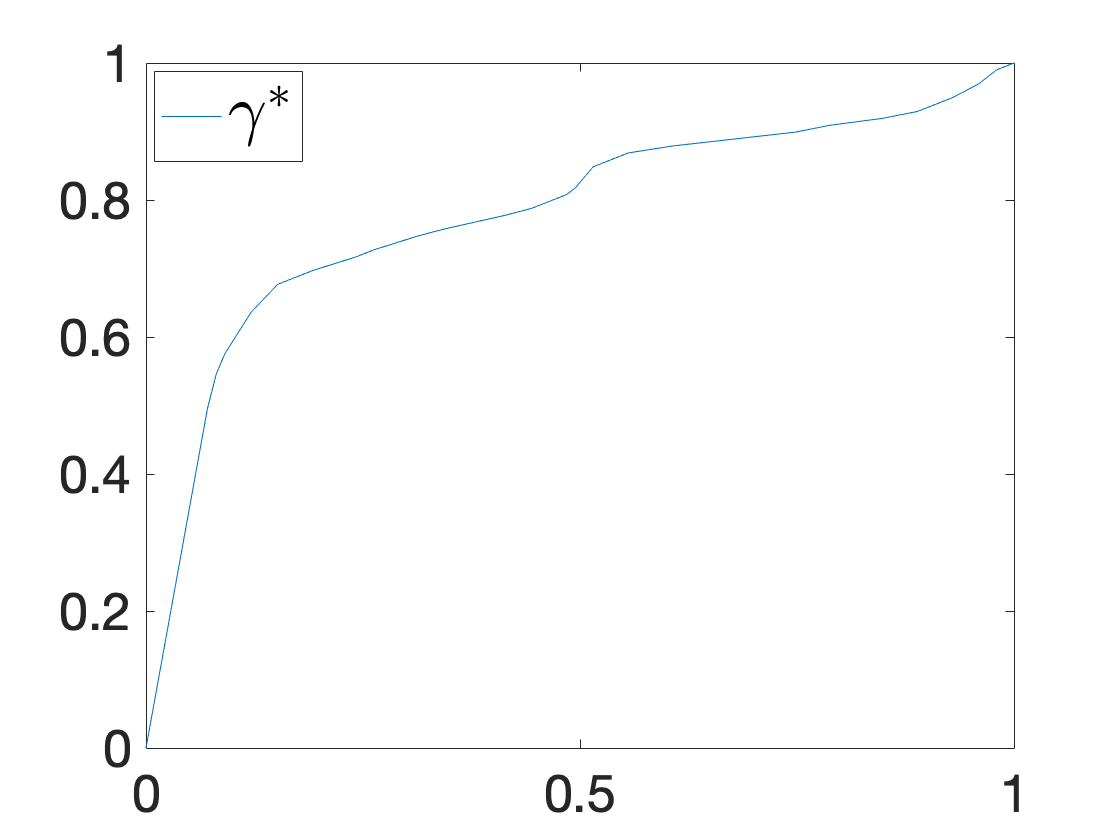} & 
   \includegraphics[width=0.265\textwidth]{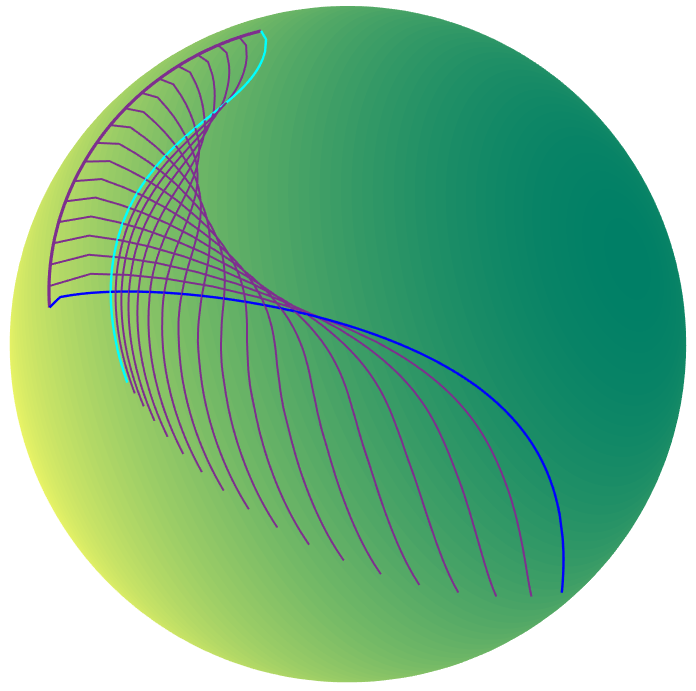} &
   \raisebox{1.1\height}{\scalebox{0.7}{\begin{tabular}{c|c}
       Methods & $d^2_{ {\mathbb{B}}/\Gamma}(p_1, p_2)$ \\
       \hline\hline 
       Algorithm \ref{alg:ampldist} & 2.52071 \\
       \hline
       \cite{ZhengwuZhang2018}
        with $N = 30, M = 60$ & 2.56984 \\
       \hline
       \cite{ZhengwuZhang2018}
        with $N = 30, M = 120$ & 2.56965 \\ 
       \hline
       \cite{ZhengwuZhang2018}
        with $N = 30, M = 240$ & 2.56952 \\  
       \hline
       \cite{ZhengwuZhang2018}
        with $N = 60, M = 60$ & 2.57047 \\
       \hline
       \cite{ZhengwuZhang2018}
        with $N = 60, M = 120$ & 2.57031 \\
          \hline
       \cite{ZhengwuZhang2018}
        with $N = 60, M = 240$ & 2.57016 \\ 
       \hline
       \cite{ZhengwuZhang2018}
        with $N = 120, M = 60$ & 2.57062 \\
       \hline
       \cite{ZhengwuZhang2018}
        with $N = 120, M = 120$ & 2.57047 \\
       \hline
       \cite{ZhengwuZhang2018}
        with $N = 120, M = 240$ & 2.57034 \\
   \end{tabular}}} \\
   \includegraphics[width=0.365\textwidth]{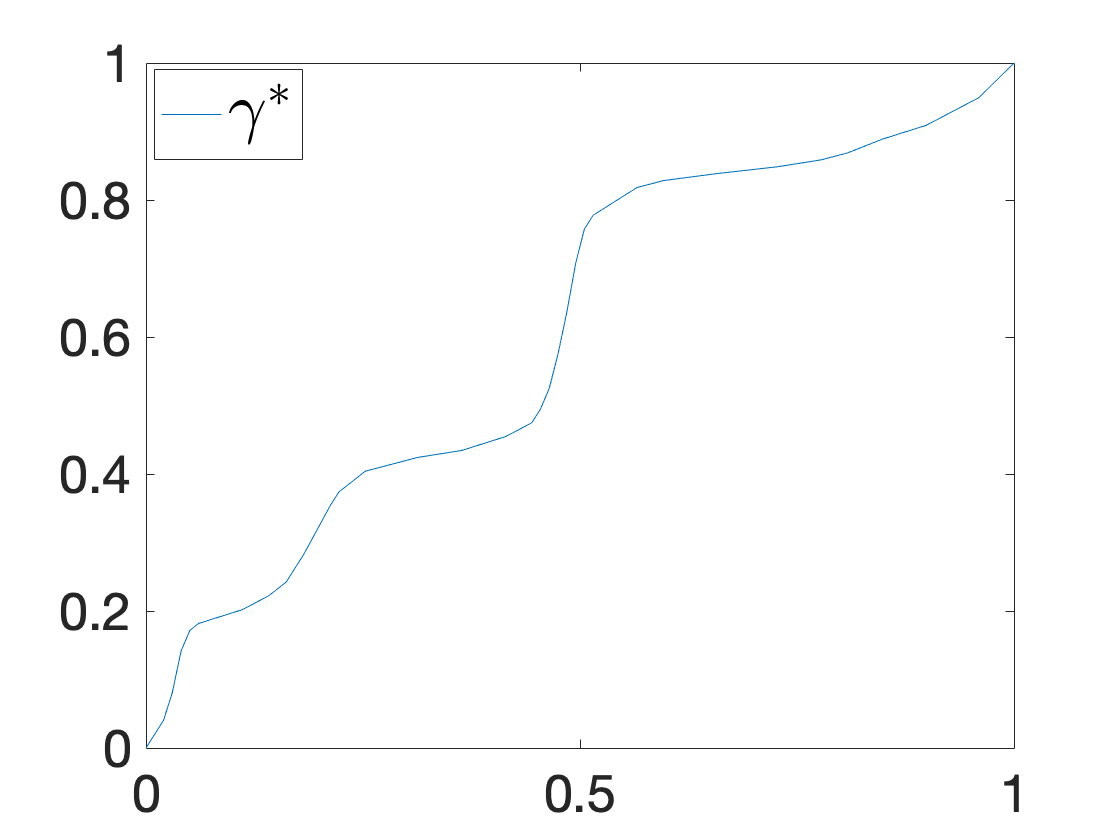} & 
   \includegraphics[width=0.265\textwidth]{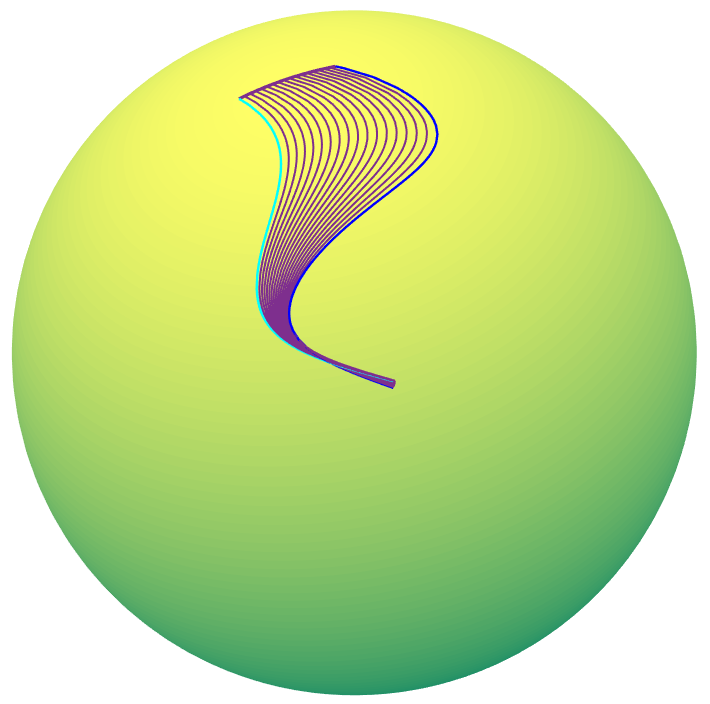} &
   \raisebox{1.1\height}{\scalebox{0.7}{\begin{tabular}{c|c}
       Methods & $d^2_{ {\mathbb{B}}/\Gamma}(p_1, p_2)$ \\
       \hline\hline 
       Algorithm \ref{alg:ampldist} & 0.237116 \\
       \hline
       \cite{ZhengwuZhang2018}
        with $N = 30, M = 60$ & 0.270649 \\
       \hline
       \cite{ZhengwuZhang2018}
        with $N = 30, M = 120$ & 0.270589 \\ 
       \hline
       \cite{ZhengwuZhang2018}
        with $N = 30, M = 240$ & 0.270527 \\  
       \hline
       \cite{ZhengwuZhang2018}
        with $N = 60, M = 60$ & 0.270651 \\
       \hline
       \cite{ZhengwuZhang2018}
        with $N = 60, M = 120$ & 0.270593 \\
          \hline
       \cite{ZhengwuZhang2018}
        with $N = 60, M = 240$ & 0.27053 \\ 
       \hline
       \cite{ZhengwuZhang2018}
        with $N = 120, M = 60$ & 0.270651 \\
       \hline
       \cite{ZhengwuZhang2018}
        with $N = 120, M = 120$ & 0.270593 \\
       \hline
       \cite{ZhengwuZhang2018}
        with $N = 120, M = 240$ & 0.270531 \\
   \end{tabular}}} \\
   \includegraphics[width=0.365\textwidth]{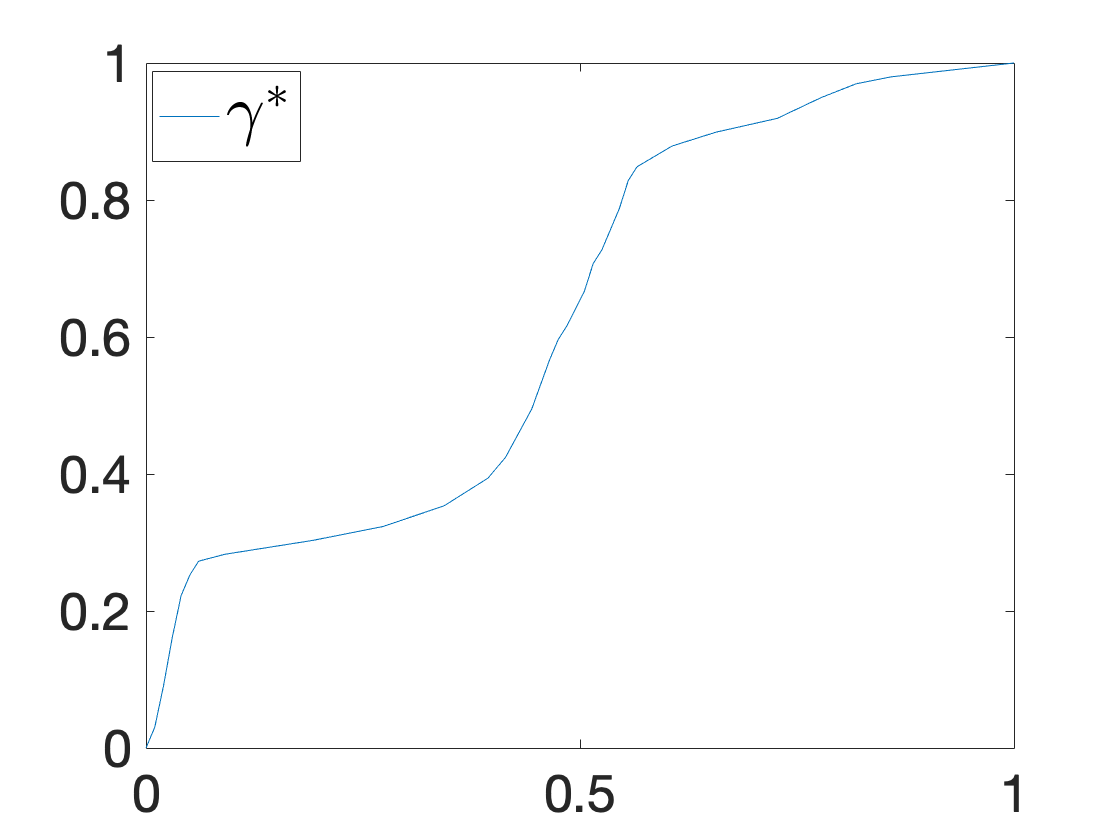} & 
   \includegraphics[width=0.265\textwidth]{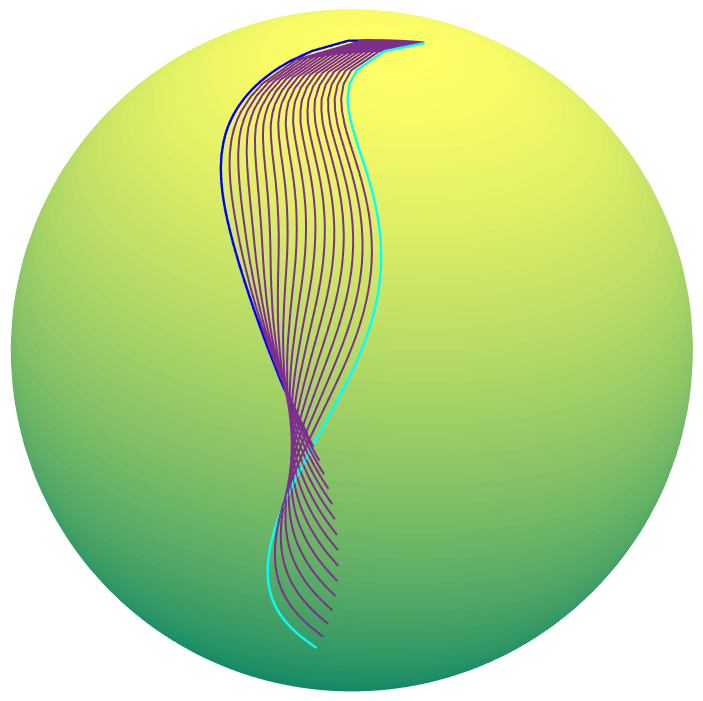} &
   \raisebox{1.1\height}{\scalebox{0.7}{\begin{tabular}{c|c}
       Methods & $d^2_{ {\mathbb{B}}/\Gamma}(p_1, p_2)$ \\
       \hline\hline 
       Algorithm \ref{alg:ampldist} & 0.483469 \\
       \hline
       \cite{ZhengwuZhang2018}
        with $N = 30, M = 60$ & 0.540953 \\
       \hline
       \cite{ZhengwuZhang2018}
        with $N = 30, M = 120$ & 0.540949 \\ 
       \hline
       \cite{ZhengwuZhang2018}
        with $N = 30, M = 240$ & 0.540947 \\  
       \hline
       \cite{ZhengwuZhang2018}
        with $N = 60, M = 60$ & 0.54096 \\
       \hline
       \cite{ZhengwuZhang2018}
        with $N = 60, M = 120$ & 0.540956 \\
          \hline
       \cite{ZhengwuZhang2018}
        with $N = 60, M = 240$ & 0.540955 \\ 
       \hline
       \cite{ZhengwuZhang2018}
        with $N = 120, M = 60$ & 0.540962 \\
       \hline
       \cite{ZhengwuZhang2018}
        with $N = 120, M = 120$ & 0.540958 \\
       \hline
       \cite{ZhengwuZhang2018}
        with $N = 120, M = 240$ & 0.540956 \\
   \end{tabular}}} \\
   \includegraphics[width=0.365\textwidth]{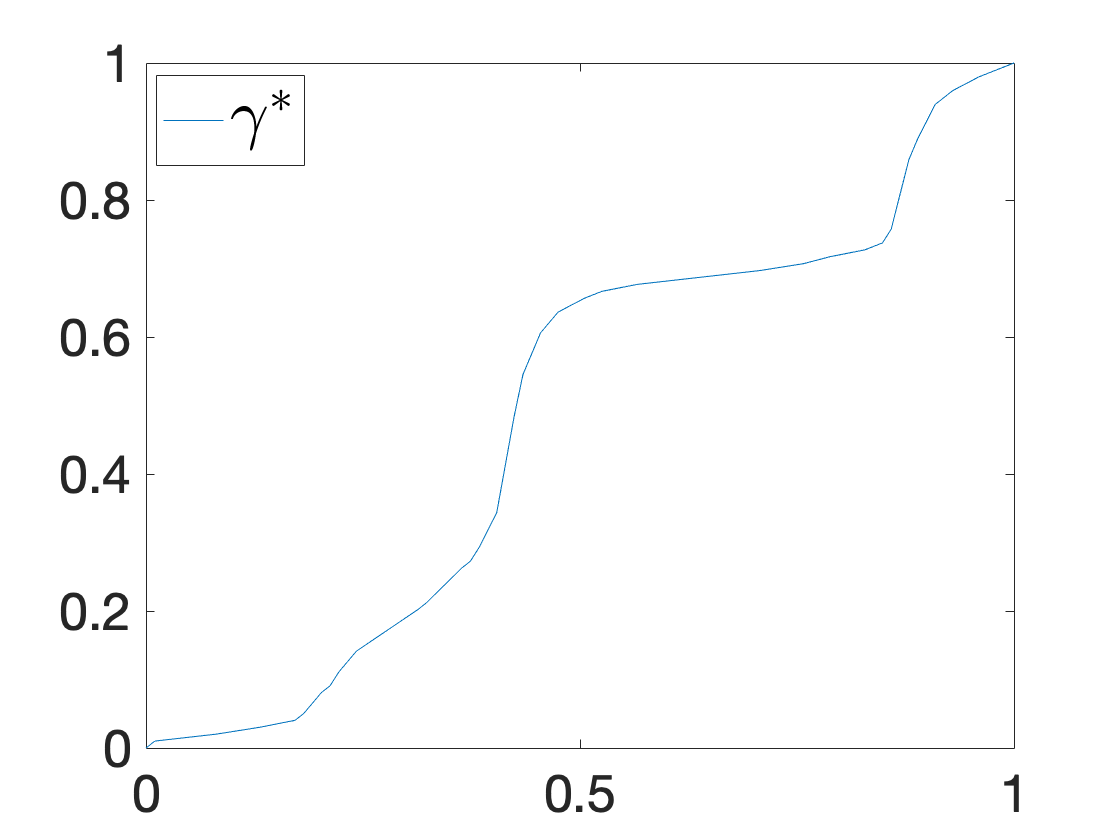} & 
   \includegraphics[width=0.265\textwidth]{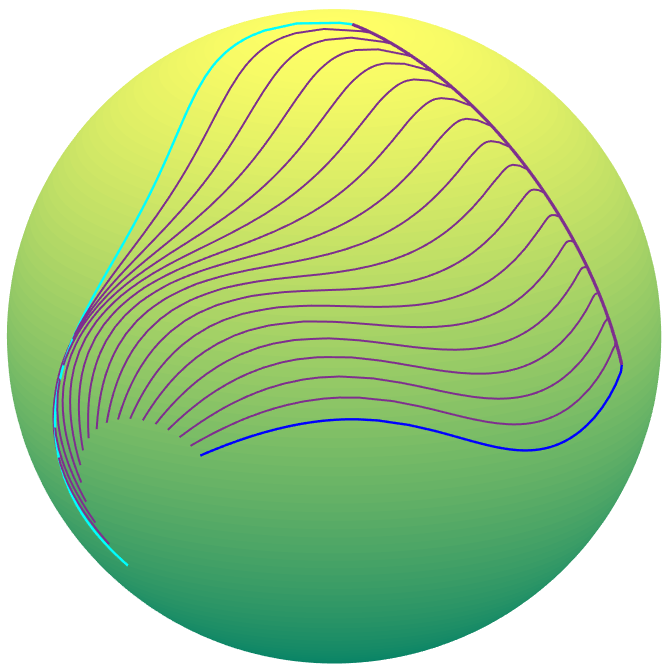} &
   \raisebox{1.1\height}{\scalebox{0.7}{\begin{tabular}{c|c}
       Methods & $d^2_{ {\mathbb{B}}/\Gamma}(p_1, p_2)$ \\
       \hline\hline 
       Algorithm \ref{alg:ampldist} & 2.63583 \\
       \hline
       \cite{ZhengwuZhang2018}
        with $N = 30, M = 60$ & 2.64099 \\
       \hline
       \cite{ZhengwuZhang2018}
        with $N = 30, M = 120$ & 2.63927 \\ 
       \hline
       \cite{ZhengwuZhang2018}
        with $N = 30, M = 240$ & 2.6393 \\  
       \hline
       \cite{ZhengwuZhang2018}
        with $N = 60, M = 60$ & 2.64242 \\
       \hline
       \cite{ZhengwuZhang2018}
        with $N = 60, M = 120$ & 2.64071 \\
          \hline
       \cite{ZhengwuZhang2018}
        with $N = 60, M = 240$ & 2.64064 \\ 
       \hline
       \cite{ZhengwuZhang2018}
        with $N = 120, M = 60$ & 2.64277 \\
       \hline
       \cite{ZhengwuZhang2018}
        with $N = 120, M = 120$ & 2.64106 \\
       \hline
       \cite{ZhengwuZhang2018}
        with $N = 120, M = 240$ & 2.64096 \\
   \end{tabular}}} \\
   (a) & (b) &  (c)   
  \end{tabular}
 \end{center}
 \caption{Geodesic calculation between two function amplitudes. Panel (a) shows the warping function $\gamma^{*}$ obtained by Algorithm \ref{alg:ampldist}, panel (b) shows the two functions on $\S^2$ and the geodesic between them, and panel (c) shows the table of geodesic distances computed by Algorithm \ref{alg:ampldist} and the algorithm from \cite{ZhengwuZhang2018} with various $N$ and $M$.}
 \label{fig:PairsOfSimCurves}
\end{figure}

\begin{figure}
 \begin{tabular}{c|c|c}
  \includegraphics[width=0.315\linewidth]{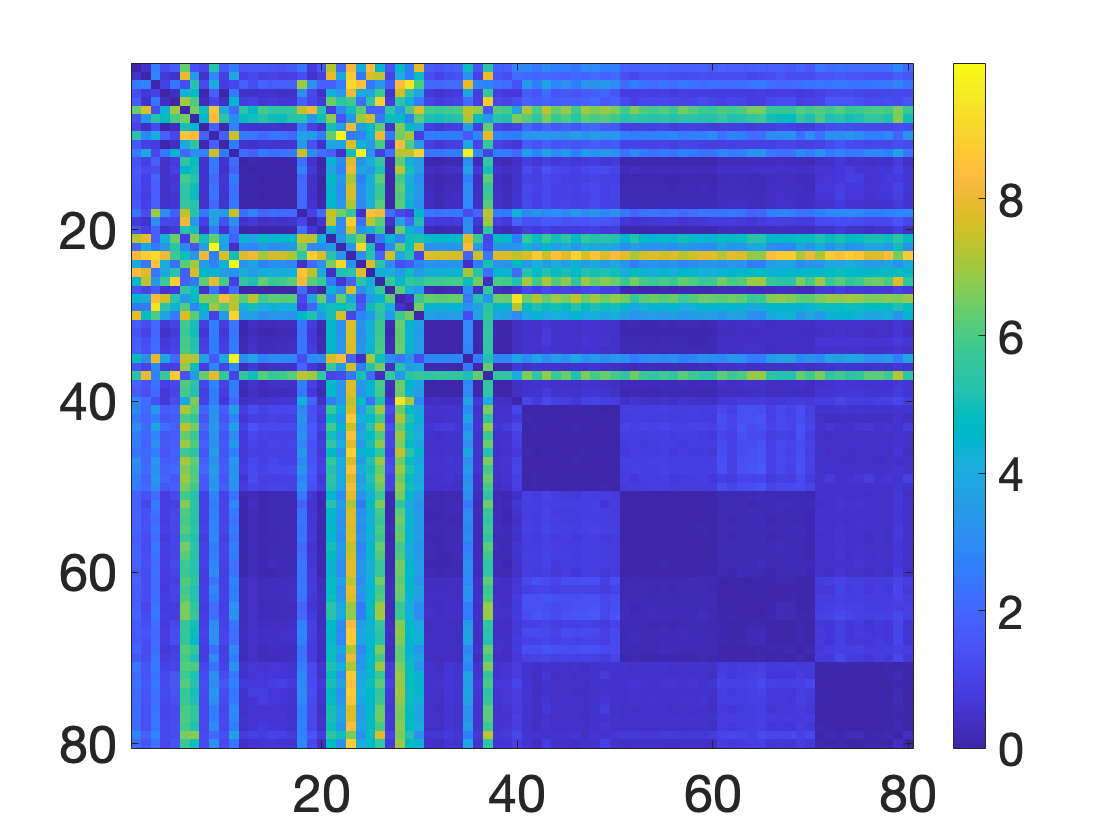} & 
  \includegraphics[width=0.315\linewidth]{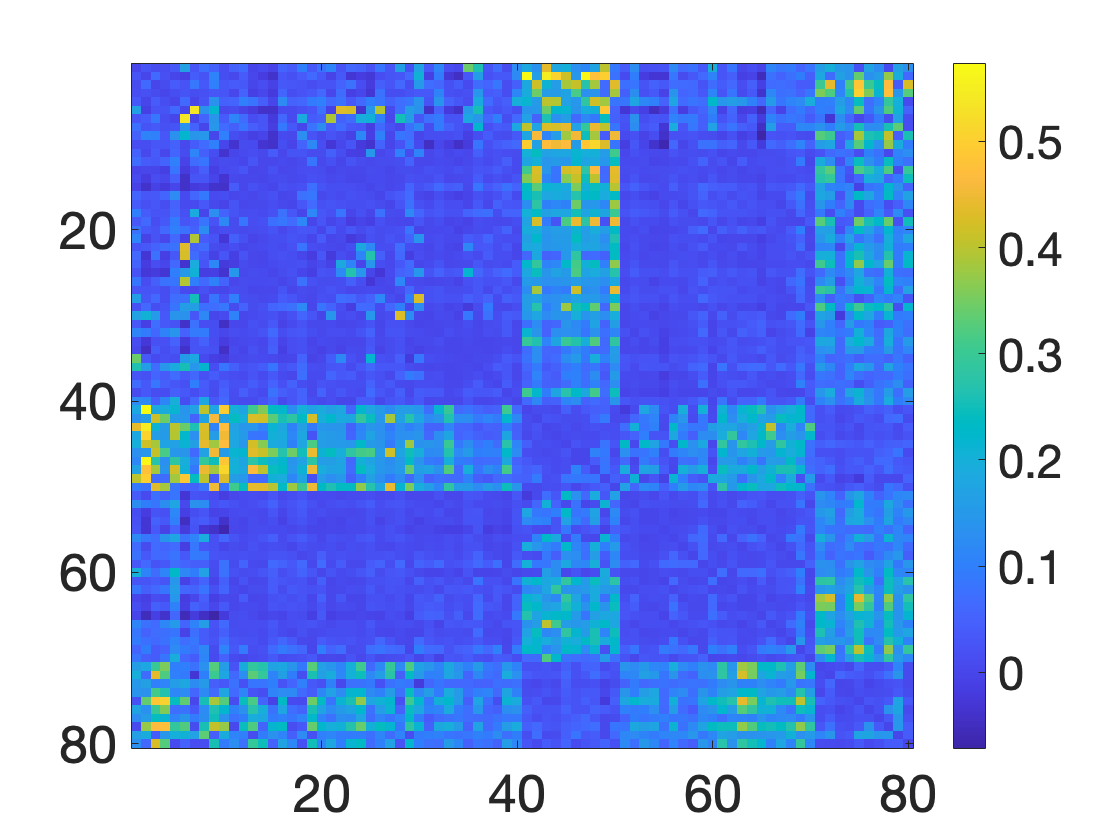} &
  \includegraphics[width=0.315\linewidth]{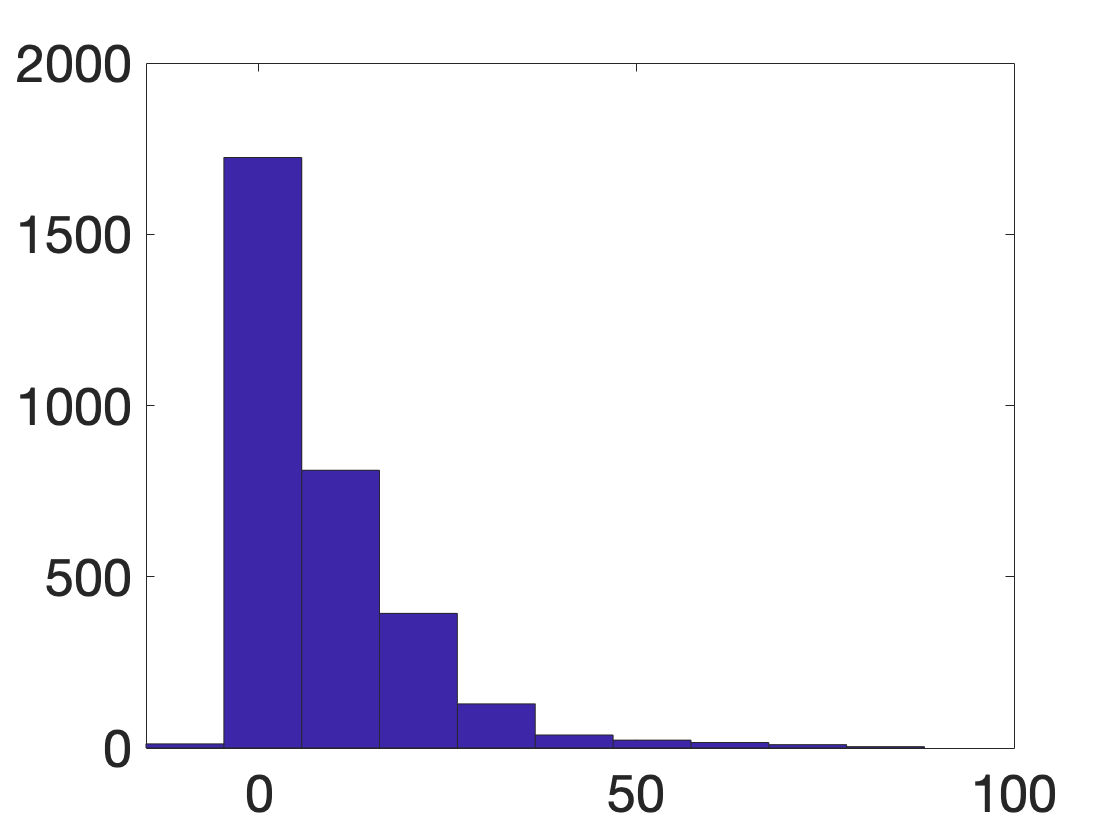} \\
  (a) & (b) & (c) 
\end{tabular}
 \caption{Pairwise squared distance among the amplitudes of $80$ simulated functions. Let $D_a$ be the squared distance matrix calculated by Algorithm \ref{alg:ampldist}, and $D_b$ be the matrix calculated by the algorithm in \cite{ZhengwuZhang2018} (with $N = 120$ and $M = 240$).  Panel (a) shows $D_a$, panel (b) shows $D_b-D_a$, and panel (c) shows the histogram of the upper triangle of $100*(D_b-D_a)./D_b$.}
 \label{fig:SquareDistanceMatrices2}
\end{figure}

\begin{figure}
\begin{center}
\begin{tabular}{ccccc}
 \includegraphics[width=0.3\textwidth]{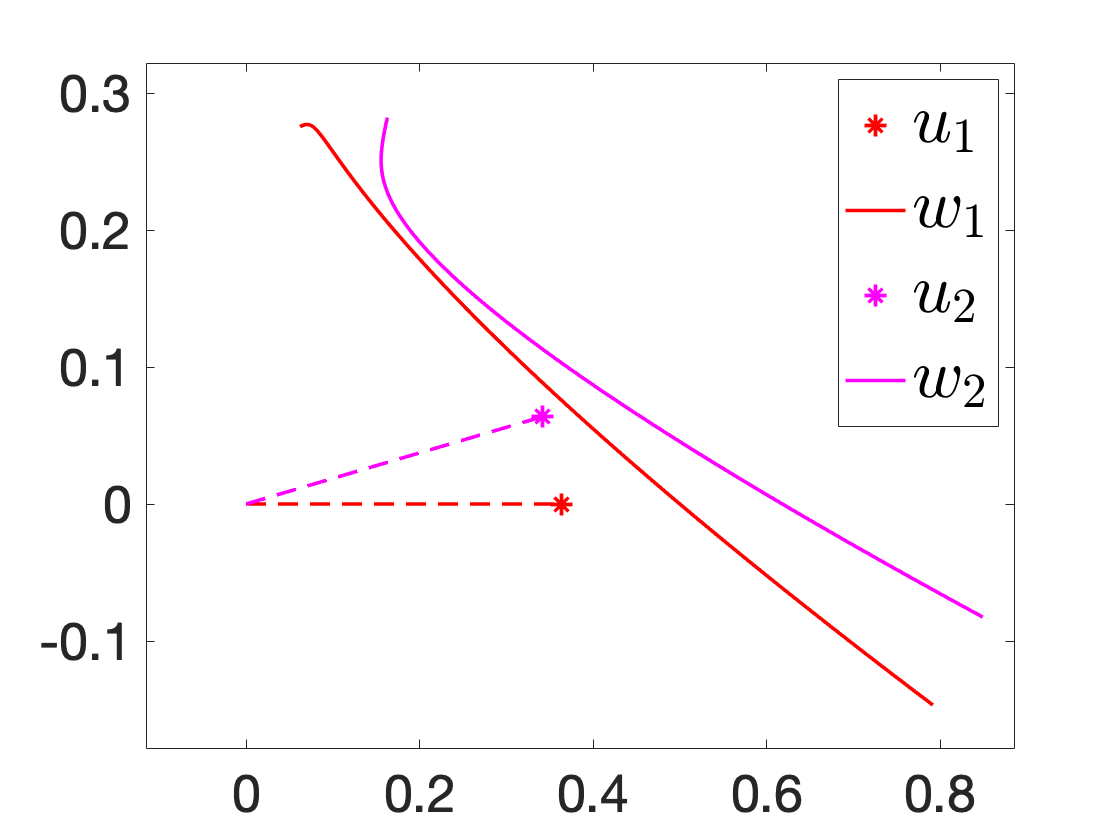} & 
 \raisebox{6.5\height}{\scalebox{2}{$\simeq$}} &
 \includegraphics[width=0.19\textwidth]{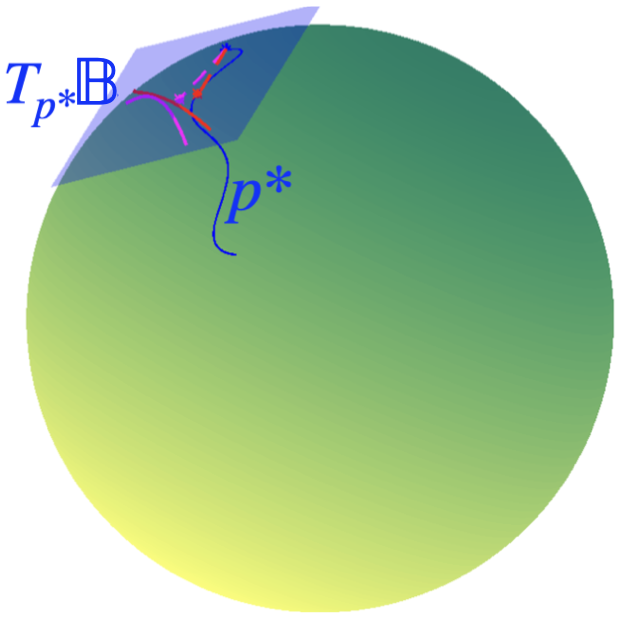} & \raisebox{2.5\height}{\scalebox{1.5}{$\xLongrightarrow{\exp_{p^{*}}(u_i, w_i)}$}} &
 \includegraphics[width=0.19\textwidth]{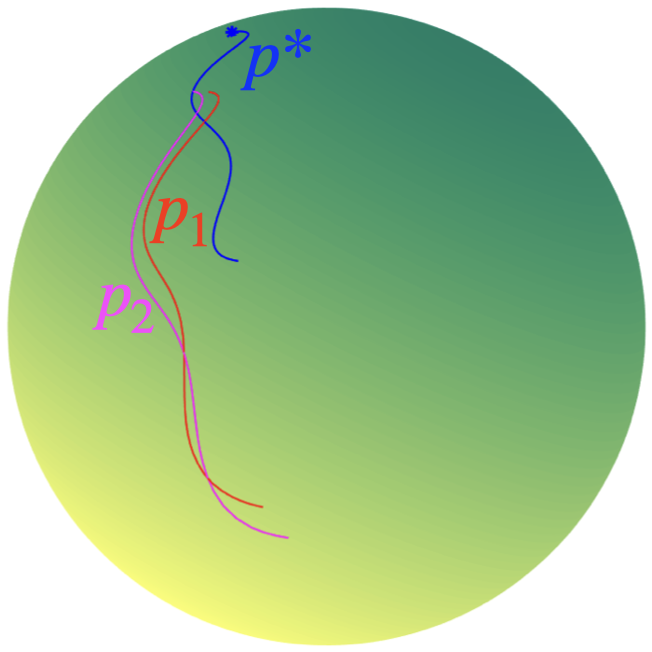} \\
   & (a) & & & (b)  
\end{tabular}
\end{center}
\caption{Illustration on data simulation. Panel (a) shows two samples on $T_{p^{*}} {\mathbb{B}}$ simulated using the probability model \eqref{eqn:sim1}.  Panel (b) shows two curves $\exp_{p^{*}}(u_1, w_1)$ and $\exp_{p^{*}}(u_2, w_2)$ that are interest here.}
\label{fig:SimCurves1}
\end{figure}

\begin{figure}
\begin{center}
\begin{tabular}{ccccc}
 \includegraphics[width=0.22\textwidth]{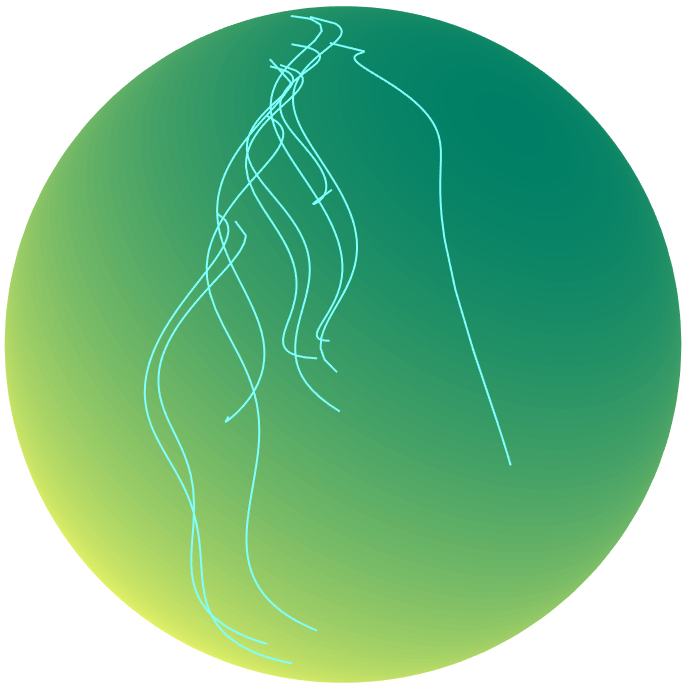} & 
 \raisebox{6.5\height}{\scalebox{2}{$+$}} &
 \includegraphics[width=0.3\textwidth]{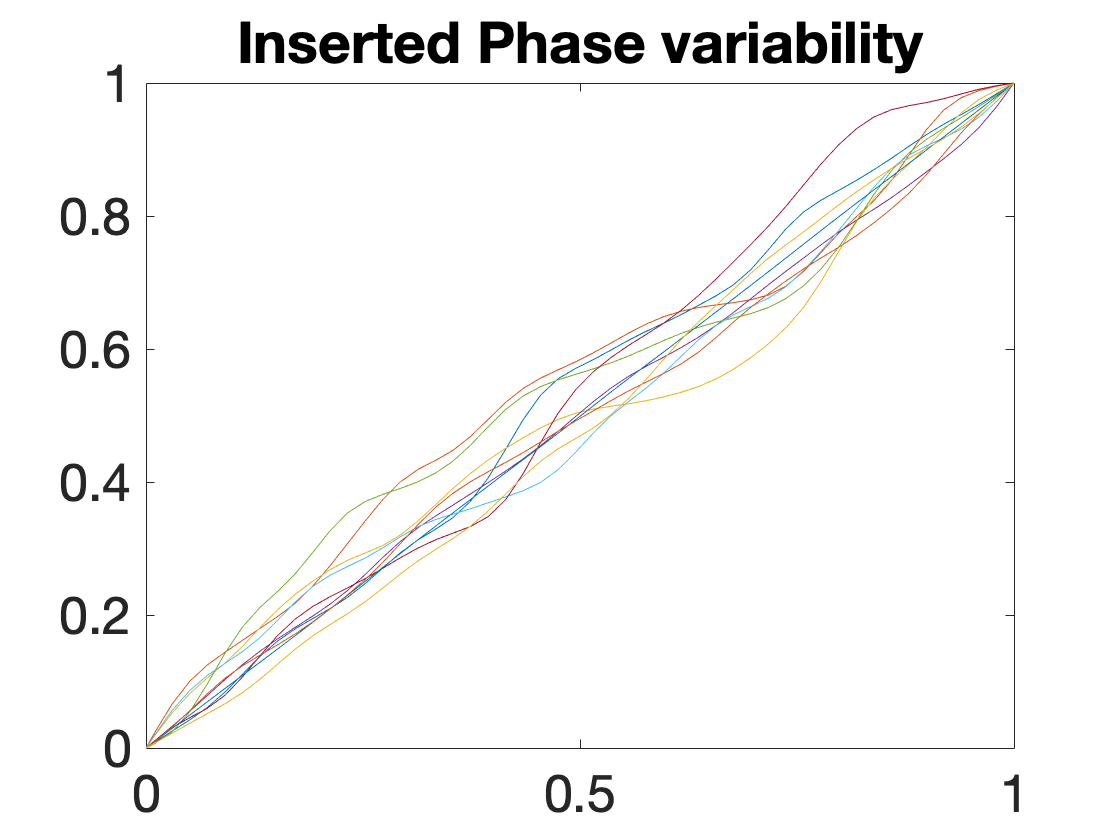} & \raisebox{6.5\height}{\scalebox{2}{$+$}} &
 \includegraphics[width=0.3\textwidth]{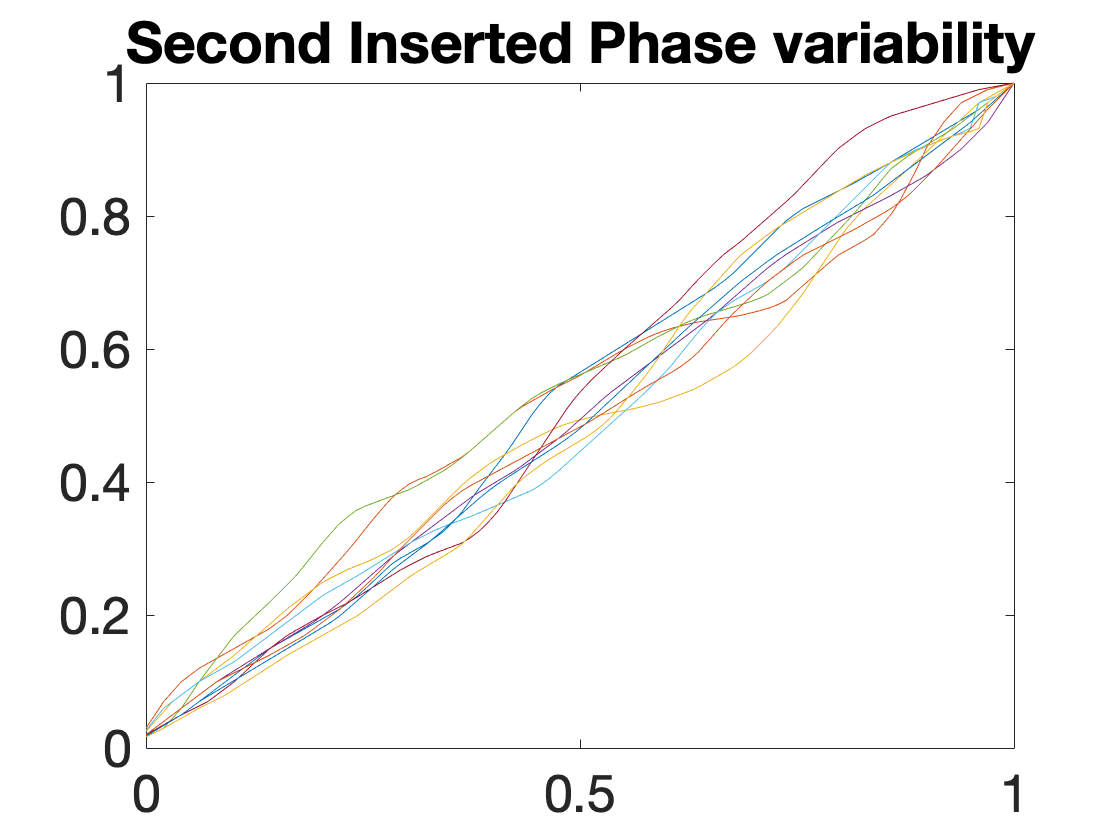} \\
   (a) & & (b) & & (c)  
\end{tabular}
\end{center}
\caption{Adding phase variability to simulated functions on $\S^2$. Panel (a) shows curves simulated using the probability model \eqref{eqn:sim1}.  Panel (b) shows the phase variability that was inserted into functions. Panel (c) shows the phase variability that was inserted into functions after inserting the phase variability (b).}
\label{fig:SimCurves2}
\end{figure}

Next, we consider the sample  Fr\'{e}chet mean on $\C/\tilde{\Gamma}$. Given a set of functions $\{p_1,...,p_n\}$, to optimize the Fr\'{e}chet function, the proposed Algorithm \ref{alg:aloptx} performs the following steps: 1) for a given starting point $x\in\S^2$, it jointly updates the base-curve $\beta_i,$  the warping function $\gamma_i,$ and the TSRVC $q_{\mu}$ of the sample Fr\'{e}chet mean $p_{\mu}$; and then 2) performs the gradient step using the exponential mapping on $\S^2,$ since the gradient of the Fr\'{e}chet function corresponding to the TSRVC is equal to 0. Similarly, we used a Gaussian process model \eqref{eqn:sim1} to simulate tangent vectors on $T_{p^{*}} {\mathbb{B}}$ then map it to $ {\mathbb{B}}$ using the exponential mapping $\exp_{p^{*}}(\cdot)$. Figure \ref{fig:SimCurves1} panel (a) shows two examples of the sampled $(u, w)$ and panel (b) shows the corresponding functions on $\S^2$. We also inserted phase variability to the simulated functions (see Figure \ref{fig:SimCurves2}). Next, we compared our Algorithm \ref{alg:aloptx} with the one from \cite{ZhengwuZhang2018} (Algorithm 3 in \cite{ZhengwuZhang2018}) to find the optimal sample Fr\'echet mean. Figure \ref{fig:SimCurves3} shows four different examples, where panel (a) shows the phase variability obtained from Algorithm \ref{alg:aloptx}, panel (b) shows the simulated functions $p_j$, the true mean $p^{*}$, and the sample mean $\tilde{p}_{\mu}$, and panel (c) shows the average of mean squared distances obtained from different algorithms. Given the complexity of the simulated functions, apparently,  extrinsic mean method, i.e., first calculate the Euclidean mean in $\mathbb{R}^3$ and then map it back to $\S^2$, will not perform well, and hence we did not include in our comparison. From the results, we see that our Fr\'echet mean captures the amplitude (the shape) of these functions well, and gives a smaller Fr\'echet function value.  

\begin{figure}
 \begin{center}
  \begin{tabular}{ccc}
   \includegraphics[width=0.365\textwidth]{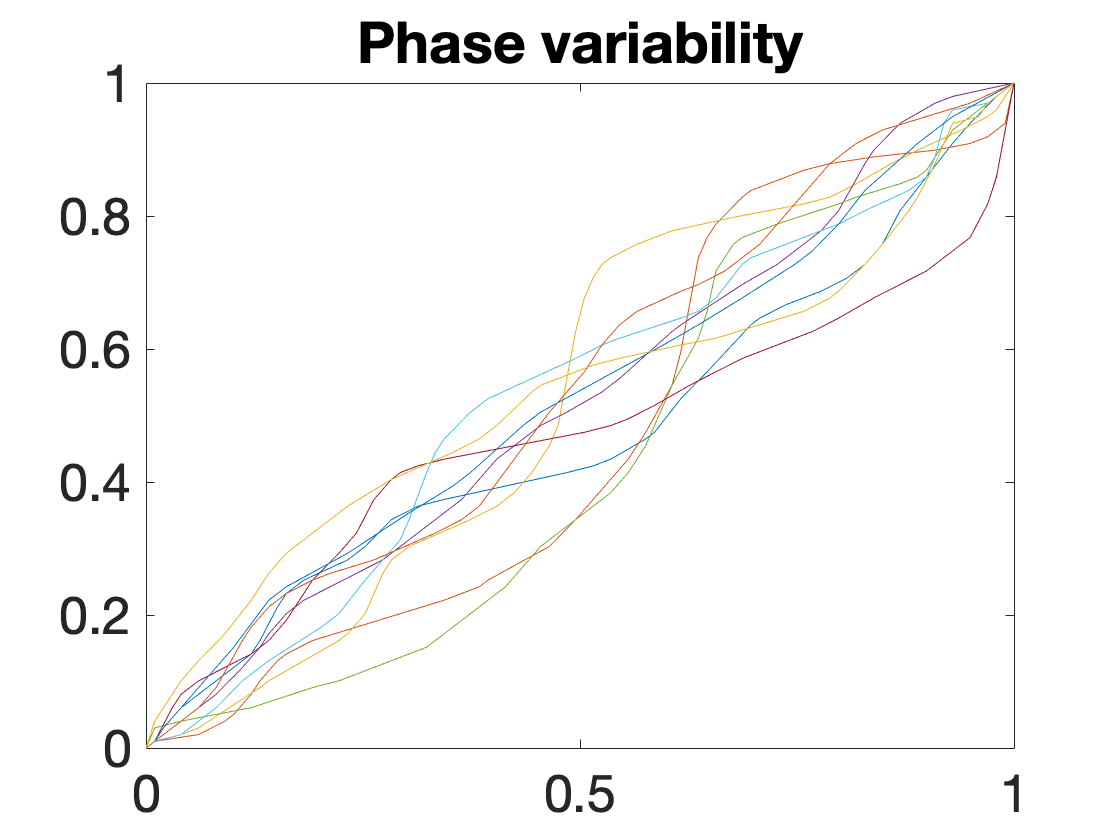} & 
   \includegraphics[width=0.265\textwidth]{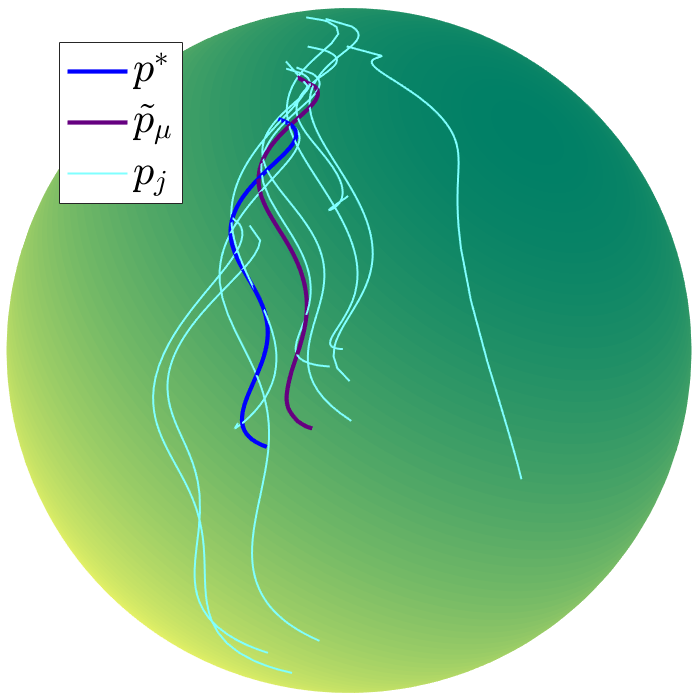} &
   \raisebox{1.1\height}{\scalebox{0.7}{\begin{tabular}{c|c}
       Methods & $\frac{1}{n}\sum_{i=1}^nd^2_{ {\mathbb{B}}/\Gamma}(p_i, \tilde{p}_{\mu})$ \\
       \hline\hline 
       Algorithm \ref{alg:aloptx} & 0.272878 \\
       \hline
       \cite{ZhengwuZhang2018}
        with $N = 30, M = 60$ & 0.28613 \\
       \hline
       \cite{ZhengwuZhang2018}
        with $N = 30, M = 120$ & 0.286135 \\ 
       \hline
       \cite{ZhengwuZhang2018}
        with $N = 30, M = 240$ & 0.286153 \\  
       \hline
       \cite{ZhengwuZhang2018}
        with $N = 60, M = 60$ & 0.286137 \\
       \hline
       \cite{ZhengwuZhang2018}
        with $N = 60, M = 120$ & 0.286144 \\
          \hline
       \cite{ZhengwuZhang2018}
        with $N = 60, M = 240$ & 0.28616 \\ 
       \hline
       \cite{ZhengwuZhang2018}
        with $N = 120, M = 60$ & 0.286139 \\
       \hline
       \cite{ZhengwuZhang2018}
        with $N = 120, M = 120$ & 0.286146 \\
       \hline
       \cite{ZhengwuZhang2018}
        with $N = 120, M = 240$ & 0.286162 \\
   \end{tabular}}} \\
   \includegraphics[width=0.365\textwidth]{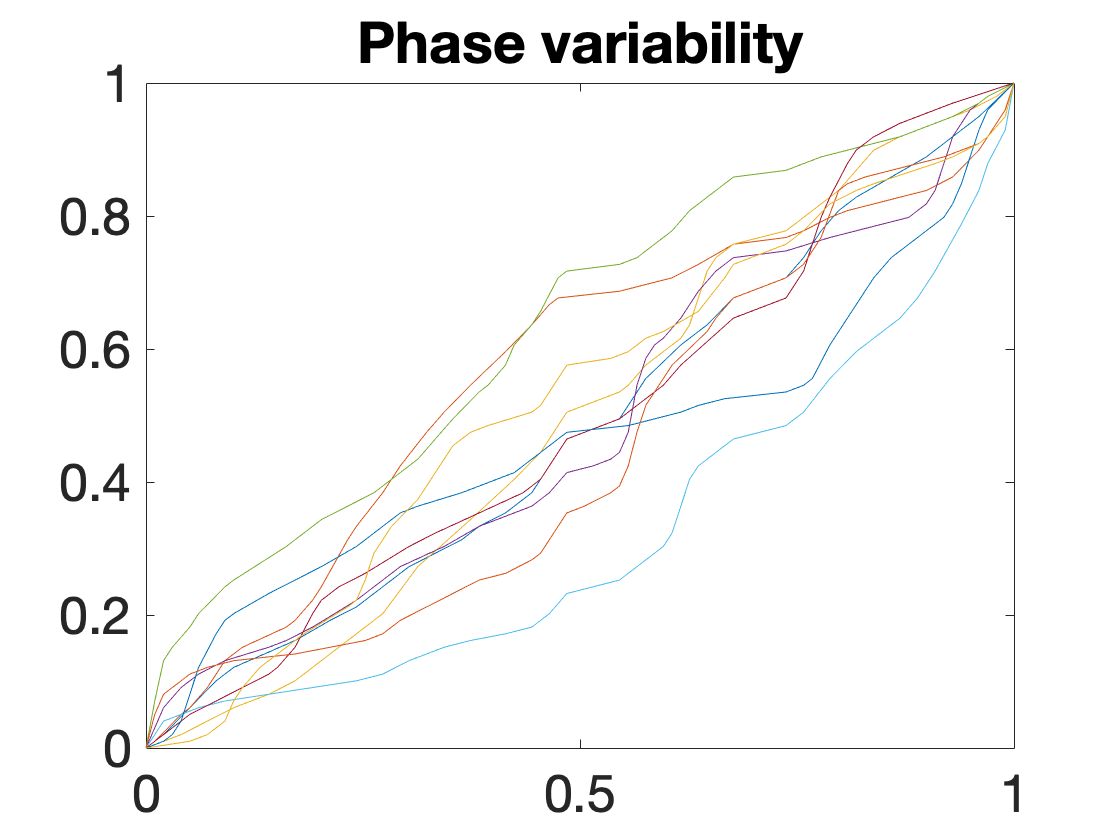} & 
   \includegraphics[width=0.265\textwidth]{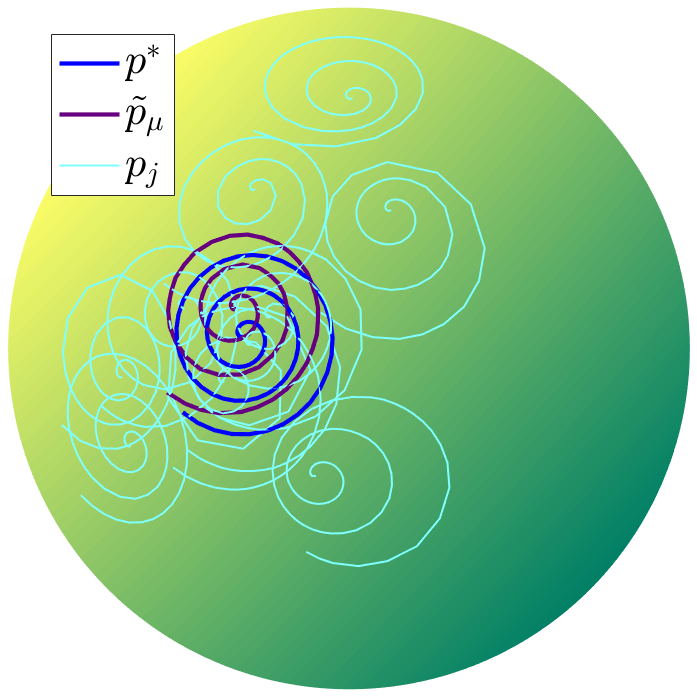} &
   \raisebox{1.1\height}{\scalebox{0.7}{\begin{tabular}{c|c}
       Methods & $\frac{1}{n}\sum_{i=1}^nd^2_{ {\mathbb{B}}/\Gamma}(p_i, \tilde{p}_{\mu})$ \\
       \hline\hline 
       Algorithm \ref{alg:aloptx} & 0.225475 \\
       \hline
       \cite{ZhengwuZhang2018}
        with $N = 30, M = 60$ & 0.243124 \\
       \hline
       \cite{ZhengwuZhang2018}
        with $N = 30, M = 120$ & 0.236718 \\ 
       \hline
       \cite{ZhengwuZhang2018}
        with $N = 30, M = 240$ & 0.236772 \\  
       \hline
       \cite{ZhengwuZhang2018}
        with $N = 60, M = 60$ & 0.243322 \\
       \hline
       \cite{ZhengwuZhang2018}
        with $N = 60, M = 120$ & 0.236672 \\
          \hline
       \cite{ZhengwuZhang2018}
        with $N = 60, M = 240$ & 0.236792 \\ 
       \hline
       \cite{ZhengwuZhang2018}
        with $N = 120, M = 60$ & 0.243327 \\
       \hline
       \cite{ZhengwuZhang2018}
        with $N = 120, M = 120$ & 0.236677 \\
       \hline
       \cite{ZhengwuZhang2018}
        with $N = 120, M = 240$ & 0.236797 \\
   \end{tabular}}} \\
   \includegraphics[width=0.365\textwidth]{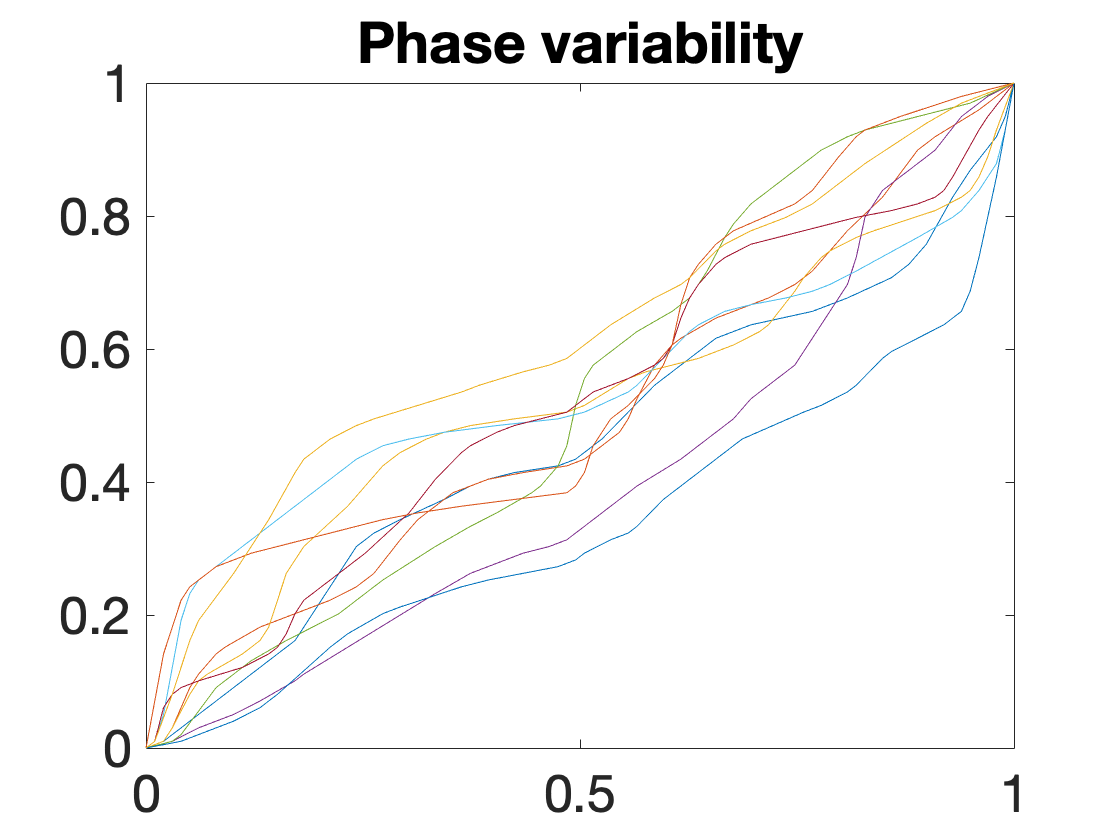} & 
   \includegraphics[width=0.265\textwidth]{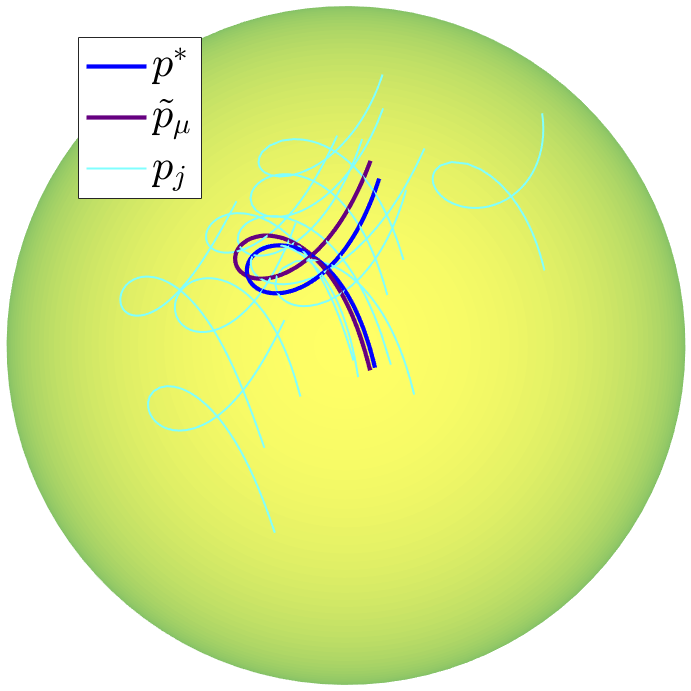} &
   \raisebox{1.1\height}{\scalebox{0.7}{\begin{tabular}{c|c}
       Methods & $\frac{1}{n}\sum_{i=1}^nd^2_{ {\mathbb{B}}/\Gamma}(p_i, \tilde{p}_{\mu})$ \\
       \hline\hline 
       Algorithm \ref{alg:aloptx} & 0.124914 \\
       \hline
       \cite{ZhengwuZhang2018}
        with $N = 30, M = 60$ & 0.130731 \\
       \hline
       \cite{ZhengwuZhang2018}
        with $N = 30, M = 120$ & 0.13071 \\ 
       \hline
       \cite{ZhengwuZhang2018}
        with $N = 30, M = 240$ & 0.130721 \\  
       \hline
       \cite{ZhengwuZhang2018}
        with $N = 60, M = 60$ & 0.130741 \\
       \hline
       \cite{ZhengwuZhang2018}
        with $N = 60, M = 120$ & 0.130719 \\
          \hline
       \cite{ZhengwuZhang2018}
        with $N = 60, M = 240$ & 0.130731 \\ 
       \hline
       \cite{ZhengwuZhang2018}
        with $N = 120, M = 60$ & 0.130743 \\
       \hline
       \cite{ZhengwuZhang2018}
        with $N = 120, M = 120$ & 0.130721 \\
       \hline
       \cite{ZhengwuZhang2018}
        with $N = 120, M = 240$ & 0.130733 \\
   \end{tabular}}} \\
   \includegraphics[width=0.365\textwidth]{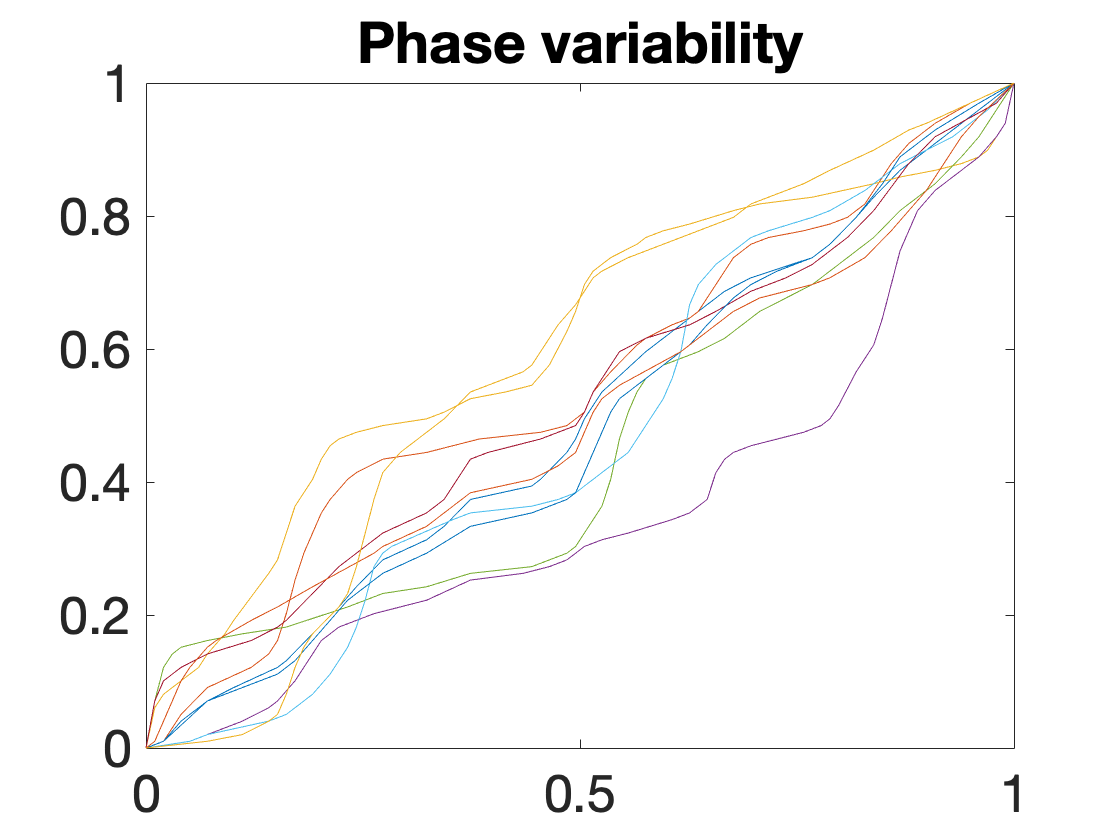} & 
   \includegraphics[width=0.265\textwidth]{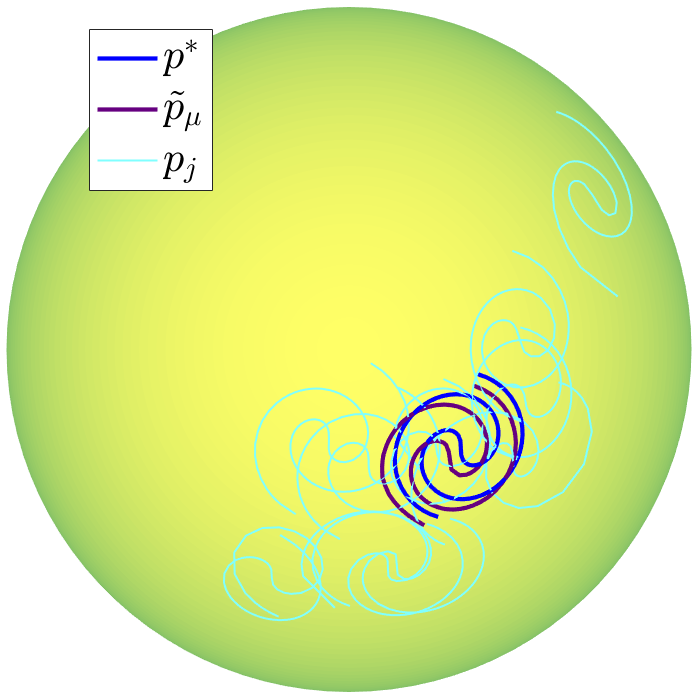} &
   \raisebox{1.1\height}{\scalebox{0.7}{\begin{tabular}{c|c}
       Methods & $\frac{1}{n}\sum_{i=1}^nd^2_{ {\mathbb{B}}/\Gamma}(p_i, \tilde{p}_{\mu})$ \\
       \hline\hline 
       Algorithm \ref{alg:aloptx} & 0.217113 \\
       \hline
       \cite{ZhengwuZhang2018}
        with $N = 30, M = 60$ & 0.231254 \\
       \hline
       \cite{ZhengwuZhang2018}
        with $N = 30, M = 120$ & 0.23092 \\ 
       \hline
       \cite{ZhengwuZhang2018}
        with $N = 30, M = 240$ & 0.230834 \\  
       \hline
       \cite{ZhengwuZhang2018}
        with $N = 60, M = 60$ & 0.231281 \\
       \hline
       \cite{ZhengwuZhang2018}
        with $N = 60, M = 120$ & 0.230947 \\
          \hline
       \cite{ZhengwuZhang2018}
        with $N = 60, M = 240$ & 0.230861 \\ 
       \hline
       \cite{ZhengwuZhang2018}
        with $N = 120, M = 60$ & 0.231288 \\
       \hline
       \cite{ZhengwuZhang2018}
        with $N = 120, M = 120$ & 0.230953 \\
       \hline
       \cite{ZhengwuZhang2018}
        with $N = 120, M = 240$ & 0.230868 \\
   \end{tabular}}} \\
   (a) & (b) &  (c)   
  \end{tabular}
 \end{center}
 \caption{Examples of sample Fr\'echet mean for function amplitudes. Panel (a) shows the warping function $\{\gamma_1, \dots, \gamma_n\}$ obtained by Algorithm \ref{alg:aloptx}, panel (b) shows the sample data $\{p_1, \dots, p_n\},$ the Fr\'{e}chet mean $p^{*}$, the sample Fr\'{e}chet mean $p_{\mu}$, and panel (c) shows a table of optimal sample F\'{e}chet functions computed by Algorithm \ref{alg:aloptx} and the algorithm from \cite{ZhengwuZhang2018} with various $N$ and $M$.}
 \label{fig:SimCurves3}
\end{figure}

\section{Real Data Analysis}

We also illustrate our framework on  two real datasets: bird migration data \cite{Kochert} and Atlantic hurricane tracks \cite{Landsea2013AtlanticHD} (see Figure \ref{fig:intro_hurrican_bird} for a snapshot of the data).

\begin{itemize}
\item The {\bf bird migration data} \cite{Kochert} has 35 migration trajectories of Swainson’s Hawk from western North America to Argentina and back, observed from 1995 to 1997. The collected data were used to help identify the most important areas for pesticide control to reduce fatalities of the bird in Argentina.

\item The {\bf hurricane data} \cite{Landsea2013AtlanticHD} contains hurricane tracks originated from the Atlantic ocean and Gulf of Mexico. The US National Hurricane Center (NHC) conducted post-storm analyses of hurricanes that have been recorded by the National Oceanic and Atmospheric Administration (NOAA). HURDAT2 is the database we used here and and can be found on \url{https://www.nhc.noaa.gov/data/}.
\end{itemize}

We implemented Algorithm \ref{alg:optx} to find the sample Fr\'{e}chet mean of Swainson hawk migrations on $ {\mathbb{B}}$ (see the first row in Figure \ref{fig:FM} (a)) and Algorithm \ref{alg:aloptx} to find  the sample Fr\'{e}chet mean on $ {\mathbb{B}}/\Gamma$ (see the first row in Figure \ref{fig:FM} (b)). Comparing the two means, we can see that birds travel with non-synchronized speed and the amplitude mean is more representative of the common pattern in the data, e.g., the mean on $\C$ is much shorter than the mean on $\C/\Gamma$ and deviates from the major path due to averaging the non-synchronized trajectories or functions. Similar analysis was also done for the hurricane data and the results are presented in the second row of Figure \ref{fig:FM}. There is also a descent phase variability in the hurricane data. 

\begin{figure}
\begin{center}
\begin{tabular}{c|cc}
 \includegraphics[width=0.25\textwidth]{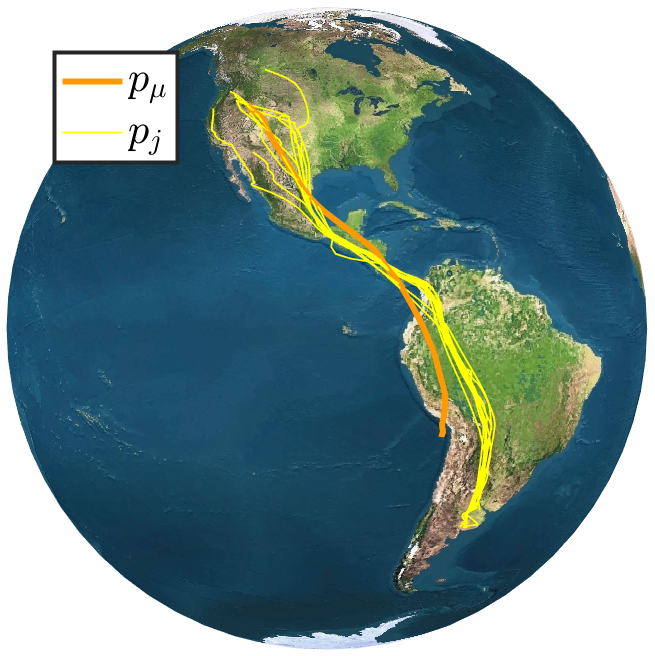} &
  \includegraphics[width=0.25\textwidth]{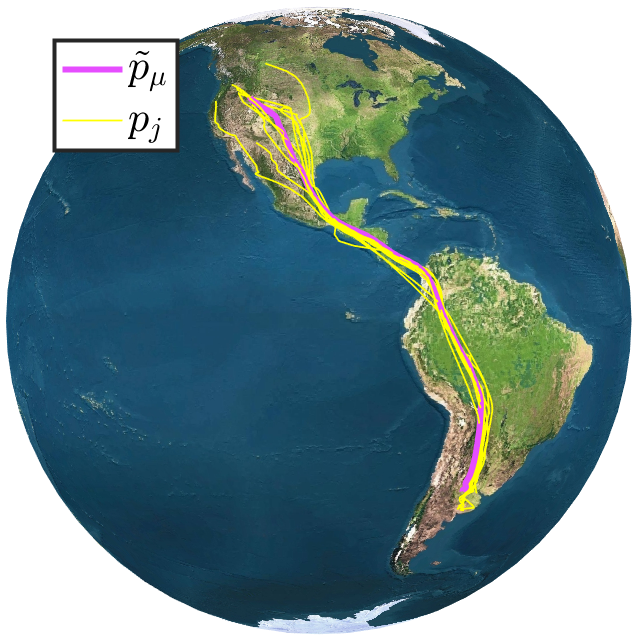} &
 \includegraphics[width=0.35\textwidth]{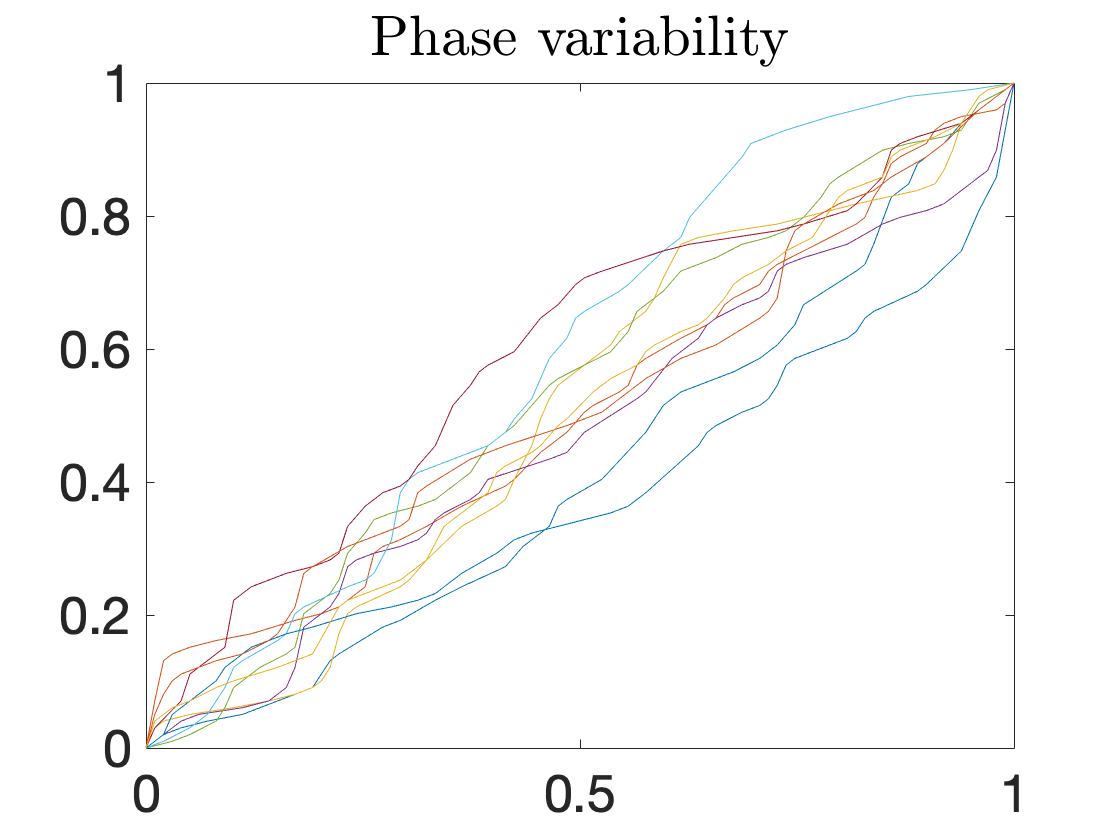} \\
 \includegraphics[width=0.25\textwidth]{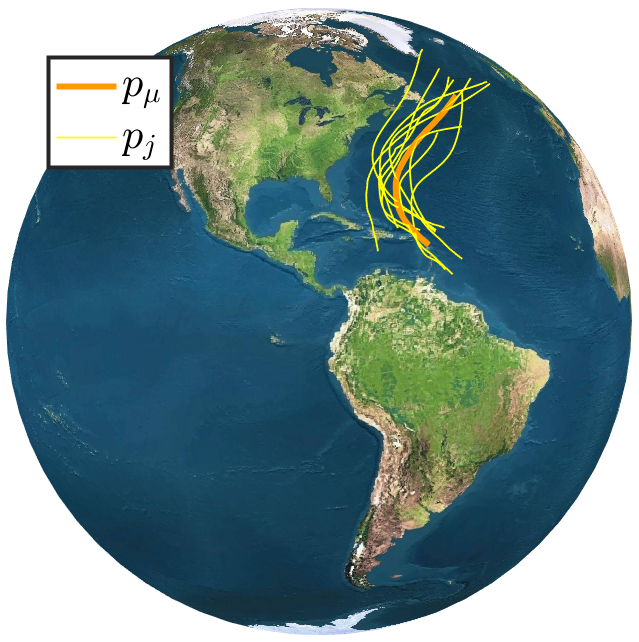} &
  \includegraphics[width=0.25\textwidth]{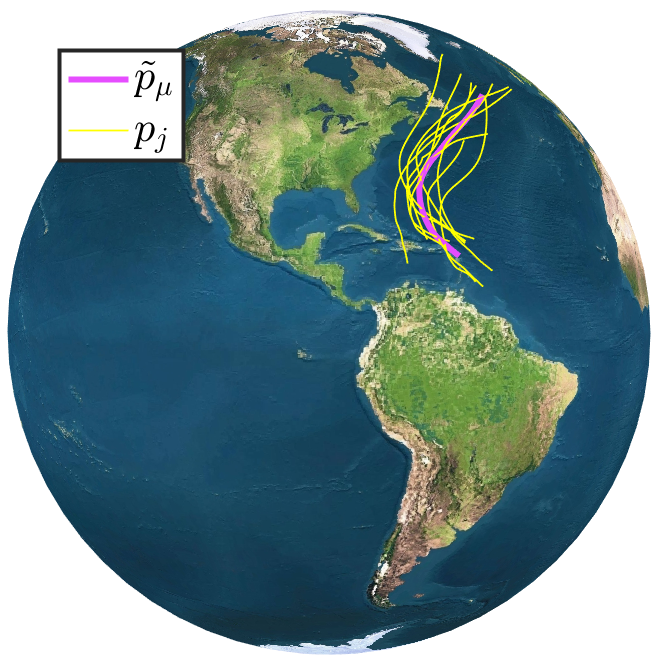} &
 \includegraphics[width=0.35\textwidth]{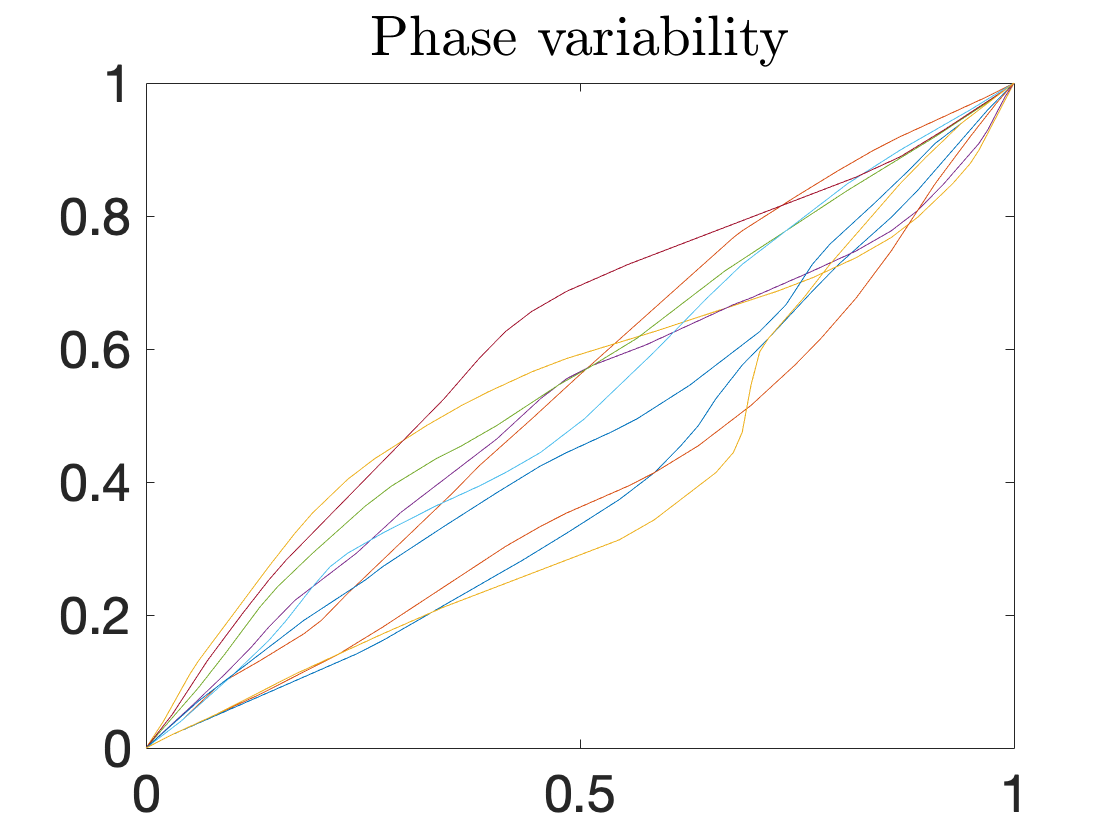}  \\
   (a) & (b) & (c) \\
\end{tabular}
\end{center}
\caption{Real data analysis on bird migration trajectories and hurricane paths. Panel (a) shows the sample Fr\'echet mean $p_{\mu}$ obtained by Algorithm \ref{alg:optx}.  Panel (b) shows the sample Fr\'echet mean $\tilde{p}_{\mu}$ derived by Algorithm \ref{alg:aloptx} and the corresponding time warping functions for alignment.}
\label{fig:FM}
\end{figure}

To better compare and visualize data on $\C$ and $\C/\Gamma$, we computed the sample covariance matrix of TSRVCs along the sample Fr\'echet mean before and after alignment. On $T_{p_\mu} {\mathbb{B}}$ (before alignment), at time $t$ along the mean, we calculate the covariance of the sample TSRVCs as
\begin{align*}
 K_{\mu}(t) = \frac{1}{n-1}\sum_{j=1}^n w_j(t)w_j(t)^T \qquad {\rm where} \quad (u_j, w_j) = \exp_{p_{\mu}}p_j,
\end{align*}
and on $T_{\tilde{p}_\mu} {\mathbb{B}}/\Gamma$, we have
\begin{align*}
 \tilde{K}_{\mu}(t) = \frac{1}{n - 1}\sum_{j=1}^n \tilde{w}_j(t)\tilde{w}_j(t)^T \qquad {\rm where} \quad (\tilde{u}_j, \tilde{w}_j) = \exp_{\tilde{p}_{\mu}}\tilde{p}_j.
\end{align*}
Each $K_\mu(t)$ is a $2\times 2$ tensor and can be plot as an ellipsoid. We parallel transport $K_\mu(t)$ from the tangent space on $T_{p_{\mu}(0)}\S^2$ to $T_{p_{\mu}(t)}\S^2$ along $p_{\mu}$ and display them along $p_{\mu}$. Figure \ref{fig:V} shows these covariance matrices before alignment (panel (a)) and after alignment (panel (c)). We also computed the trace of $K_{\mu}(t)$ ($\rho_{\mu}(t) = \text{trace}\big(K_{\mu}(t)\big)$ and $\tilde{\rho}_{\mu}(t) = \text{trace}\big(\tilde{K}_{\mu}(t)\big)$) to summarize the variance as one number at time $t$, which is displayed on panel (b).  These results show the importance of considering phase variability in functional data analysis on $\S^2$. 

\begin{figure}
\begin{center}
\begin{tabular}{c|c|c}
 \includegraphics[width=0.25\textwidth]{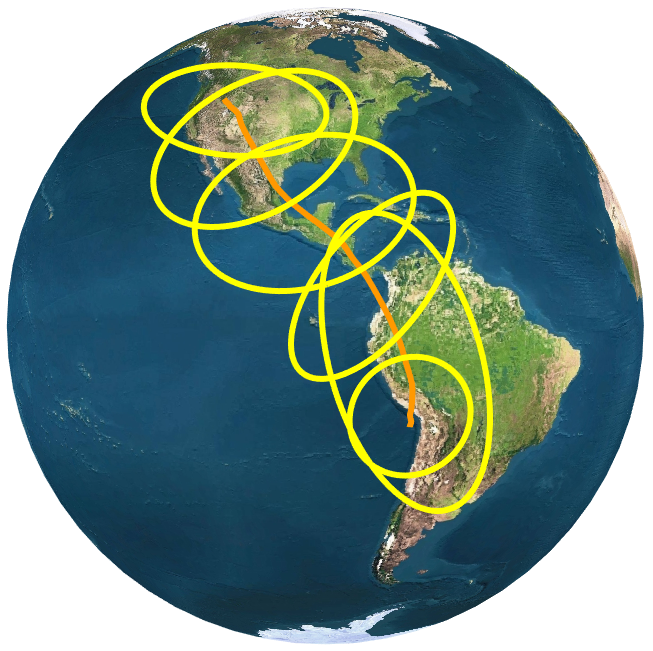} & 
 \includegraphics[width=0.35\textwidth]{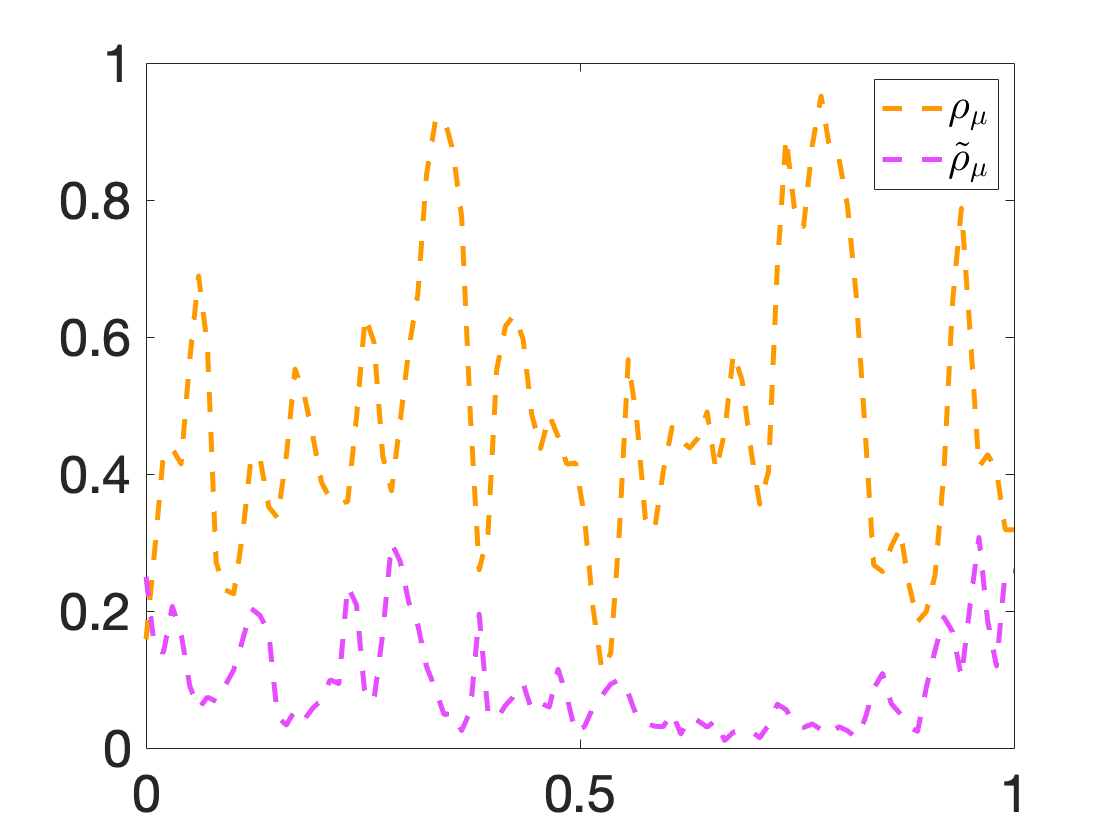} & 
 \includegraphics[width=0.25\textwidth]{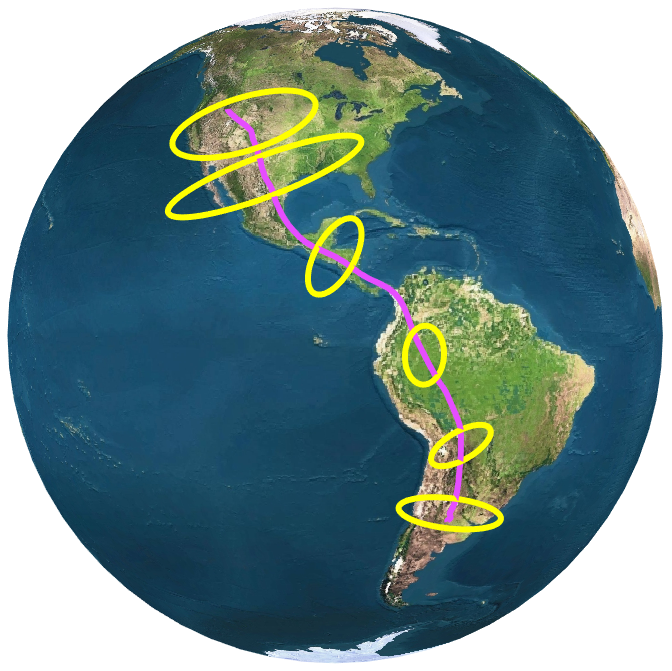} \\
 \includegraphics[width=0.25\textwidth]{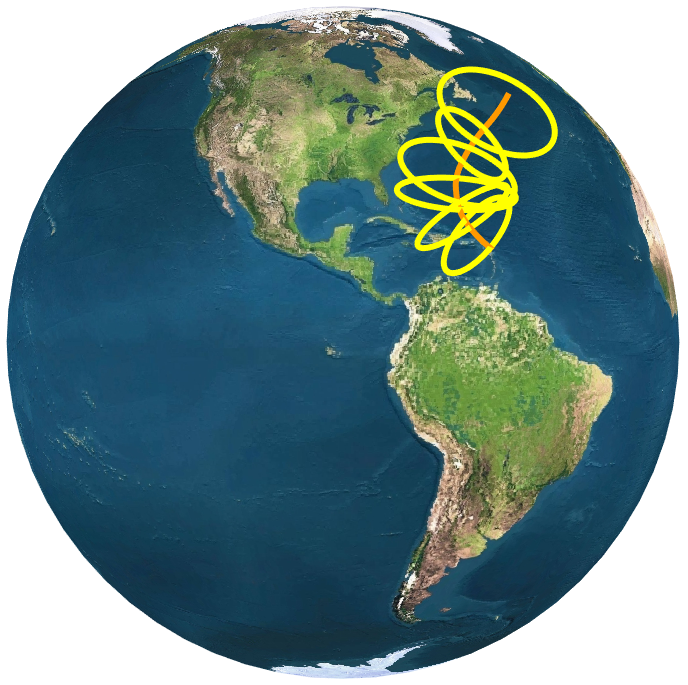} & 
 \includegraphics[width=0.35\textwidth]{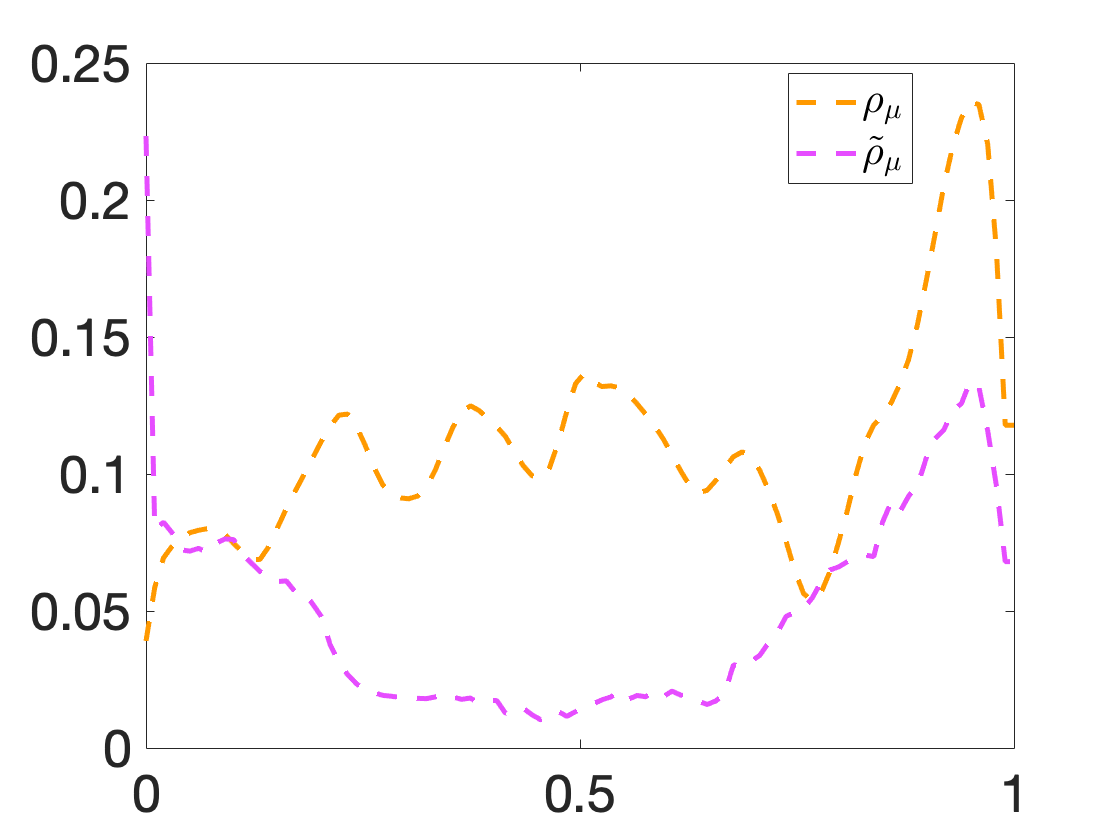} & 
 \includegraphics[width=0.25\textwidth]{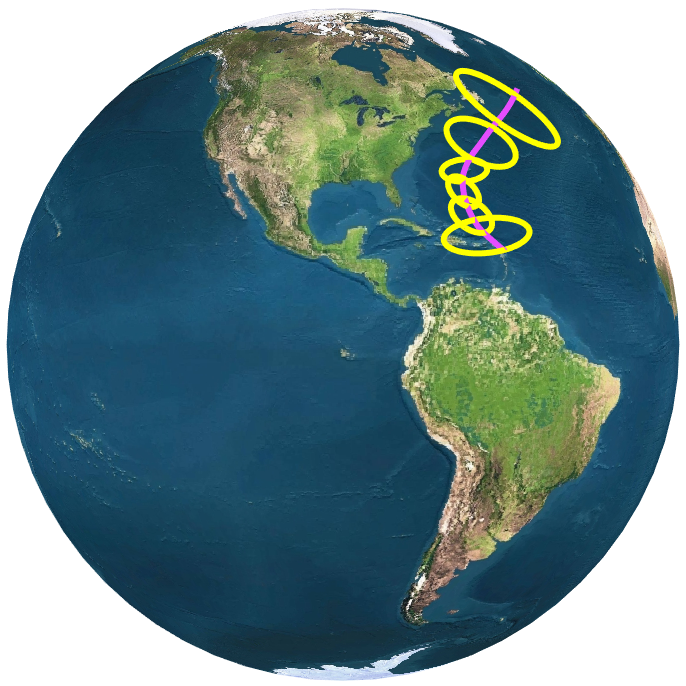} \\
   (a) & (b) & (c) \\
\end{tabular}
\end{center}
\caption{Visualization of variability of TSRVCs before and after alignment. 
Panel (a) shows the covariance ($K_\mu(t)$) of TSRVCs before alignment, and panel (c) shows the covariance after alignment. Panel (b) shows the trace of $K_\mu(t)$s.}
\label{fig:V}
\end{figure}

At last, we compared our Algorithm \ref{alg:aloptx} with the one from \cite{ZhengwuZhang2018} (Algorithm 3 in \cite{ZhengwuZhang2018})  in Table \ref{tab:RealData}. We can see that proposed framework achieves a significant improvement (around 20\%) in the Swainson hawk data, and a slight improvement in the hurricane data (around 3\%) when optimizing the Fr\'{e}chet function.  

\begin{table}
\caption{Comparison of the Algorithm \ref{alg:aloptx} with the one from \cite{ZhengwuZhang2018} to find the sample Fr\'echet mean on real data.}
\label{tab:RealData}       
\begin{tabular}{c||c}
Swainson hawk data & Hurricane data \\
\hline\hline\hline 
\begin{tabular}{c|c}
 Methods & $\frac{1}{n}\sum_{i=1}^n d^2_{ {\mathbb{B}}/\Gamma}(p_i, \tilde{p}_{\mu})$ \\
 \hline\hline 
 Algorithm \ref{alg:aloptx} & 0.105873 \\
 \hline
 \cite{ZhengwuZhang2018} with $N = 30, M = 60$ & 0.136115 \\
 \hline
 \cite{ZhengwuZhang2018} with $N = 30, M = 120$ & 0.136112 \\ 
 \hline
 \cite{ZhengwuZhang2018} with $N = 30, M = 240$ & 0.136113 \\  
 \hline
 \cite{ZhengwuZhang2018} with $N = 60, M = 60$ & 0.136115 \\
 \hline
 \cite{ZhengwuZhang2018} with $N = 60, M = 120$ & 0.136112 \\
 \hline
 \cite{ZhengwuZhang2018} with $N = 60, M = 240$ & 0.136113 \\ 
 \hline
 \cite{ZhengwuZhang2018} with $N = 120, M = 60$ & 0.136116 \\
 \hline
 \cite{ZhengwuZhang2018} with $N = 120, M = 120$ & 0.136112 \\
 \hline
 \cite{ZhengwuZhang2018} with $N = 120, M = 240$ & 0.136113 
\end{tabular} &
\begin{tabular}{c|c}
 Methods & $\frac{1}{n}\sum_{i=1}^nd^2_{ {\mathbb{B}}/\Gamma}(p_i, \tilde{p}_{\mu})$ \\
 \hline\hline 
 Algorithm \ref{alg:aloptx} & 0.052998 \\
 \hline
 \cite{ZhengwuZhang2018} with $N = 30, M = 60$ & 0.0545199 \\
 \hline
 \cite{ZhengwuZhang2018} with $N = 30, M = 120$ & 0.0545143 \\ 
 \hline
 \cite{ZhengwuZhang2018} with $N = 30, M = 240$ & 0.0545157 \\  
 \hline
 \cite{ZhengwuZhang2018} with $N = 60, M = 60$ & 0.0545207 \\
 \hline
 \cite{ZhengwuZhang2018} with $N = 60, M = 120$ & 0.0545143 \\
 \hline
 \cite{ZhengwuZhang2018} with $N = 60, M = 240$ & 0.0545158 \\ 
 \hline
 \cite{ZhengwuZhang2018} with $N = 120, M = 60$ & 0.0545207 \\
 \hline
 \cite{ZhengwuZhang2018} with $N = 120, M = 120$ & 0.0545143 \\
 \hline
 \cite{ZhengwuZhang2018} with $N = 120, M = 240$ & 0.0545158 
\end{tabular}
\end{tabular}
\end{table}

\section{Conclusion}
This paper studies functional data on $\S^2$, and provides a method to intrinsically compute the amplitude mean for a set of samples. Functions on $\S^2$ are represented as a pair, consisting of its starting point and a transported square-root curve (TSRVC). With this representation, the domain of interest becomes a fiber bundle, denoted as $\C$. We then study the Riemannian structure of $\C$, and its quotient space $ {\mathbb{B}}/\tilde{\Gamma}$ containing the amplitude of smooth functions. 
Gradient descent algorithms are developed to find the geodesic between two points on $\C$ and $\C/\tilde{\Gamma}$.
In this process, we simplify the geodesic search by developing several novel tools, including closed form solutions for parallel transport of vectors along circle arcs on $\S^2$, and a closed form gradient of the distance function between two functions. From there, we extend our toolbox to perform analysis of a set of functions and by developing computational algorithms for finding the sample Fr\'echet means on $\C$ and $ {\mathbb{B}}/\tilde{\Gamma}$, where novel exponential and inverse exponential maps are introduced. The proposed framework is comprehensively evaluated in both simulated and real data, and is shown to be superior to its competitor. Implementation of these tools is published on GitHub repository: \url{https://github.com/Bayan2019/2DSphericalTrajectories}.

\bibliographystyle{spmpsci}
\bibliography{bibfile, sample}

\end{document}